\documentclass[, ]{informs3}
\OneAndAHalfSpacedXI 

\usepackage{breqn}
\usepackage{amsmath}
\usepackage{amssymb}
\usepackage{subcaption}
\usepackage{bm}
\usepackage{url}
\usepackage{graphicx}
\usepackage{epstopdf}
\usepackage{epsfig,psfrag}
\usepackage{placeins}
\usepackage{verbatim}  
\usepackage{enumerate}
\usepackage{color}
\usepackage{pdfsync}
\usepackage{hyperref}
\usepackage{algorithm}
\usepackage{algpseudocode}
\usepackage{pifont}
\usepackage{placeins}
\usepackage{bbm}
\usepackage{mathrsfs}

\usepackage[T2A]{fontenc}
\usepackage[utf8]{inputenc}

\usepackage{geometry}

\usepackage{tikz}\usetikzlibrary{patterns,angles,calc,arrows.meta,decorations.markings}
\usepackage{tkz-euclide}
\usepackage{thmtools,thm-restate}
\usepackage{empheq}

\global\long\def\E{\mathbb{E}}
\global\long\def\bproof{\emph{Proof.}~}

\global\long\def\R{\mathit{R}}

\global\long\def\L{\mathit{L}}
\global\long\def\U{\mathit{U}}

\global\long\def\tK{\mathit{\tilde{K}}}
\global\long\def\n{\mathsf{n}}
\global\long\def\Nm{\mathsf{N^-}}
\global\long\def\Np{\mathsf{N^+}}

\global\long\def\1{\mathbf{1}}

\global\long\def\mu{\nu}
\global\long\def\sigma{s}
\global\long\def\pi{\zeta}

\global\long\def\LNEW{\underline{\mathcal{B}}}
\global\long\def\UNEW{\overline{\mathcal{B}}}
\global\long\def\MNEW{\mathcal{R}}
\global\long\def\barLNEW{\mathcal{B}^c}
\global\long\def\aNEW{\mathsf{a}}
\global\long\def\hNEW{\mathsf{h}}

\global\long\def\I{\mathcal{D}}
\global\long\def\Lambda{\Pi}
\global\long\def\Psi{\Sigma}

\declaretheorem{selfthm}

\usepackage{natbib}
 \bibpunct[, ]{(}{)}{,}{a}{}{,}%
 \def\bibfont{\small}%
\TheoremsNumberedThrough     

\EquationsNumberedThrough

\MANUSCRIPTNO{ }
\renewcommand{\theARTICLETOP}{}
\begin{document}

\RUNAUTHOR{Zhengli Wang}

\RUNTITLE{Strategies for Milestone-driven Start-ups}

\ARTICLEAUTHORS{
\AUTHOR{Zhengli Wang}
\AFF{Faculty of Business \& Economics, the University of Hong Kong, \EMAIL{wzl1@hku.hk}} 
}

\TITLE{Strategies for Milestone-driven Start-ups \\  in  Multi-activity Settings}
\ABSTRACT{\textbf{Problem definition:} New venture start-ups need to ``survive'' through multiple stages of reaching milestone targets. We investigate the strategies for start-ups in a milestone-oriented setting. We examine a model of an entrepreneurial start-up firm, where its state is captured by a diffusion process. The entrepreneur can choose between multiple activities (or controls), which incur different cost and determine the drift and the variance of the process. Depending on whether the process reaches a fixed upper boundary or a lower one, the start-up firm succeeds or fails. \textbf{Methodology/results:} 
Continuous-time stochastic models with multiple ($\ge 3$) controls are typically very challenging to deal with. In this work, we are able to completely solve for the optimal policy and provide an explicit characterization of its structure. In particular, the optimal policy only uses controls from a set characterized by a so-called \textit{efficient frontier} curve that orders the controls by two intuitive measures: \textit{riskiness} (drift-to-volatility ratio) and \textit{cost-effectiveness} (drift-to-cost ratio). A unique feature of our model is that depending on the model parameters, the \textit{efficient frontier} curves can be of different types, resulting in qualitatively different structures of the optimal policy. As far as we know, this is the first study that analyzes a stochastic control model which admits \textit{efficient frontier} curves of different types. \textbf{Managerial implications:} Our work provides start-up firms with intuitive measures to evaluate their activities and offers valuable insights on how the optimal strategies in a milestone-oriented setting change qualitatively contingent upon the specific scenario. We believe the results provide a foundational block in the study of entrepreneurial decision-making.}

\maketitle

\section{Introduction}

There have been substantial interest in studying and understanding how to build a successful venture start-up, as the amount of global venture investment has reached $\$643$ billion in $2021$, which represents an average of more than $60\%$ yearly growth since the last $10$ years ([\cite{teare2022global}]). At a high level, new venture start-ups need to ``survive'' through multiple stages, with the investors providing the capital in each stage ([\cite{lerner2020venture}], [\cite{colombo2021use}]). 

In particular, a start-up firm aims to reach a certain milestone or target ([\cite{mas2017investor}]): if it manages to do so, then the investors provide additional funding (otherwise, the firm may have to be abandoned ([\cite{eisenmann2021startups}])). For instance, Calm is a startup that provides guided meditations and other relaxation exercises through its meditation and sleep app ([\cite{nytimes2017calm}]). In the early days, the milestone targets were focused on customer acquisition. The company raised a \$88 million round of funding after reaching 1 million paying subscribers ([\cite{forbes2019meditation}]). For another instance, Slack is a team collaboration and communication tool that enables users to chat, share files, and work together in groups ([\cite{ftimes2015slack}]). The company's milestone targets were focused on increasing the number of daily active users (DAUs). The company raised \$427 million after reaching 8 million DAUs ([\cite{techcrunch2018slack}]). To achieve the milestone target, a startup firm typically engages in different activities that depend on the firm's stage ([\cite{wang2016key}]). These activities help to create value and reach milestones ([\cite{schindehutte2001understanding}]). For instance, in the case of Calm and Slack, the activities can be different marketing and promotional activities, such as displaying advertisements in various platforms ([\cite{insider2022slack}]), offering credits to selected users after successful referral to their friends ([\cite{saasquatch2022slack}]) and offering discounts to certain group of targeted users ([\cite{wethrift2023save}]). 

The goal of this work is to investigate the strategies for start-ups in a milestone-oriented setting. The central question is: what strategy may be best-suited for the startup in such a setting? Inherently this is a very difficult question as the entrepreneurial process represents a very complex phenomenon and (1) there is a lack of framework to guide the entrepreneur within a start-up and (2) it is probably very difficult to model every aspect of the process. Therefore as a first step, researchers analyze simpler and stylized models that capture some of the most salient features in the process (see e.g.  [\cite{crama2008milestone}], [\cite{besbes2012dynamic}], [\cite{harrison2015investment}]), which include but are not limited to uncertainty, limited resources and different activities that can be undertaken to influence outcomes of the firm.

This paper develops a general framework for start-ups operating in a milestone-driven environment. In the framework, the state of the start-up firm is modeled as a diffusion process. The entrepreneur chooses among different costly activities that affect the drift and variance of the process. The entrepreneur receives a reward when the process hits an upper boundary (which represents a predetermined milestone target) and receives nothing when the process hits a lower boundary. The objective is to maximize the expected terminal payoff minus the total costs incurred by the chosen activities.

\textbf{Main Results.} We explicitly characterize the multi-control setting's optimal policy, which can be represented by a sequence of disjoint intervals on the state-axis where within each interval the entrepreneur uses a certain control (i.e. the optimal policy is Markovian and the optimal control to use only depends on the state). The optimal policy has two features: (1) it uses only the controls from a set characterized by a so-called \textit{efficient frontier} curve; (2) its \textit{control sequence}, which is the order arrangement of the controls as the state increases, follows the a natural ordering specified by the \textit{efficient frontier} curve. The \textit{efficient frontier} can be obtained via a simple numerical procedure that selects and sorts the controls by two measures: drift-to-volatility ratio (which we call \textit{riskiness}) and drift-to-cost ratio (which we call \textit{cost-effectiveness}), which we show represent sufficient statistics.

We show that the \textit{efficient frontier} and the natural ordering specified by it has two different forms that belong to two regimes, which are characterized by the distance between the fixed lower and upper boundaries. When the two boundaries are relatively near, the optimal policy uses increasingly \textit{riskier} and more \textit{cost-effective} controls on a type-I \textit{efficient frontier} as the state of the firm increases. On the other hand, when the two boundaries are relatively far away, the optimal policy uses increasingly less \textit{risky} controls and more \textit{cost-effective} ones on a type-II \textit{efficient frontier}, as the state increases. We further illustrate the value of the optimal policy by comparing it with the myopic policy (the policy that only uses the most \textit{cost-effective} control) and show that there can be a large gap between the two policies' expected payoffs.

\textbf{Technical contributions.} Continuous-time stochastic models with multiple ($\ge 3$) controls are typically very challenging to deal with, and very few papers are able to provide a closed-form characterization of the optimal policy under fully general model parameters. To find a tractable solution to our formulation, we overcome two main difficulties. 

The first difficulty is to be able to prescribe the consideration set of optimal controls and be able to specify the order arrangement of them as the state varies (assuming the solution is Markovian and the optimal control used only depends on the state). By exploiting the structure of the sequence of ordinary differential equations (ODEs) associated with the value function, we can show two key observations: (1) when the value function has steeper slope at a state, then controls with a higher drift-to-cost ratio and a higher drift-to-volatility ratio will be used at that state; (2) the value function can be either strictly convex or strictly concave, depending on the distance between the two boundaries. These two observations help us identify the two distinctly different forms of \textit{efficient frontier} curves (which result in the two different consideration sets of controls that can potentially be used in the optimal policy), each associated with a different control order arrangement as the state varies. As far as we know, this is the first study that analyzes a stochastic control model that admits \textit{efficient frontier} curves of different types. This is in contrast with previous studies (such as [\cite{radner1996risk}]), where the \textit{efficient frontier}  curve is always of the same type and does not depend on the underlying parameters of the problem. 

Even with the prescription of an \textit{efficient frontier} curve and the order arrangement of the controls, another difficulty is that we do not know which control is the first one to be used in the optimal policy (here, ``first control'' means the control used when the state is very close to the lower boundary). This issue adds a significant degree of complication to the model as we cannot directly characterize the value function and verify that it is optimal. This is also in stark contrast with previous literature. For instance in [\cite{harrison2015investment}], the first control used in the optimal policy is always the one with the highest cost-to-information quality ratio. This is not the case for our setting, where the optimal control sequence can start with \textit{any} control on the \textit{efficient frontier}. To deal with the problem, we apply an adaptation of [\cite{harrison2015investment}]'s approach. In particular, we design a set of sufficiently smooth candidate value functions parametrized by a parameter (which represents the initial slope of the value function), where varying its value can lead to different initial control used in the control sequence but keeping the order arrangement of the controls on the \textit{efficient frontier}. We show that the set of trial value functions is ordered by the specified parameter, and there exists a unique one which leads to the optimal value function that solves the HJB equation. The method we present here can potentially be applied to other fixed-boundary stochastic control problems.

\textbf{Managerial Contributions.} This paper has many managerial contributions. 
First, we show that all controls can be characterized by two measures: \textit{riskiness} (drift-to-volatility ratio) and \textit{cost-effectiveness} (drift-to-cost ratio). These two measures can be easily interpreted by the entrepreneur. They are sufficient to prescribe the optimal policy, and they provide a clear illustration of the trade-off between drift, variance and cost.

Second, we prescribe the set of controls that can potentially be used in the optimal policy. This can be of great value to the entrepreneur because when there are many controls available, it is sufficient for the entrepreneur to focus only on those controls that belong to this set. This can potentially reduce the complexity of decision making significantly, as the entrepreneur can possibly get rid of a large number of controls in the consideration set. 

Lastly, we obtain the optimal policy and explicitly illustrate its structure. At a high level, the structure depends on how far the fixed lower boundary is from the upper one: when it is relatively near, the entrepreneur should use increasingly \textit{riskier} and more \textit{cost-effective} controls as the firm's performance metric improves; when it is relatively far away, the entrepreneur should shift towards less \textit{risky} and more \textit{cost-effective} controls as the firm's performance metric improves. The optimal policy is simple, intuitive, and can be easily understood by the entrepreneur.
        
\textbf{Free Lower Boundary Variation. } We also solve the free lower boundary analog of the multi-control problem, which represents a setting where the entrepreneur has the flexibility to decide when to abandon the start-up firm. We show that in the free lower boundary case, there emerges yet another type of \textit{efficient frontier} curve that is different from the two types in the fixed lower boundary case, based on which we characterize the optimal policy. In addition, we provide tight lower and upper bounds for the optimal lower boundary, which we show can be characterized by the \textit{effective-drift} of just two controls. This can help provide guidance for the entrepreneur to set the lower boundary at an appropriate place.

\textbf{Organization of the Paper. } The remainder of the paper is organized as follows. \S \ref{sec:related lit} discusses the related literature and \S \ref{sec:Model Formulation} presents the model formulation. \S \ref{sec: prelim def} introduces some necessary definitions and provides numerical examples on why the intuitions from the $2$-control setting may not hold in the multi-control one. \S \ref{sec:FxLB} presents the optimal policy for the fixed lower boundary case and provides the intuition behind it. \S \ref{sec:FrLB} presents the optimal policy for the free lower boundary variation and contrasts it with the fixed lower boundary one. \S  \ref{sec:Implications} summarizes the key insights for entrepreneurs and investors, and finally \S \ref{sec:conclusion} presents the concluding remarks.

\section{Related literature}
\label{sec:related lit}

Continuous-time stochastic models with multiple ($\ge 3$) controls are typically intractable to deal with, and very few papers are able to provide a closed-form characterization of the optimal policy under fully general model parameters. To our knowledge, only two such papers exist in the Operations Management community: [\cite{radner1996risk}]  and [\cite{harrison2015investment}]. [\cite{radner1996risk}] solves the optimal operating strategy for a dividend-paying firm whose net earnings are uncertain and evolve as a Brownian motion.  [\cite{harrison2015investment}] studies a firm that can employ different controls to obtain information about project prospects with a continuous-time Bayesian framework. Similar to these two papers, we solve the problem by explicitly constructing an \textit{efficient frontier} curve, which specifies the set of controls that may be used in the optimal policy. Differentiating from the two papers, our formulation gives rise to a richer structure of the optimal policy. This is mainly manifested in two ways. First, depending on the model parameters, the \textit{efficient frontier} curves can be of different types, resulting in qualitatively different structures of the optimal policy. Second, within the same type of \textit{efficient frontier}, the first control used in the optimal control sequence (order arrangement of the controls as the state increases) may not be fixed. In fact, in our formulation, the first control can be \textit{any} one on the \textit{efficient frontier}. This results in a more complex yet richer structure of the optimal policy and is in stark contrast with [\cite{radner1996risk}] and [\cite{harrison2015investment}] (e.g. for [\cite{harrison2015investment}], the first control used in the optimal policy is always the one with the highest cost-to-information quality ratio). 


A closely related study, [\cite{wang2022new}], models a new venture with diffusion control. This work can be viewed as a generalization with multiple controls. The generalization is highly non-trivial. Relative to [\cite{wang2022new}], the current paper demonstrates the limitation of their results and show that their insights may be misleading in the multi-control setting (see \S\ref{sec:mtvEgTwoCtrlInsightsFail}). The current paper also differentiates from [\cite{wang2022new}] by explicitly introducing a new metric, the \textit{cost-effectiveness}. This new metric possesses a highly intuitive interpretation, helps provide valuable insights into the optimal policy’s structure and illustrates a clear trade-off among each control’s characteristics (namely drift, variance and cost). In addition, the current paper utilizes entirely different and novel proof techniques that involve prescribing the \textit{efficient frontier} set of optimal controls.

Recent works in Operations Management have used drift-variance diffusion control to study decision-making under uncertainty, mainly in investment behavior and product launch. Investment-focused works include [\cite{kwon2011acquisition}], which examines firm exit decisions, and [\cite{sunar2021competitive}], which explores investment timing and size. Product launch studies include [\cite{lobel2016optimizing}], which optimizes launch time and pricing, [\cite{sunar2019optimal}], which analyzes launch timing and effort in influencing technological evolution, and [\cite{araman2022diffusion}], which studies the pre-launch assortment selection. Other applications include inventory system management ([\cite{araman2009dynamic}], [\cite{wu2014optimal}]), healthcare diagnostic services ([\cite{alizamir2013diagnostic}]) or clinical trials ([\cite{wang2020adaptive}]), effort allocation or incentive design ([\cite{chen2020optimal}], [\cite{li2024leveling}]), creator economy ([\cite{ma2025user}], [\cite{wang2025dynamic}]), Bayesian learning ([\cite{kuang2024weak}], [\cite{sunar2024optimal}]), resource allocation ([\cite{ata2019dynamic}]), customer churn prevention ([\cite{kanoria2023managing}]), and operational risk management ([\cite{kim2024operational}]). These works justify our modeling approach and highlight the challenges of solving drift-variance diffusion control problems. 

Our work also draws inspiration from the literature on lean startups and milestone financing. Specifically, under the lean startup framework of venture creation (see [\cite{blank2013lean}], [\cite{yoo2021theoretical}]), a start-up firm develops prototypes of increasing sophistication, receives responses from customers by introducing these prototypes to the market, and then confirms the product prospects once these prototypes gain sufficient traction from customers. In this regard, our model offers a mathematical abstraction of the problem of deciding between different prototypes. Another inspiration for our formulation comes from the literature on milestone financing (see [\cite{crama2008milestone}], [\cite{besbes2012dynamic}]), where firms either receive payment or incur penalties when they succeed or fail to reach a pre-specified milestone. They help motivate the modeling approach in this paper.

\section{Model Formulation}
\label{sec:Model Formulation}

We study a continuous-time model of a start-up firm, where its state is represented by a diffusion process $X(t)$ and an entrepreneur controls how the process evolves. We let $(\Omega, \mathcal{F}, \mathcal{P})$ denote a probability space upon which a standard Brownian Motion $B = \{B(t)\}_{t \ge 0}$ is defined, and let $\mathcal{F}_t$ denote the smallest sub-$\sigma$-algebra of $\mathcal{F}$ such that $\{B(s)\}_{0 \le s \le t}$ is measurable. At every $t \ge 0$, the entrepreneur chooses between multiple controls $i \in \I = \{1,2,3,...,N\}$, where control $i$ is associated with a cost rate $c(i) > 0$, an infinitesimal drift $\mu(i) > 0$ and a volatility $\sigma(i) > 0$. For the convenience of exposition, we will also use the notation $c(i)=c_i$, $\mu(i)=\mu_i$, $\sigma(i)=\sigma_i$, $\Lambda = \{c(i): i \in \I \}$, $\Theta = \{ \mu(i): i \in \I \}$ and $\Psi = \{ \sigma(i): i \in \I \}$. A control policy is defined as a right continuous process $\mathcal{A}=\{\mathcal{A}(t)\}_{t\ge 0}$ on $(\Omega, \mathcal{F}, \mathcal{P})$ that takes value in $\I$. Given a particular control policy $\mathcal{A}(t)$, the process $X(t)$ evolves in the following manner:
\begin{equation}
\label{eq:SDE}
dX(t) = \mu_{\mathcal{A}(t)} dt + \sigma_{\mathcal{A}(t)} dB(t).
\end{equation}

The entrepreneur receives a reward $\MNEW$ when $X(t)$ hits an upper boundary $\UNEW$. The upper boundary represents the milestone target, which is usually pre-determined and commonly assumed as an exogenous variable (see for example [\cite{besbes2012dynamic}]). To determine when the entrepreneur will abandon the startup, our model includes a lower boundary upon hitting which the entrepreneur will exit. We take two approaches regarding this boundary: for the fixed lower boundary case (\S\ref{sec:FxLB}), we treat the lower boundary as a fixed variable that is predetermined; for the free lower boundary extension (\S\ref{sec:FrLB}), we consider the lower boundary as a variable that the entrepreneur can adjust and optimize as necessary. 

The process $X(t)$ terminates when it hits either the lower boundary or the upper one (see Figure \ref{fig:diffProcessHitUpper}). Let $T  = \inf \{t \mid X(t) \not\in (\LNEW,\UNEW) \}$ denote the (random) stopping time. Let $\{\mathcal{F}_t^X\}_{t \ge 0}$ be the filtration generated by $\{X(t)\}_{t \ge 0}$ and $\mathcal{A} = \{\mathcal{A}(t)\}_{t \ge 0}$ is called an admissible control policy if it is adapted to $\{\mathcal{F}_t^X\}_{t \ge 0}$. For the remainder of the paper, we only focus on control policies that are admissible. The entrepreneur wants to find an admissible control policy $\{\mathcal{A}(t)\}_{t\ge 0}$ that maximizes the expected payoff
\begin{equation}
\label{func:objFunc}
V(x) = \E \left[  \int_{0}^{T} -c\left( \mathcal{A}(t) \right)  dt + \MNEW \mathbf{1}_{ \{X(T) = \UNEW \} } \bigg| X(0)=x 
\right].
\end{equation}

\begin{figure}[h]
    \centering
    \begin{subfigure}[b]{0.43\textwidth}
        \includegraphics[scale=0.55]{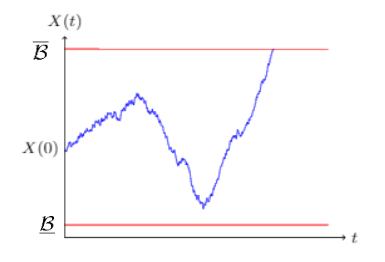}
        \caption{If the diffusion process hits $\UNEW$ first, the entrepreneur is rewarded with $\MNEW$.}
    \end{subfigure}
    \quad
    \quad
    \begin{subfigure}[b]{0.43\textwidth}
        \includegraphics[scale=0.55]{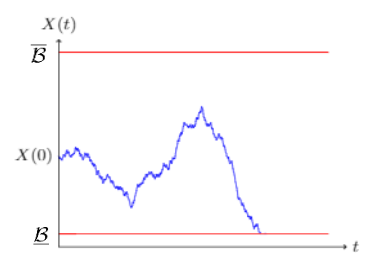}
        \caption{If the diffusion process hits $\LNEW$ first, the entrepreneur receives nothing.}
    \end{subfigure}
 \caption{Outcomes of the diffusion process when it hits the boundaries.}
 \label{fig:diffProcessHitUpper}
\end{figure}


\subsection{Discussion about Model Assumptions} 

The process described by (\ref{eq:SDE}) provides an abstraction of the venture creation process. The state variable $X(t)$ represents a measure of the value of the venture start-up. The controls represent activities that the start-up firm takes in the process of creating value. Such mathematical abstraction using a Brownian motion whose drift and variance depend on the control is a standard modeling technique commonly used in the Operations Management community, see for example [\cite{lobel2016optimizing}],  [\cite{sunar2019optimal}],  [\cite{kim2022investment}], [\cite{kanoria2023managing}].


In addition, we assume that the reward is $0$ when $X(t)$ hits the lower boundary $\LNEW$. This assumption is without loss of generality. If there is a non-zero reward $\tilde{\mathcal{R}}$ of hitting $\LNEW$, we can replace $\tilde{\mathcal{R}}$ by $0$ and $\MNEW$ with $\MNEW-\tilde{\mathcal{R}}$.
 
We assume that the reward $\MNEW$ is deterministic, which is also without loss of generality. The results of this paper do not change if the deterministic reward is replaced by a random variable with mean $\MNEW$.

In our model, we do not assume discounting, which imposes a time-based penalty for obtaining the reward. This is because our model focuses on the startup phase, which is typically a relatively short duration (ranging from one to two years). To account for the time-based effect of penalty, we note that we can incorporate a cost of delay into the model formulation. In particular, a cost of delay per unit time $e$ can be introduced, i.e. if the stopping time is $T$ then there is an additional penalty of $eT$. In this case, we can replace the cost $c_i$ by $c_i+e$ for all controls $i \in \I$ and the results of our model remain valid.


\subsection{Summary of Model Elements}

To summarize, for the fixed lower boundary problem, the following quantities are given:
\begin{itemize}
\item $\{\mu_i, \sigma_i, c_i\}_{i=1}^N$,
\item $\LNEW$, $\UNEW > 0$,
\item $\MNEW > 0$.
\end{itemize}
For the free lower boundary problem, the following quantities are given:
\begin{itemize}
\item $\{\mu_i, \sigma_i, c_i\}_{i=1}^N$,
\item $\UNEW > 0$,
\item $\MNEW > 0$,
\end{itemize}
where the lower boundary $\LNEW$ has become a decision variable.

\section{Preliminaries}
\label{sec: prelim def}

Since the state $X(t)$ summarizes all the information needed for purposes of decision-making, we restrict the policies under consideration to be Markovian (i.e. $\cal{A}(t)$ is a function of $X(t)$), in the forms of interval policies.

\begin{definition}
\label{def:kItvlPlc}
A set $\pi = \{(i_j,I_j)\}_{j=1}^k $ is a \textit{$k$-interval policy} for the fixed lower boundary problem, if the followings hold:

(i) $\forall j \in \{1,...,k\}, i_j \in \I$,

(ii) $\forall j \in \{1,...,k\}$, $I_j = [S_{j-1},S_j)$, where $\LNEW=S_0<S_1<...< S_{k-1}<S_k = \UNEW$,

(iii)  $\forall j \in \{1,...,k-1\}, i_j \not= i_{j+1}$.

Similarly, we will say a set $\hat{\pi} = \{(i_j,I_j)\}_{j=1}^k \cup \{\LNEW\}$ is a $k$-interval policy for the free lower boundary problem if (i), (ii) and (iii) above are satisfied. We will also refer a $k$-interval policy to be \it{static} if $k = 1$, and \it{dynamic} if $k \ge 2$.
\newline
\end{definition}

Essentially for a $k$-interval policy, $i_1$ is the control used when the state $X(t) \in [S_0, S_1)$, $i_2$ is the control used when the state $X(t) \in [S_1, S_2)$, and so on. Figure \ref{fig:k_intervalPolicies_def} presents examples of $k$-interval policies for $k=1,2,3,4$: the quantities below the axis represent switching or terminating boundaries and those above the axis represent the index of the control used under this policy when $X(t)$ is between these boundaries. We have two remarks. First, each disjoint interval's length must be strictly positive and neighboring intervals must involve different controls. Second, for the convenience of exposition we will refer to the policies in (a) ``3 only'', ``1 only''; in (b) ``31'', ``12''; and in (c) ``123'', ``2131''.

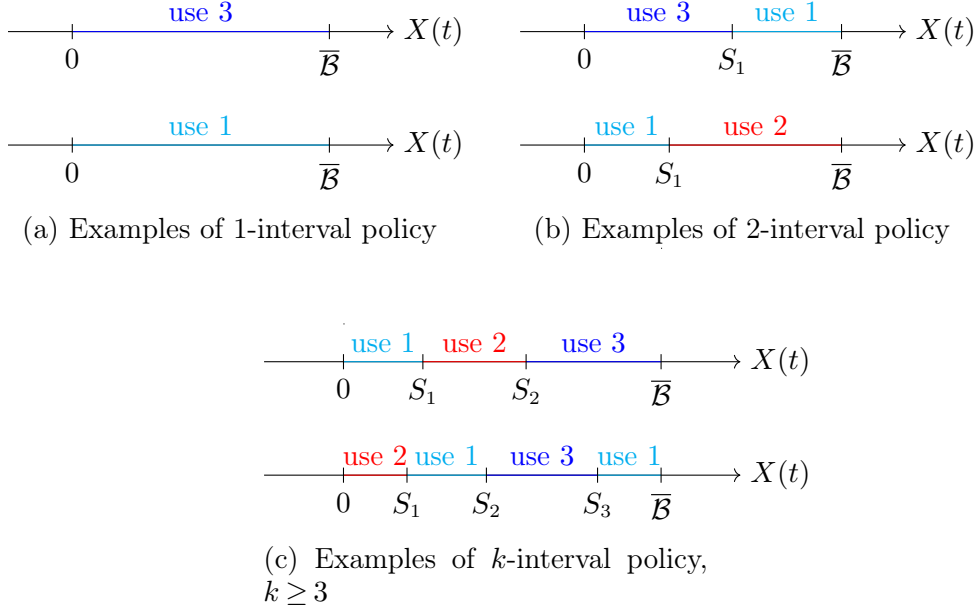
\begin{figure}[h]
    \centering
    
        \begin{subfigure}[b]{0.4\textwidth}
    \begin{tikzpicture}[xscale = 1.7]
    
        \draw[->] (-0.5,1.5) -- (2.5,1.5) node[pos=1, right, black] {$X(t)$};
\draw[-, blue] (0, 1.5)-- (2, 1.5)
node[pos=0.5, above] {use $3$};
\draw[-] (0,1.5+0.1) -- (0,1.5-0.1) node[pos=1, below, black] {$0$};
\draw[-] (2,1.5+0.1) -- (2,1.5-0.1) node[pos=1, below, black] {$\UNEW$};

\draw[->] (-0.5,0) -- (2.5,0) node[pos=1, right, black] {$X(t)$};
\draw[-, cyan] (0, 0)-- (2, 0)
node[pos=0.5, above] {use $1$};

\draw[-] (0,0.1) -- (0,-0.1) node[pos=1, below, black] {$0$};

\draw[-] (2,0.1) -- (2,-0.1) node[pos=1, below, black] {$\UNEW$};

    \end{tikzpicture}
    
    \caption{Examples of 1-interval policy}
    \end{subfigure}
    \qquad
    \begin{subfigure}[b]{0.4\textwidth}
    \begin{tikzpicture}[xscale = 1.7]
    
    \draw[->] (-0.5,1.5) -- (2.5,1.5) node[pos=1, right, black] {$X(t)$};
\draw[-, blue] (0, 1.5)-- (1.15, 1.5)
node[pos=0.5, above] {use $3$};
\draw[-, cyan] (1.15, 1.5) -- (2, 1.5)
node[pos=0.5, above] {use $1$};
\draw[-] (0,1.5+0.1) -- (0,1.5-0.1) node[pos=1, below, black] {$0$};
\draw[-] (1.15,1.5+0.1) -- (1.15,1.5-0.1) node[pos=1, below, black] {$S_1$};
\draw[-] (2,1.5+0.1) -- (2,1.5-0.1) node[pos=1, below, black] {$\UNEW$};

\draw[->] (-0.5,0) -- (2.5,0) node[pos=1, right, black] {$X(t)$};
\draw[-, cyan] (0, 0)-- (0.66, 0)
node[pos=0.5, above] {use $1$};
\draw[-, red] (0.66, 0) -- (2, 0)
node[pos=0.5, above] {use $2$};
\draw[-] (0,0.1) -- (0,-0.1) node[pos=1, below, black] {$0$};
\draw[-] (0.66,0.1) -- (0.66,-0.1) node[pos=1, below, black] {$S_1$};
\draw[-] (2,0.1) -- (2,-0.1) node[pos=1, below, black] {$\UNEW$};

    \end{tikzpicture}
    
    \caption{Examples of 2-interval policy}
    \end{subfigure}

        \begin{subfigure}[b]{0.4\textwidth}
    \begin{tikzpicture}[xscale = 2.1]
    \draw[-] (0,2) -- (0,2);
    \draw[->] (-0.5,1.5) -- (2.5,1.5) node[pos=1, right, black] {$X(t)$};
\draw[-, cyan] (0, 1.5)-- (0.5, 1.5) node[pos=0.5, above] {use $1$};
\draw[-, red] (0.5, 1.5)-- (1.15, 1.5) node[pos=0.5, above] {use $2$};
\draw[-, blue] (1.15, 1.5) -- (2, 1.5) node[pos=0.5, above] {use $3$};
\draw[-] (0,1.5+0.1) -- (0,1.5-0.1) node[pos=1, below, black] {$0$};
\draw[-] (0.5,1.5+0.1) -- (0.5,1.5-0.1) node[pos=1, below, black] {$S_1$};
\draw[-] (1.15,1.5+0.1) -- (1.15,1.5-0.1) node[pos=1, below, black] {$S_2$};
\draw[-] (2,1.5+0.1) -- (2,1.5-0.1) node[pos=1, below, black] {$\UNEW$};
\draw[-] (2,3) -- (2,3);

\draw[->] (-0.5,0) -- (2.5,0) node[pos=1, right, black] {$X(t)$};
\draw[-, red] (0, 0)-- (0.4, 0)
node[pos=0.5, above] {use $2$};
\draw[-, cyan] (0.4, 0) -- (0.9, 0)
node[pos=0.5, above] {use $1$};
\draw[-, blue] (0.9, 0) -- (1.6, 0)
node[pos=0.5, above] {use $3$};
\draw[-, cyan] (1.6, 0) -- (2.0, 0)
node[pos=0.5, above] {use $1$};
\draw[-] (0,0.1) -- (0,-0.1) node[pos=1, below, black] {$0$};
\draw[-] (0.4,0.1) -- (0.4,-0.1) node[pos=1, below, black] {$S_1$};
\draw[-] (0.9,0.1) -- (0.9,-0.1) node[pos=1, below, black] {$S_2$};
\draw[-] (1.6,0.1) -- (1.6,-0.1) node[pos=1, below, black] {$S_3$};
\draw[-] (2,0.1) -- (2,-0.1) node[pos=1, below, black] {$\UNEW$};

    \end{tikzpicture}
    
    \caption{Examples of $k$-interval policy, $k \ge 3$}
    \end{subfigure}
    \caption{Different types of interval-policies}
    \label{fig:k_intervalPolicies_def}
\end{figure}

To figure out which controls will be used in the optimal policy for the multi-control setting, we will need to define a few more useful quantities.

\begin{definition}
The \textit{effective-drift (ED)} $\aNEW_i$, for each control $i \in \I$, is defined as 
\begin{equation*}
\aNEW_i= 2 \mu_i/\sigma_i^2,
\end{equation*}
\end{definition}
We will call a control with a relatively large \textit{ED} a \textit{safe} control, a control with a relatively small \textit{ED} a \textit{risky} control. Intuitively, using a \textit{safer} control will make the process $X(t)$ advance towards the upper boundary with more certainty (the precise statement is shown by Lemma \ref{lem:HighEDHighProbHitU} in the Appendix).

\begin{definition}
\label{def:CE}
The \textit{cost-effectiveness (CE)} $CE_i$, for each control $i \in \I$, is defined as 
\begin{equation*}
CE_i= \mu_i/c_i.
\end{equation*}
\end{definition}

A control with a higher \textit{CE} means that: for per unit of cost, the control achieves higher drift. Intuitively, if all controls have $\sigma$ values that are very small and the reward $\MNEW$ is large, then controls with a higher \textit{CE} will be favored more. 
\newline

\textbf{\textit{ED} \& \textit{CE} are sufficient statistics. } Although a control is specified by three quantities: drift, variance and cost, it essentially can be represented just by its \textit{effective-drift} and \textit{cost-effectiveness}. In other words, if two controls $i$ and $j$ have $\aNEW_i = \aNEW_j$ and $CE_i = CE_j$, then the two controls are equivalent (in the sense that the resulting value function is the same). Below we provide a high-level intuition. 

\textbf{Intuition (Law of Terminal State).} Given that we are at state $X(t) = x_0$, and suppose control $i$ has parameters $(\mu_i, \sigma_i, c_i)$ and control $j$ has parameters $(\mu_j, \sigma_j, c_j)$ and they satisfy $\frac{\mu_i}{\sigma_i^2} = \frac{\mu_j}{\sigma_j^2}$ and $\frac{\mu_i}{c_i} = \frac{\mu_j}{c_j}$. Then if we use control $i$ for $\tau$ unit of time, then the terminal state $X(t+\tau) \sim N(x_0 + \mu_i \tau, \sigma_i^2 \tau)$ and we incur a total cost of $c_i \tau$. If we use control $j$ for $\frac{\mu_i}{\mu_j} \tau$ unit of time, then the terminal state $X(t+ \frac{\mu_i}{\mu_j} \tau) \sim N(x_0 + \mu_j \frac{\mu_i}{\mu_j} \tau, \sigma_j^2 \frac{\mu_i}{\mu_j} \tau ) = N(x_0 + \mu_i \tau, \sigma_i^2 \tau)$ and we incur a total cost of  $c_j \frac{\mu_i}{\mu_j} \tau = c_i \tau$, i.e. the distribution of the terminal state and the total cost incurred are the same. In other words, using control $i$ for $\tau$ unit of time is equivalent to using control $j$ for $\frac{\mu_i}{\mu_j} \tau$ unit of time. 


\begin{definition}
\label{def:h}
The \textit{scaled cost} $\hNEW_i$, for each control $i \in \I$, is defined as 
\begin{equation*}
\hNEW_i \triangleq \aNEW_i / CE_i = \frac{2 c_i}{\sigma_i^2}.
\end{equation*}
\end{definition}

The scaled cost of a control can be thought of as a scaled version of the cost that corresponds to the \textit{effective-drift}. It can also be interpreted as cost relative to volatility. In particular, a control's \textit{cost-effectiveness} can be computed as either $\mu_i/c_i$ or $\aNEW_i/\hNEW_i$.

Since $\aNEW_i$ and $CE_i$ (or equivalently $\aNEW_i$ and $\hNEW_i$) are sufficient statistics for each control $i$, from now on we can represent any quantity specific to a control to be in terms of $\aNEW_i$ and $CE_i$ (or equivalently $\aNEW_i$ and $\hNEW_i$). To avoid unnecessary technical complications, we will assume the \textit{effective-drift}, \textit{cost-effectiveness} and \textit{scaled cost} are different for all controls throughout the whole paper, unless specified otherwise.

\section{Fixed Lower Boundary Case}
\label{sec:FxLB}

In this section, we present the optimal policy for the fixed lower boundary case. \S \ref{subsec:OptPlcStrtrFxLB} provides a high-level description of the optimal policy's general structure. \S \ref{subsec:EF_FxLB} presents the optimal policy in detail. In particular, we will see that depending on whether the lower boundary $\LNEW$ is close to or far away from $\UNEW$, only controls in a certain subset will be used in the optimal policy (and this subset will be specified by the \textit{efficient frontier}). \S\ref{subsec:myoPlcBad_FxLB} compares the optimal policy with the myopic one and shows that the latter can perform arbitrarily worse than the former (with respect to the milestone reward $\MNEW$).

\subsection{High-Level: The Structure of the Optimal Policy}
\label{subsec:OptPlcStrtrFxLB}

We begin by presenting a high-level description of the optimal policy's general structure in Proposition \ref{prop:Kctrls_FxLB}. We will then provide an interpretation of it and explain the intuition behind why the optimal policy has such a structure.

\begin{restatable}{proposition}{OptPlcStrtrFxLB}
\label{prop:Kctrls_FxLB}
Fix $\{c_i, \mu_i, \sigma_i\}_{i=1}^N > 0$, $\LNEW$, $\UNEW$ with $\LNEW < \UNEW$, $\MNEW > 0$, there is an optimal policy that is at most $N $-interval. In particular, there exist $\tK \le N$ and switch thresholds $\LNEW = S^*_0 < S^*_1 < ... < S^*_\tK = \UNEW$ such that the optimal policy uses control $i^*_j$ in the interval $[S^*_{j-1}, S^*_j)$, $j = 1,2,...,\tK$. Moreover define $\barLNEW \triangleq \UNEW- \MNEW \max_i{CE_i}$, we have

\textbf{Case (1)} if $\LNEW > \barLNEW$ then $i^*_1, i^*_2, ..., i^*_{\tK}$ satisfy
\begin{align}
    &\aNEW(i^*_1) > \aNEW(i^*_2) > ... > \aNEW(i^*_{\tK}), \label{ineq:orderCrt1_FixedLB_AtLst1Eff} \\
    &CE(i^*_1) < CE(i^*_2) < ... < CE(i^*_{\tK}), \label{ineq:orderCrt2_FixedLB_AtLst1Eff}
\end{align}

\textbf{Case (2)} if $\LNEW < \barLNEW$ then $i^*_1, i^*_2, ..., i^*_{\tK}$ satisfy
\begin{align}
    &\aNEW(i^*_1) < \aNEW(i^*_2) < ... < \aNEW(i^*_{\tK}),  \label{ineq:orderCrt1_FixedLB_AllIneff} \\
    &CE(i^*_1) < CE(i^*_2) < ... < CE(i^*_{\tK}). \label{ineq:orderCrt2_FixedLB_AllIneff}
\end{align}
\end{restatable}

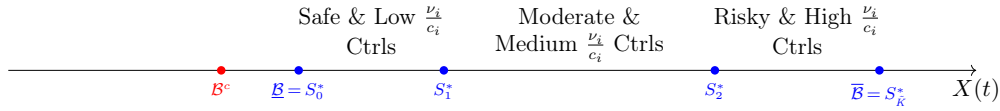
\begin{figure}[h]
    \centering
    \scalebox{0.8}
    {
    \begin{tikzpicture}[xscale = 3.2]
\draw[->] (0,0) -- (5,0)
node[pos=0.22,fill=red,circle, inner sep=0pt,minimum size=4pt,label=below:\textcolor{red}{\scriptsize{$\barLNEW$}}]{}
node[pos=0.3,fill=blue,circle, inner sep=0pt,minimum size=4pt,label=below:\textcolor{blue}{\scriptsize{$\LNEW=S^*_0$}}]{}
node[pos=0.375, above] {\begin{tabular}{@{}c@{}} Safe \& Low $\frac{\mu_i}{c_i}$  \\  Ctrls\end{tabular}}
node[pos=0.45,fill=blue,circle, inner sep=0pt,minimum size=4pt,label=below:\textcolor{blue}{\scriptsize{$S^*_1$}}]{}
node[pos=0.59, above] {\begin{tabular}{@{}c@{}} Moderate \& \\ Medium $\frac{\mu_i}{c_i}$     Ctrls\end{tabular}}
node[pos=0.73,fill=blue,circle, inner sep=0pt,minimum size=4pt,label=below:\textcolor{blue}{\scriptsize{$S^*_2$}}]{}
node[pos=0.815, above] {\begin{tabular}{@{}c@{}} Risky \& High $\frac{\mu_i}{c_i}$   \\  Ctrls\end{tabular}}
node[pos=0.9,fill=blue,circle, inner sep=0pt,minimum size=4pt,label=below:\textcolor{blue}{\scriptsize{$\UNEW=S^*_{\tK}$}}]{}
node[pos=1, below] {$X(t)$};
\end{tikzpicture}
}
    \caption{High-level structure of the optimal policy for case (1): $\LNEW > \barLNEW$. In general, the optimal policy uses relatively \textit{safer} and less \textit{cost-effective} controls near $\LNEW$, uses relatively \textit{riskier} and more \textit{cost-effective} controls near $\UNEW$ and uses moderate ones in the intermediate. }
    \label{fig:PlcItptFixedLB_AtLeast1CtrlEff}
\end{figure}

\begin{figure}[h]
    \centering
    \scalebox{0.8}
    {
\begin{tikzpicture}[xscale = 3.2]
\draw[->] (0,0) -- (5,0)
node[pos=0.22,fill=red,circle, inner sep=0pt,minimum size=4pt,label=below:\textcolor{red}{\scriptsize{$\barLNEW$}}]{}
node[pos=0.1,fill=blue,circle, inner sep=0pt,minimum size=4pt,label=below:\textcolor{blue}{\scriptsize{$\LNEW=S^*_0$}}]{}
node[pos=0.25, above] {\begin{tabular}{@{}c@{}} Risky \& Low $\frac{\mu_i}{c_i}$  \\  Ctrls\end{tabular}}
node[pos=0.4,fill=blue,circle, inner sep=0pt,minimum size=4pt,label=below:\textcolor{blue}{\scriptsize{$S^*_1$}}]{}
node[pos=0.55, above] {\begin{tabular}{@{}c@{}} Moderate \& Medium $\frac{\mu_i}{c_i}$   \\  Ctrls\end{tabular}}
node[pos=0.7,fill=blue,circle, inner sep=0pt,minimum size=4pt,label=below:\textcolor{blue}{\scriptsize{$S^*_2$}}]{}
node[pos=0.8, above] {\begin{tabular}{@{}c@{}} Safe \& High $\frac{\mu_i}{c_i}$   \\  Ctrls\end{tabular}}
node[pos=0.9,fill=blue,circle, inner sep=0pt,minimum size=4pt,label=below:\textcolor{blue}{\scriptsize{$\UNEW=S^*_{\tK}$}}]{}
node[pos=1, below] {$X(t)$};
\end{tikzpicture}
}

    \caption{High-level structure of the optimal policy for case (2): $\LNEW < \barLNEW$. In general, the optimal policy uses relatively \textit{riskier} and less \textit{cost-effective} controls near $\LNEW$, uses relatively \textit{safer} and more \textit{cost-effective} controls near $\UNEW$ and uses moderate ones in the intermediate. }
    \label{fig:PlcItptFixedLB_AllCtrlsIneff}
\end{figure}
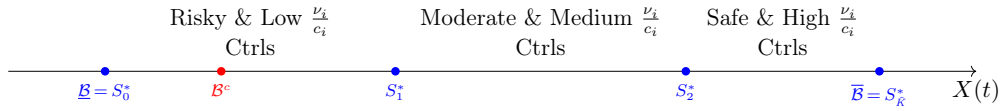

\textbf{Interpretation of Proposition \ref{prop:Kctrls_FxLB}. } Proposition \ref{prop:Kctrls_FxLB} presents the structure of the optimal policy for the fixed lower boundary case. The optimal policy is at most $N$-interval and its general structure depends on the relationship of $\LNEW$ with respect to $\barLNEW$. The general structure of the optimal policy for the case of $\LNEW > \barLNEW$ is demonstrated in Figure \ref{fig:PlcItptFixedLB_AtLeast1CtrlEff}, where the entrepreneur uses relatively \textit{safer} (higher \textit{effective-drift}) and less \textit{cost-effective} controls near the lower boundary $\LNEW$, uses relatively \textit{riskier} (lower \textit{effective-drift}) and more \textit{cost-effective} controls near the upper boundary $U$ and uses moderate ones in the intermediate. For instance, suppose there are $4$ controls (indexed $1,2,3,4$) with $\aNEW_1 < \aNEW_2 < \aNEW_3 < \aNEW_4$, $CE_1 < CE_2$ and $CE_2 > CE_3 > CE_4$. Then if $\LNEW > \barLNEW$, possible optimal candidate policies can be ``432'', ``43'', ``42'', ``32'', ``4 only'', ``3 only'' and  ``2 only'' (see Figure \ref{fig:pttlOptPlc_fixedLB}).  Note that we may be able to further rule out some of these candidate policies to be the optimal one, and we will present the details in \S \ref{subsec:EF_FxLB}.

In particular for the optimal policy in the case of $\LNEW > \barLNEW$, the controls are ordered by both their \textit{effective-drift} and \textit{cost-effectiveness}. The reader may wonder how the $2$-dimensional ordering is possible. For instance, what happens if there is a control whose \textit{effective-drift} and \textit{cost-effectiveness} are both the highest? The answer is that in such a scenario, no other controls can be ordered together with this control in the manner demonstrated in Figure \ref{fig:PlcItptFixedLB_AtLeast1CtrlEff} (i.e. Equations (\ref{ineq:orderCrt1_FixedLB_AtLst1Eff}) and (\ref{ineq:orderCrt2_FixedLB_AtLst1Eff})), and the optimal policy will be to use this control only. For example, suppose we have the $4$ controls described in Figure \ref{fig:pttlOptPlc_fixedLB} and we add in a control $5$ with $\aNEW_5 > \aNEW_4$ and $CE_5 > CE_2$, then the optimal policy must be ``5 only''.

\begin{figure}
    \captionsetup[subfigure]{labelformat=empty}
    \begin{subfigure}[b]{0.4\linewidth}
\centering
  \begin{tabular}{c|c|c}
 $\tK$ & \begin{tabular}{@{}c@{}}Possible\\Optimal Policy\end{tabular}  & Controls\\ \hline
  $3$ & ``432'' & $(i_1,i_2,i_3)=(4,3,2)$\\ \hline
    $2$ & ``43'' & $(i_1,i_2)=(4,3)$\\ \hline
$2$ & ``42'' & $(i_1,i_2)=(4,2)$\\ \hline
$2$ & ``32'' & $(i_1,i_2)=(3,2)$\\ \hline
$1$ & ``4 only'' & $i_1=4$\\ \hline
$1$ & ``3 only'' & $i_1=3$\\ \hline
$1$ & ``2 only'' & $i_1=2$\\ 
    
\end{tabular}
        \caption{\scriptsize{(a) Possible optimal policies for $\LNEW > \barLNEW$.}} \label{subfig:pttlOptPlc_atLeast1CtrlEff}
    \end{subfigure}
    \quad
    \quad
    \quad
    \quad
    \quad
    \begin{subfigure}[b]{0.4\linewidth}
      \centering
  \begin{tabular}{c|c|c}
 $\tK$ & \begin{tabular}{@{}c@{}}Possible\\Optimal Policy\end{tabular} & Controls\\ \hline
$2$ & ``12'' & $(i_1,i_2)=(1,2)$\\ \hline
$1$ & ``1 only'' & $i_1=1$\\ \hline
$1$ & ``2 only'' & $i_1=2$\\ 
    
\end{tabular}
        \caption{\scriptsize{(b) Possible optimal policies for $\LNEW < \barLNEW$.}} \label{subfig:pttlOptPlc_allCtrlsIneff}
    \end{subfigure}
\caption{Potential optimal policies when there are $4$ controls (indexed $1,2,3,4$) with $\aNEW_1 < \aNEW_2 < \aNEW_3 < \aNEW_4$, $CE_1 < CE_2$ and $CE_2 > CE_3 > CE_4$. }
\label{fig:pttlOptPlc_fixedLB}
\end{figure}

Similarly, the general structure of the optimal policy for the case of $\LNEW < \barLNEW$ is demonstrated in Figure \ref{fig:PlcItptFixedLB_AllCtrlsIneff}, where the entrepreneur uses relatively \textit{riskier} (lower \textit{effective-drift}) and less \textit{cost-effective} controls near the lower boundary $\LNEW$, uses relatively \textit{safer} (high \textit{effective-drift}) and more \textit{cost-effective} controls near the upper boundary $\UNEW$ and uses moderate ones in the intermediate. For instance, again consider the  $4$ controls in Figure \ref{fig:pttlOptPlc_fixedLB}. Then if $\LNEW < \barLNEW$, possible optimal policies can be ``12'', ``1 only'' and ``2 only''.

\textbf{Intuition of the Optimal Policy Structure. } Recall that $\barLNEW \triangleq \UNEW -\MNEW \max_i{CE_i}$. Intuitively, this quantity represents the break-even lower boundary threshold for the fluid analog of (\ref{eq:SDE}). Specifically, if the variance of the diffusion process is zero and the fixed lower boundary $\LNEW > \barLNEW$, the entrepreneur is guaranteed to have a positive total payoff anywhere between $\LNEW$ and $\UNEW$ with the most \textit{cost-effective} control.

Proposition \ref{prop:Kctrls_FxLB} essentially highlights the phase change of the optimal policy structure when the fixed lower boundary $\LNEW$ varies from above $\barLNEW$ to below it: in the case of $\LNEW > \barLNEW$, the entrepreneur uses relatively \textit{safe} controls near the lower boundary $\LNEW$ whereas in the case of $\LNEW < \barLNEW$ it is the reverse; however in both cases, the entrepreneur uses more \textit{cost-effective} controls near the upper boundary $\UNEW$. 

The intuition is as follows. In the former case where $\LNEW$ and $\UNEW$ are relatively near, the costs of the controls are essentially relatively small compared to the reward of hitting the upper boundary and the entrepreneur wants to aim for that boundary and avoid hitting the lower boundary. So when $X(t)$ is close to the lower boundary the entrepreneur wants to try as hard as possible to ``push'' the process away from it and thus uses the higher \textit{effective-drift} controls (even though they may be less cost-effective), because they are more likely to do so (see Lemma \ref{lem:HighEDHighProbHitU}). When $X(t)$ is close to the upper boundary $\UNEW$, all controls have a relatively high chance of hitting $\UNEW$, and hence the entrepreneur can afford to use a relatively more \textit{cost-effective} control that minimizes overall control costs. 

On the other hand in the latter case where $\LNEW$ and $\UNEW$ are relatively far, all controls become relatively costly compared to the reward, and the entrepreneur does not want to stay in the system for too long to avoid getting burnt by the control costs. The primary concern becomes the probability of hitting the \textit{nearest boundary}. When $X(t)$ is close to the $\LNEW$, the entrepreneur uses a relatively low \textit{effective-drift} control to achieve a higher probability of hitting the boundary. When $X(t)$ is close to the $\UNEW$, the entrepreneur uses a relatively high \textit{effective-drift} control because it not only is more \textit{cost-effective} but also achieves a higher chance of reaching the nearest boundary $\UNEW$.

We have now had an idea about the general structure of the optimal policy. To determine which controls are in the optimal policy, we will need a concept called the \textit{efficient frontier}, which specifies a subset of controls used in the optimal policy. \S \ref{subsec:EF_FxLB} presents this in further detail.

\subsection{Detailed Characterization of the Optimal Policy}
\label{subsec:EF_FxLB}

We are now in a position to present how we can obtain the exact optimal policy. Depending on whether $\LNEW > \barLNEW$ or $\LNEW < \barLNEW$, we will choose a subset of the controls and order them in terms of both \textit{effective-drift} and \textit{cost-effectiveness} according to (\ref{ineq:orderCrt1_FixedLB_AtLst1Eff}), 
(\ref{ineq:orderCrt2_FixedLB_AtLst1Eff}) or
(\ref{ineq:orderCrt1_FixedLB_AllIneff}), 
(\ref{ineq:orderCrt2_FixedLB_AllIneff}). Only controls in this subset will be used in the optimal policy, and they constitute what we call the \textit{efficient frontier}. 

A unique feature of our model is that the \textit{efficient frontier} can be of different types, depending on whether $\LNEW > \barLNEW$ or $\LNEW < \barLNEW$. Below, we characterize the type-I \textit{efficient frontier} for the case of $\LNEW > \barLNEW$ and the type-II one for the case of $\LNEW < \barLNEW$. \newline

\begin{figure}[h]
\FIGURE
{    
\begin{tikzpicture}[xscale = 1/11, yscale = 1/9]
\def\etaVec{{-3,0,10,20,30,40,50,60}}
\def\cVec{{11, 8, 9-3, 10-3, 12-3, 18-3, 28-3, 44-3}}
\def\Itcpt{6}
\draw[dashed] (-\Itcpt,0) -- ( \etaVec[4]*1.3+\Itcpt*1.3-\Itcpt, \cVec[4]*1.3);  
\foreach \i in {1,...,7}{
    \draw[-] (\etaVec[\i-1],\cVec[\i-1]) -- (\etaVec[\i],\cVec[\i]);
}
\path[ pattern= north east lines,pattern color = blue!30] ({-3,11}) -- ({0, 8}) -- ({10, 9-3}) -- ({20, 10-3}) -- ({30, 12-3}) --  ({40, 18-3}) --  ({50, 28-3}) --  ({60, 44-3}) --  ({60, 44}) --  ({-3, 44}) -- cycle;
\draw[dashed] ({-3,11}) --({-3, 44});
\draw[dashed] ({60, 44-3}) --({60, 44});
\foreach \i in {0}{
    \pgfmathtruncatemacro {\j}{\i-1}
    \filldraw[black] (\etaVec[\i],\cVec[\i]) circle (5pt) node[above right] {\scriptsize $(\aNEW_{\j},\hNEW_{\j})$};
}
\foreach \i in {1}{
    \pgfmathtruncatemacro {\j}{\i-1}
    \filldraw[black] (\etaVec[\i],\cVec[\i]) circle (5pt) node[below] {\scriptsize $(\aNEW_{\j},\hNEW_{\j})$};
}
\foreach \i in {2,...,3}{
    \pgfmathtruncatemacro {\j}{\i-1}
    \filldraw[black] (\etaVec[\i],\cVec[\i]) circle (5pt) node[above] {\scriptsize $(\aNEW_{\j},\hNEW_{\j})$};
}
\foreach \i in {4,...,7}{
    \pgfmathtruncatemacro {\j}{\i-1}
    \filldraw[black] (\etaVec[\i],\cVec[\i]) circle (5pt) node[below right] {\scriptsize $(\aNEW_{\j},\hNEW_{\j})$};
}
\def\addetaVec{{15,30.0,45}}
\def\addcVec{{25.3,21.2,31.6}}
\foreach \i in {0,...,2}{
    \pgfmathtruncatemacro {\j}{\i+7}
    \filldraw[black] (\addetaVec[\i],\addcVec[\i]) circle (5pt) node[above] {\scriptsize $(\aNEW_{\j},\hNEW_{\j})$};
}
\draw[-] (-\Itcpt,0) -- (60,0)    node[pos=1,label=below:\textcolor{black}{$\aNEW$}]{};
\draw[-] (-\Itcpt,0) -- (-\Itcpt,38) node[pos=1,label=above:\textcolor{black}{$\hNEW$}]{};
\end{tikzpicture}
}
{Control index labeling via the lower convex hull of all controls on the $(\aNEW,\hNEW)$-plane.  \label{fig:EF_FxLB_prep}}
{\scriptsize{In this example, there are a total of $11$ controls. $\Np = 6$, $\Nm = -1$, $\n = \argmax_i{CE_i} = 3$.}}
\end{figure}

\textbf{Control Index Labeling. } We arrange the controls on the $(\aNEW, \hNEW)$-plane and will label their indices as follows. First we obtain the lower convex hull (see Figure \ref{fig:EF_FxLB_prep}) induced by all the controls. Then, for the controls that reside on it, we denote control $1$ to be the one with the lowest scaled cost, i.e.
\begin{align}
    \hNEW_1 = \min_{i \in \I} \hNEW_i.
\end{align}
For the controls that are on the lower convex hull and are on the right of control $1$, we label them with index $2$, $3$, \dots, $\Np$ in increasing effective drifts (see Figure \ref{fig:EF_FxLB_prep}). Similarly, for the controls that are on the lower convex hull and are on the left of control $1$, we label them with index $0$, $-1$, \dots, $\Nm$ in decreasing effective drifts (see Figure \ref{fig:EF_FxLB_prep}). The reason of such indexing will become apparent later. Note that
\begin{align}
    \aNEW_{\Nm} = \min_{i \in \I} \aNEW_i, \text{ } \aNEW_{\Np} = \max_{i \in \I} \aNEW_i. 
\end{align}
For the remaining controls that are not on the lower convex hull, we arbitrarily label them with index $\Np+1$, $\Np+2$, \dots, etc. For the convenience of exposition, we will also denote control $\n$ to be the one with the highest \textit{cost-effectiveness}, i.e. 
\begin{align}
\n = \argmax_i{CE_i}. \label{eq:def_n_highestCE}
\end{align}
An example of the control index labeling is presented in Figure \ref{fig:EF_FxLB_prep}.

\textbf{Efficient Frontier for the case of $\LNEW > \barLNEW$. } The type-I \textit{efficient frontier}  consists of the controls $\n$, $\n+1$, ..., $\Np$ and they satisfy (see Figure \ref{fig:EF_FxLB}(a))
\begin{align}
& 0 < \aNEW_{\n} < \aNEW_{\n+1} ... < \aNEW_{\Np}, \text{ } 0 < \hNEW_{\n} < \hNEW_{\n+1} < ... < \hNEW_{\Np}, \label{ineq:EF_FxLB_1Eff1}\\
& 0 < \frac{\hNEW_{\n}}{\aNEW_{\n}} < \frac{\hNEW_{\n+1} - \hNEW_{\n}}{\aNEW_{\n+1} - \aNEW_{\n}} < \frac{\hNEW_{\n+2} - \hNEW_{\n+1}}{\aNEW_{\n+2} - \aNEW_{\n+1}} < ... < \frac{\hNEW_{\Np} - \hNEW_{\Np-1}}{\aNEW_{\Np} - \aNEW_{\Np-1}}, \label{ineq:EF_FxLB_1Eff2}\\
&\hNEW_i = \overline{\phi}(\aNEW_i) \text{, for $\n \le i \le \Np$,} \label{ineq:EF_FxLB_1Eff3}\\
&\hNEW_i \ge \overline{\phi}(\aNEW_i) \text{, for $i \le \n-1 $ or $i \ge \Np+1$,} \label{ineq:EF_FxLB_1Eff4}
\end{align}
where $\overline{\phi}(\cdot)$ is the strictly increasing, piece-wise linear and convex function that connects $(0,0), (\aNEW_{\n},\hNEW_{\n})$, $(\aNEW_{\n+1},\hNEW_{\n+1})$, ... , $(\aNEW_{\Np}, \hNEW_{\Np})$. Controls $\n$, $\n+1$, ... , $\Np$ are either the left endpoint, the right endpoint, or a point on $\overline{\phi}(\cdot)$ at which the slope changes. Denote $\overline{\Phi}$ to be the set of controls on $\overline{\phi}(\cdot)$, i.e. $\overline{\Phi} = \{\n,\n+1,...,\Np\}$, which represents the \textit{efficient frontier}  for the case of $\LNEW > \barLNEW$.

\begin{figure}[h]
    \captionsetup[subfigure]{labelformat=empty}
    \begin{subfigure}[b]{0.4\linewidth}
\centering
  \begin{tikzpicture}[xscale = 1/11, yscale = 1/9]
\def\etaVec{{-3,0,10,20,30,40,50,60}}
\def\cVec{{11, 8, 9-3, 10-3, 12-3, 18-3, 28-3, 44-3}}
\def\Itcpt{6}
\draw[-] (-\Itcpt,0) -- ( \etaVec[4], \cVec[4]);  
\foreach \i in {5,...,7}{
    \draw[-] (\etaVec[\i-1],\cVec[\i-1]) -- (\etaVec[\i],\cVec[\i]);
}
\foreach \i in {0}{
    \pgfmathtruncatemacro {\j}{\i-1}
    \filldraw[black] (\etaVec[\i],\cVec[\i]) circle (5pt) node[above right] {\scriptsize $(\aNEW_{\j},\hNEW_{\j})$};
}
\foreach \i in {1}{
    \pgfmathtruncatemacro {\j}{\i-1}
    \filldraw[black] (\etaVec[\i],\cVec[\i]) circle (5pt) node[below] {\scriptsize $(\aNEW_{\j},\hNEW_{\j})$};
}
\foreach \i in {2,...,3}{
    \pgfmathtruncatemacro {\j}{\i-1}
    \filldraw[black] (\etaVec[\i],\cVec[\i]) circle (5pt) node[above] {\scriptsize $(\aNEW_{\j},\hNEW_{\j})$};
}
\foreach \i in {4,...,7}{
    \pgfmathtruncatemacro {\j}{\i-1}
    \filldraw[black] (\etaVec[\i],\cVec[\i]) circle (5pt) node[below right] {\scriptsize $(\aNEW_{\j},\hNEW_{\j})$};
}
\def\addetaVec{{15,30.0,45}}
\def\addcVec{{25.3,21.2,31.6}}
\foreach \i in {0,...,2}{
    \pgfmathtruncatemacro {\j}{\i+7}
    \filldraw[black] (\addetaVec[\i],\addcVec[\i]) circle (5pt) node[above] {\scriptsize $(\aNEW_{\j},\hNEW_{\j})$};
}
\draw[-] (-\Itcpt,0) -- (60,0)    node[pos=1,label=below:\textcolor{black}{$\aNEW$}]{};
\draw[-] (-\Itcpt,0) -- (-\Itcpt,38) node[pos=1,label=above:\textcolor{black}{$\hNEW$}]{};
\end{tikzpicture}
        \caption{\scriptsize{(a) The \textit{efficient frontier}  (piecewise linear solid curve) for the case of $\LNEW > \barLNEW$: the piecewise linear function $\overline{\phi}(\cdot)$ that satisfies (\ref{ineq:EF_FxLB_1Eff3}) and (\ref{ineq:EF_FxLB_1Eff4}).}} \label{subfig:EF_FxLB_1Eff}
    \end{subfigure}
    \quad
    \quad
    \quad
    \quad
    \quad
    \begin{subfigure}[b]{0.4\linewidth}
      \centering
  \begin{tikzpicture}[xscale = 1/11, yscale = 1/9]
\def\etaVec{{-3,0,10,20,30,40,50,60}}
\def\cVec{{11, 8, 9-3, 10-3, 12-3, 18-3, 28-3, 44-3}}
\def\Itcpt{6}
\draw[dashed] (-\Itcpt,0) -- ( \etaVec[4], \cVec[4]);
\draw[-] (\etaVec[4], \cVec[4]) -- (\etaVec[4]*2+\Itcpt*2-\Itcpt, \cVec[4]*2); 
\foreach \i in {1,...,4}{
    \draw[-] (\etaVec[\i-1],\cVec[\i-1]) -- (\etaVec[\i],\cVec[\i]);
}
\foreach \i in {0}{
    \pgfmathtruncatemacro {\j}{\i-1}
    \filldraw[black] (\etaVec[\i],\cVec[\i]) circle (5pt) node[above right] {\scriptsize $(\aNEW_{\j},\hNEW_{\j})$};
}
\foreach \i in {1}{
    \pgfmathtruncatemacro {\j}{\i-1}
    \filldraw[black] (\etaVec[\i],\cVec[\i]) circle (5pt) node[below] {\scriptsize $(\aNEW_{\j},\hNEW_{\j})$};
}
\foreach \i in {2,...,4}{
    \pgfmathtruncatemacro {\j}{\i-1}
    \filldraw[black] (\etaVec[\i],\cVec[\i]) circle (5pt) node[above] {\scriptsize $(\aNEW_{\j},\hNEW_{\j})$};
}
\foreach \i in {5}{
    \pgfmathtruncatemacro {\j}{\i-1}
    \filldraw[black] (\etaVec[\i],\cVec[\i]) circle (5pt) node[above] {\scriptsize $(\aNEW_{\j},\hNEW_{\j})$};
}
\foreach \i in {6,...,7}{
    \pgfmathtruncatemacro {\j}{\i-1}
    \filldraw[black] (\etaVec[\i],\cVec[\i]) circle (5pt) node[below right] {\scriptsize $(\aNEW_{\j},\hNEW_{\j})$};
}
\def\addetaVec{{15,30.0,45}}
\def\addcVec{{25.3,21.2,31.6}}
\foreach \i in {0,...,2}{
    \pgfmathtruncatemacro {\j}{\i+7}
    \filldraw[black] (\addetaVec[\i],\addcVec[\i]) circle (5pt) node[above] {\scriptsize $(\aNEW_{\j},\hNEW_{\j})$};
}
\draw[-] (-\Itcpt,0) -- (60,0)    node[pos=1,label=below:\textcolor{black}{$\aNEW$}]{};
\draw[-] (-\Itcpt,0) -- (-\Itcpt,38) node[pos=1,label=above:\textcolor{black}{$\hNEW$}]{};
\end{tikzpicture}

        \caption{\scriptsize{(b) The \textit{efficient frontier}  (piecewise linear solid curve) for the case of $\LNEW < \barLNEW$: the piecewise linear function $\underline{\phi}(\cdot)$ that satisfies (\ref{ineq:EF_FxLB_AllIneff3}) and (\ref{ineq:EF_FxLB_AllIneff4}).}} \label{subfig:EF_FxLB_AllIneff}
    \end{subfigure}
\caption{The \textit{efficient frontier}  for $\LNEW > \barLNEW$ (left) and for $\LNEW < \barLNEW$ (right).}
\label{fig:EF_FxLB}
\end{figure}

\textbf{Efficient Frontier for the case of $\LNEW < \barLNEW$. } 
The type-II \textit{efficient frontier}  consists of the controls $\Nm$, ..., $1$, $2$, ..., $\n$ and they satisfy (see Figure \ref{fig:EF_FxLB}(b))
\begin{align}
&0 < \aNEW_1 < ... < \aNEW_{\n}, \text{ } 0 < \hNEW_1 < ... < \hNEW_{\n}, \label{ineq:EF_FxLB_AllIneff1}\\
& 0 < \aNEW_{\Nm} < \aNEW_{\Nm+1} < ... <  \aNEW_0 < \aNEW_1, \text{ } 0 < \hNEW_1 < \hNEW_0 < ... < \hNEW_{\Nm}, \label{ineq:EF_FxLB_AllIneff12}\\
& \frac{\hNEW_{\Nm+1} - \hNEW_{\Nm}}{\aNEW_{\Nm+1} - \aNEW_{\Nm}} ... < \frac{\hNEW_1 - \hNEW_0}{\aNEW_1 - \aNEW_0} < 0 < \frac{\hNEW_2 - \hNEW_1}{\aNEW_2 - \aNEW_1} < \frac{\hNEW_3 - \hNEW_2}{\aNEW_3 - \aNEW_2} < ... < \frac{\hNEW_{\n} - \hNEW_{\n-1}}{\aNEW_{\n} - \aNEW_{\n-1}}, \label{ineq:EF_FxLB_AllIneff2}\\
&\hNEW_i = \underline{\phi}(\aNEW_i) \text{, for $1 \le i \le \n$,} \label{ineq:EF_FxLB_AllIneff3}\\
&\hNEW_i \ge \underline{\phi}(\aNEW_i) \text{, for $i \ge  \n+1$,} \label{ineq:EF_FxLB_AllIneff4}
\end{align}
where $\underline{\phi}(\cdot)$ is the piece-wise linear and convex function that connects $(\aNEW_{\Nm}, \hNEW_{\Nm}), (\aNEW_1,\hNEW_1)$, $(\aNEW_2,\hNEW_2)$, ... , $(\aNEW_{\n},\hNEW_{\n})$. Controls $\Nm$, ... , $1$, $2$, ... , $\n$ are either the left endpoint, the right endpoint, or a point on $\underline{\phi}(\cdot)$ at which the slope changes. Denote $\underline{\Phi}$ to be the set of controls on $\underline{\phi}(\cdot)$, i.e. $\underline{\Phi} = \{\Nm, \Nm+1, ..., \n\}$.

\textbf{Statement of the Optimal Policy.} Having defined the \textit{efficient frontier}  for both cases of $\LNEW > \barLNEW$ and $\LNEW < \barLNEW$, we are now in a position to present the optimal policy. For either case, the optimal policy uses controls on the respective \textit{efficient frontier}  in a sequential and monotone manner as specified in Theorem \ref{thm:Kctrls_FxLB}. 

\begin{restatable}{theorem}{OptPlcFxLB}
\label{thm:Kctrls_FxLB}
For the fixed lower boundary problem, fix $\{c_i, \mu_i, \sigma_i\}_{i=1}^N > 0$, $\LNEW$, $\UNEW$ with $\LNEW < \UNEW$, $\MNEW > 0$, there is an optimal interval policy $\pi^* = \left\{\left(i_j^*,I_j^*=[S_{j-1}^*,S_j^*) \right)\right\}_{j=1}^\tK$, where $1 \le \tK \le \max\{\Np - \n + 1, \text{ } \n - \Nm + 1\}$. Moreover

\textbf{Case (1)} if $\LNEW > \barLNEW$ then $i_j^* \in \overline{\Phi}$ for $j = 1,2,...,\tK \le \Np - \n + 1$, and
\begin{align}
    & i_1^* \ge \n + (\tK - 1), \label{eq:FxLB_1Eff_ctrlOrder1} \\
    & i_{j+1}^* = i_1^* - j\text{, for $1 \le j \le \tK-1$,}\\
    &\aNEW(i_1^*) > \aNEW(i_2^*) > ... > \aNEW(i_{\tK}^*),\\ 
    &CE(i_1^*) < CE(i_2^*) < ... < CE(i_{\tK}^*); \label{eq:FxLB_1Eff_ctrlOrder4}
\end{align}

\textbf{Case (2)} if $\LNEW < \barLNEW$ then $i_j^* \in \underline{\Phi}$ for $j = 1,2,...,\tK \le \n - \Nm + 1$, and
\begin{align}
    & i_1^* \le \n - (\tK - 1), \label{eq:FxLB_AllIneff_ctrlOrder1}\\
    & i_{j+1}^* = i_1^* + j\text{, for $1 \le j \le \tK-1$,}\\
    &\aNEW(i_1^*) < \aNEW(i_2^*) < ... < \aNEW(i_{\tK}^*),  \\
    &CE(i_1^*) < CE(i_2^*) < ... < CE(i_{\tK}^*). \label{eq:FxLB_AllIneff_ctrlOrder4}
\end{align}
Furthermore, in both cases the value function $V(\cdot)$ is $\mathcal{C}^2$ over $[\LNEW,\UNEW)$ and satisfies
\begin{align}
&V(\LNEW) = 0, \text{ } V(\UNEW) = \MNEW, \label{eq:FxLB_suffCondOpt1}\\
& -c_i + \mu_i V'(x) + \frac{\sigma_i^2}{2}V''(x) \le 0 \text{, for all $x \in [\LNEW,\UNEW)$ and $i \in \I$,}\\
& -c_{i_j^*} + \mu_{i_j^*} V'(x) + \frac{\sigma_{i_j^*}^2}{2}V''(x) = 0 \text{, for $x \in [S_{j-1}^*,S_j^*)$, $1 \le j \le \tK$}. \label{eq:FxLB_suffCondOpt3}
\end{align}

\end{restatable}

\begin{figure}[h]
    \centering
    \scalebox{0.8}
    {
    \begin{tikzpicture}[xscale = 3.2]
    \def\Lbar{0.2}
    \def\L{0.25}
    \def\iOne{0.3}
    \def\SOne{0.35}
    \def\iTwo{0.435}
    \def\STwo{0.52}
    \def\DotDotDot{0.6}
    \def\SKminusOne{0.68}
    \def\iK{0.8}
    \def\U{0.92}
    \def\Xt{1}
\draw[->] (0,0) -- (5,0)
node[pos={\Lbar},fill=red,circle, inner sep=0pt,minimum size=4pt,label=below:\textcolor{red}{\scriptsize{$\barLNEW$}}]{}
node[pos=\L,fill=blue,circle, inner sep=0pt,minimum size=4pt,label=below:\textcolor{blue}{\scriptsize{$\LNEW=S^*_0$}}]{}
node[pos=\iOne, above] {use $i_1^*$}
node[pos=\SOne,fill=blue,circle, inner sep=0pt,minimum size=4pt,label=below:\textcolor{blue}{\scriptsize{$S^*_1$}}]{}
node[pos=\iTwo, above] {use $i_2^* = i_1^* - 1$}
node[pos=\STwo,fill=blue,circle, inner sep=0pt,minimum size=4pt,label=below:\textcolor{blue}{\scriptsize{$S^*_2$}}]{}
node[pos=\DotDotDot, above] {...}
node[pos=\SKminusOne,fill=blue,circle, inner sep=0pt,minimum size=4pt,label=below:\textcolor{blue}{\scriptsize{$S^*_{\tK - 1}$}}]{}
node[pos=\iK, above] {use $i_{\tK}^* = i_1^* - (\tK - 1)$}
node[pos=\U,fill=blue,circle, inner sep=0pt,minimum size=4pt,label=below:\textcolor{blue}{\scriptsize{$\UNEW=S^*_{\tK}$}}]{}
node[pos=\Xt, below] {$X(t)$};
\end{tikzpicture}
}
    \caption{Detailed characterization of the optimal policy $\pi^*$ for case (1): $\LNEW > \barLNEW$. As $X(t)$ increases, the optimal policy uses controls with decreasing \textit{effective-drift} in decreasing order in $\overline{\Phi}$.}
    \label{fig:OptPlcFxLB_1Eff}
\end{figure}

\begin{figure}[h]
    \centering
    \scalebox{0.8}
    {
\begin{tikzpicture}[xscale = 3.2]
    \def\Lbar{0.2}
    \def\L{0.1}
    \def\iOne{0.175}
    \def\SOne{0.25}
    \def\iTwo{0.355}
    \def\STwo{0.46}
    \def\DotDotDot{0.56}
    \def\SKminusOne{0.66}
    \def\iK{0.79}
    \def\U{0.92}
    \def\Xt{1}
\draw[->] (0,0) -- (5,0)
node[pos=\Lbar,fill=red,circle, inner sep=0pt,minimum size=4pt,label=below:\textcolor{red}{\scriptsize{$\barLNEW$}}]{}
node[pos=\L,fill=blue,circle, inner sep=0pt,minimum size=4pt,label=below:\textcolor{blue}{\scriptsize{$\LNEW=S^*_0$}}]{}
node[pos=\iOne, above] {use $i_1^*$}
node[pos=\SOne,fill=blue,circle, inner sep=0pt,minimum size=4pt,label=below:\textcolor{blue}{\scriptsize{$S^*_1$}}]{}
node[pos=\iTwo, above] {use $i_2^* = i_1^* + 1$}
node[pos=\STwo,fill=blue,circle, inner sep=0pt,minimum size=4pt,label=below:\textcolor{blue}{\scriptsize{$S^*_2$}}]{}
node[pos=\DotDotDot, above] {...}
node[pos=\SKminusOne,fill=blue,circle, inner sep=0pt,minimum size=4pt,label=below:\textcolor{blue}{\scriptsize{$S^*_{\tK - 1}$}}]{}
node[pos=\iK, above] {use $i_{\tK}^* = i_1^* + (\tK - 1)$}
node[pos=\U,fill=blue,circle, inner sep=0pt,minimum size=4pt,label=below:\textcolor{blue}{\scriptsize{$\UNEW=S^*_{\tK}$}}]{}
node[pos=\Xt, below] {$X(t)$};
\end{tikzpicture}
}

    \caption{Detailed characterization of the optimal policy $\pi^*$ for case (2): $\LNEW < \barLNEW$. As $X(t)$ increases, the optimal policy uses controls with increasing \textit{effective-drift} in increasing order in $\underline{\Phi}$.}
    \label{fig:OptPlcFxLB_AllIneff}
\end{figure}
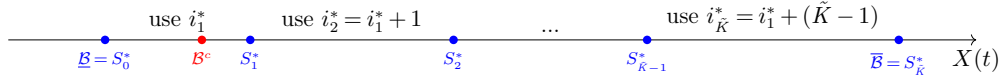

Theorem \ref{thm:Kctrls_FxLB} provides a more comprehensive and detailed description of Proposition \ref{prop:Kctrls_FxLB}. Equations (\ref{eq:FxLB_1Eff_ctrlOrder1}) - (\ref{eq:FxLB_1Eff_ctrlOrder4}) present the optimal policy for case (1). The optimal action is to use control $i_1^* \ge \n + (\tK - 1)$ when the state $X(t)$ is in the interval $[S_0^*, S_1^*)$, use control $i_1^* - 1$ when $X(t)$ is in $[S_1^*, S_2^*)$, use control $i_1^* - 2$ when $X(t)$ is in $[S_2^*, S_3^*)$ and so on (see Figure \ref{fig:OptPlcFxLB_1Eff}). As $X(t)$ increases, the optimal policy uses controls with decreasing \textit{effective-drift} in decreasing order in the \textit{efficient frontier}  $\overline{\Phi}$. The controls used all have \textit{effective-drift} at least as high as control $\n$, i.e. $i_{\tK}^* \ge \n$.

Equations (\ref{eq:FxLB_AllIneff_ctrlOrder1}) - (\ref{eq:FxLB_AllIneff_ctrlOrder4}) present the optimal policy for case (2). The optimal action is to use control $i_1^* \le \n - (\tK - 1)$ when $X(t)$ is in $[S_0^*, S_1^*)$, use control $i_1^* + 1$ when $X(t)$ is in $[S_1^*, S_2^*)$, use control $i_1^* + 2$ when $X(t)$ is in $[S_2^*, S_3^*)$ and so on (see Figure \ref{fig:OptPlcFxLB_AllIneff}). As $X(t)$ increases, the optimal policy uses controls with increasing \textit{effective-drift} in increasing order in the \textit{efficient frontier}  $\underline{\Phi}$. The controls used all have \textit{effective-drift} not exceeding that of control $\n$, i.e. $i_{\tK}^* \le \n$. 

In both cases however, as $X(t)$ increases, the optimal policy uses controls with increasing \textit{cost-effectiveness} ((\ref{eq:FxLB_1Eff_ctrlOrder4}) and (\ref{eq:FxLB_AllIneff_ctrlOrder4})). This is because as $X(t)$ reaches near the upper boundary, the entrepreneur becomes increasingly less worried about hitting the lower boundary and can afford to use controls with a higher \textit{cost-effectiveness}.

We note that for either case, we can easily identify the controls on the respective \textit{efficient frontier}. To find the actual controls used in the optimal policy, or the set of switching thresholds $S_j^*$, we refer the reader to the numerical procedure in \S \ref{sec:PfThmFxLB_case1} and \S \ref{sec:PfThmFxLB_case2} in the Appendix. In the numerical procedure, an explicit policy will be constructed to satisfy equations (\ref{eq:FxLB_suffCondOpt1}) - (\ref{eq:FxLB_suffCondOpt3}), which are essentially sufficient conditions for optimality.

\subsection{Comparison with the Myopic Policy}
\label{subsec:myoPlcBad_FxLB}

In some scenarios entrepreneurs may resort to simple decisions that maximize only the short-term benefit ([\cite{debrulle2020entrepreneurs}]). The natural kind of such policy is the \textit{myopic policy}, which is the simple static policy that always uses the control with the best \textit{cost-effectiveness}, i.e. for per unit of cost this control achieves the highest drift upwards. In other words (c.f. Definition \ref{def:kItvlPlc}),
\begin{align}
    \pi_{myo} = \{(\n, [\LNEW,\UNEW)) \},
\end{align}
where $\n = \argmax_i{CE_i}$. 

While the myopic policy is desirable due to its simplicity, the entrepreneur should be cautious with it because there exist scenarios where the myopic policy can perform significantly worse than the optimal policy. This is shown in Proposition \ref{prop:FxLBMyoPlcBad}.

\begin{restatable}{proposition}{FxLBMyoPlcBad}
\label{prop:FxLBMyoPlcBad}
For the fixed lower boundary problem, let $\LNEW < \UNEW$, $\MNEW > 0$ be fixed. For any $\epsilon > 0$, there exists a problem instance such that 
\begin{align}
    \sup_{x \in (\LNEW,\UNEW)} \frac{V_{opt}(x) - V_{myo}(x)}{\MNEW} > 1 - \epsilon,
\end{align}
where $V_{opt}(\cdot)$ and $V_{myo}(\cdot)$ denote the value function of the optimal and myopic policy, respectively. 
\end{restatable}

Essentially, the difference between the optimal policy and the myopic policy can approach arbitrarily close to the milestone reward $\MNEW$. Figure \ref{fig:myoPlcBad_1Eff_Plc} provides a numerical example that illustrates when this difference can be large. In this example, $\LNEW=0$, $\UNEW=10$, $\MNEW=10$ and there are three controls $(\aNEW_1,\hNEW_1) = (0.2, 0.01)$, $(\aNEW_2,\hNEW_2) = (1, 0.06)$, $(\aNEW_3,\hNEW_3) = (4, 0.3)$. The myopic policy $\pi_{myo}$ is to always use control $1$ (highest \textit{cost-effectiveness}), and the optimal policy is ``321''. Figure \ref{fig:MyoPlcBad_1Eff_VF} plots the value function of both policies. The biggest gap between the two value functions is $7.4$, which occurs when the state $X(t)$ is at $0.08$. In particular, this is because when $X(t)$ is close to the lower boundary $\LNEW$, the optimal policy uses the extremely \textit{safe} control $3$ and when $X(t)$ moves away from $\LNEW$, it uses the other two controls with higher \textit{cost-effectiveness} (use control $2$ at $[S_1,S_2)$ and control $1$ at $[S_2,\UNEW)$). Essentially, this guarantees that even when $X(t)$ is close to $\LNEW$ (but not too close), it will still hit $\UNEW$ with a very high probability. This together with the other two highly \textit{cost-effective} controls, makes the expected payoff almost $\MNEW$ even when $X(t)$ is close to $\LNEW$ (but not too close). For the myopic policy, when $X(t)$ is close to $\LNEW$, it will have a fair chance of hitting $\LNEW$ and the expected payoff is relatively much smaller.

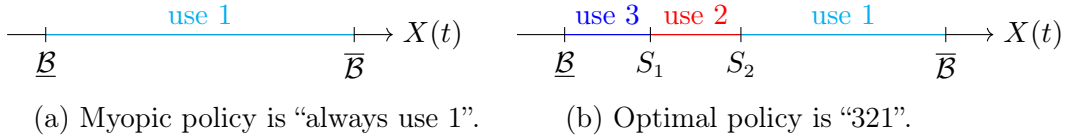
\begin{figure}[h]
    \centering
    \begin{subfigure}[b]{0.45\textwidth}
    \begin{tikzpicture}[xscale = 1.7]
        \def\LPos{0.1}
        \def\UPos{0.9}
        
        \draw[->] (-0.5,1) -- (2.5,1) node[pos=1, right, black] {$X(t)$};
        \draw[-, cyan] (-0.5+\LPos*3, 1) -- (-0.5+\UPos*3, 1) node[pos=0.5, above] {use $1$};
        \draw[-] (-0.5+\LPos*3, 1.1) -- (-0.5+\LPos*3, 0.9)  node[pos=1, below, black] {$\LNEW$};
        \draw[-] (-0.5+\UPos*3, 1.1) -- (-0.5+\UPos*3, 0.9)  node[pos=1, below, black] {$\UNEW$};
    \end{tikzpicture}
    \caption{Myopic policy is ``always use 1''.}
    \end{subfigure}
    \begin{subfigure}[b]{0.4\textwidth}
    \begin{tikzpicture}[xscale = 2.1]
        \def\LPos{0.1}
        \def\UPos{0.9}
        \def\FirstSPos{0.28}  
        \def\SecondSPos{0.47}  
    
        \draw[->] (-0.5,1) -- (2.5,1) node[pos=1, right, black] {$X(t)$};
        \draw[-, blue] (-0.5+\LPos*3, 1) -- (-0.5+\FirstSPos*3, 1)  node[pos=0.5, above] {use $3$};
        \draw[-, red] (-0.5+\FirstSPos*3, 1) -- (-0.5+\SecondSPos*3, 1)  node[pos=0.5, above] {use $2$};
        \draw[-, cyan] (-0.5+\SecondSPos*3, 1) -- (-0.5+\UPos*3, 1)  node[pos=0.5, above] {use $1$};
        \draw[-] (-0.5+\LPos*3, 1.1) -- (-0.5+\LPos*3, 0.9)  node[pos=1, below, black] {$\LNEW$};
        \draw[-] (-0.5+\UPos*3, 1.1) -- (-0.5+\UPos*3, 0.9)  node[pos=1, below, black] {$\UNEW$};
        \draw[-] (-0.5+\FirstSPos*3, 1.1) -- (-0.5+\FirstSPos*3, 0.9)  node[pos=1, below, black] {$S_1$};
        \draw[-] (-0.5+\SecondSPos*3, 1.1) -- (-0.5+\SecondSPos*3, 0.9)  node[pos=1, below, black] {$S_2$};
    \end{tikzpicture}
    \caption{Optimal policy is ``321''.}
    \end{subfigure}
    
\caption{A numerical example (fixed lower boundary case) that shows myopic policy can perform much worse than the optimal policy. Parameters are: $\LNEW=0$, $\UNEW=10$, $\MNEW=10$, $(\aNEW_1,\hNEW_1) = (0.2, 0.01)$, $(\aNEW_2,\hNEW_2) = (1, 0.06)$, $(\aNEW_3,\hNEW_3) = (4, 0.3)$. The switch points for the optimal policy are $S_1 = 2.2304$ and $S_2 = 4.3099$. }
\label{fig:myoPlcBad_1Eff_Plc}
\end{figure}

\begin{figure}[h]
    \centering
    \includegraphics[scale=0.65]{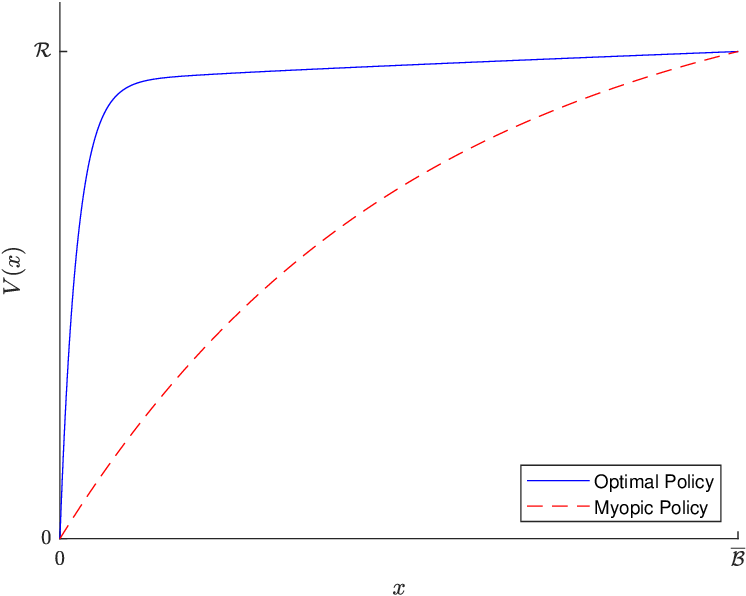}
    \caption{Value functions of the optimal (blue, solid) and myopic (red, dashed) policies. The parameters, optimal policy and myopic policy are the same as in Figure \ref{fig:myoPlcBad_1Eff_Plc}.}
    \label{fig:MyoPlcBad_1Eff_VF}
\end{figure}

To summarize, the entrepreneur should be cautious with using the myopic policy as it may result in a significant decrease in the expected payoff for the start-up firm, as compared to using the optimal policy. 

\section{Extension: Free Lower Boundary Case}
\label{sec:FrLB}

We now turn to the case where the lower boundary is not fixed but free; that is, $\LNEW$ is now determined endogenously by the entrepreneur.  We will show that there emerges yet another type of \textit{efficient frontier} curve that differs from the two types in the fixed lower boundary case, based on which we characterize the optimal policy. However, the optimal policy's control sequence has a similar (but not identical) structure as the fixed lower boundary case with $\LNEW < \barLNEW$.

We will proceed similarly as in the fixed lower boundary case. \S \ref{subsec:OptPlcStrtrFrLB} provides a high-level description of the optimal policy's general structure. \S \ref{subsec:EF_FrLB} presents the optimal policy in detail.  \S \ref{subsec:myoPlcBad_FrLB} compares the optimal policy with the myopic one and shows that similar to the fixed lower boundary case, the latter can perform arbitrarily worse than the former (with respect to the milestone reward $\MNEW$).

\subsection{High-Level: The Structure of the Optimal Policy}
\label{subsec:OptPlcStrtrFrLB}

For the free lower boundary problem, the optimal policy, at a high level, has a structure as specified by Proposition \ref{prop:Kctrls_FrLB_pvw}.
\begin{restatable}{proposition}{OptPlcStrtrFrLB}
\label{prop:Kctrls_FrLB_pvw}
For the free lower boundary problem, the optimal lower boundary $\LNEW^*$ satisfies 
\begin{align}
\LNEW^* < \barLNEW.
\end{align}
Moreover, there exists $\tK \le N $ and switch thresholds $\LNEW^* = S^*_0 < S^*_1 < ... < S^*_\tK = \UNEW$ such that the optimal policy uses control $i^*_j$ in the interval $[S^*_{j-1}, S^*_j)$, $j = 1,2,...,\tK$, and $i^*_1, i^*_2, ..., i^*_{\tK}$ satisfy
\begin{align}
    &\aNEW(i^*_1) < \aNEW(i^*_2) < ... < \aNEW(i^*_{\tK}),  \label{ineq:orderCrt1_FrLB} \\
    &CE(i^*_1) < CE(i^*_2) < ... < CE(i^*_{\tK}). \label{ineq:orderCrt2_FrLB}
\end{align}
\end{restatable}

The optimal lower boundary $\LNEW^*$ is to the left of $\barLNEW$. As a result, the structure of the optimal policy for the free lower boundary problem is similar to that of the fixed boundary one with $\LNEW < \barLNEW$. The intuition is as follows. Suppose $X(t) \ge \barLNEW$, then the entrepreneur will find it worthwhile to continue the process and still make a positive payoff. To put it more rigorously, setting the lower boundary at $\barLNEW$ will make the resulting value function have a positive slope at $\barLNEW$ and the entrepreneur has an incentive to shift the lower boundary further to the left, until a threshold where shifting any distance to the left of it will result in the value function attaining a negative value. This threshold is $\LNEW^*$ and at here, the value function attains a zero derivative.

\subsection{Detailed Characterization of the Optimal Policy}
\label{subsec:EF_FrLB}

We now proceed to find the corresponding \textit{efficient frontier}  for the free lower boundary case. We still assume that the controls are numbered as in \S\ref{subsec:EF_FxLB}.

\textbf{Efficient Frontier for the free lower boundary problem. } The type-III \textit{efficient frontier}  consists of the controls $1$, $2$, ..., $\n$ (where $\n$ is defined as in (\ref{eq:def_n_highestCE})) and they satisfy (see Figure \ref{fig:EF_FrLB})
\begin{align}
& 0 < \aNEW_1 < ... < \aNEW_{\n}, \text{ } 0 < \hNEW_1 < ... < \hNEW_{\n}, \label{ineq:EF_FrLB1}\\
& 0 < \frac{\hNEW_2 - \hNEW_1}{\aNEW_2 - \aNEW_1} < \frac{\hNEW_3 - \hNEW_2}{\aNEW_3 - \aNEW_2} < ... < \frac{\hNEW_{\n} - \hNEW_{\n-1}}{\aNEW_{\n} - \aNEW_{\n-1}} < \frac{\hNEW_{\n}}{\aNEW_{\n}}, \label{ineq:EF_FrLB2}\\
&\hNEW_i = \phi_{Fr}(\aNEW_i) \text{, for $1 \le i \le \n$,} \label{ineq:EF_FrLB3}\\
&\hNEW_i \ge \phi_{Fr}(\aNEW_i) \text{, for $i \not \in \{1,2,...,\n\}$,} \label{ineq:EF_FrLB4}
\end{align}
where $\phi_{Fr}(\cdot)$ is the piece-wise linear and convex function that connects $(0, \hNEW_1), (\aNEW_1,\hNEW_1)$, $(\aNEW_2,\hNEW_2)$, ... , $(\aNEW_{\n},\hNEW_{\n})$. Controls $1$, $2$, ... , $\n$ are either the left endpoint, the right endpoint, or a point on $\phi_{Fr}(\cdot)$ at which the slope changes. Denote $\Phi_{Fr}$ to be the set of controls on $\phi_{Fr}(\cdot)$, i.e. $\Phi_{Fr} = \{1,2, ..., \n\}$. 
\begin{figure}[h]
\FIGURE
{    
\begin{tikzpicture}[xscale = 1/11, yscale = 1/9]
\def\etaVec{{-3,0,10,20,30,40,50,60}}
\def\cVec{{11, 8, 9-3, 10-3, 12-3, 18-3, 28-3, 44-3}}
\def\Itcpt{6}
\draw[dashed] (-\Itcpt,0) -- ( \etaVec[4], \cVec[4]);
\draw[-] (\etaVec[4], \cVec[4]) -- (\etaVec[4]*2+\Itcpt*2-\Itcpt, \cVec[4]*2); 
\draw[-] (-\Itcpt,\cVec[2]) -- (\etaVec[2],\cVec[2]);
\foreach \i in {3,...,4}{
    \draw[-] (\etaVec[\i-1],\cVec[\i-1]) -- (\etaVec[\i],\cVec[\i]);
}
\foreach \i in {0}{
    \pgfmathtruncatemacro {\j}{\i-1}
    \filldraw[black] (\etaVec[\i],\cVec[\i]) circle (5pt) node[above right] {\scriptsize $(\aNEW_{\j},\hNEW_{\j})$};
}
\foreach \i in {1}{
    \pgfmathtruncatemacro {\j}{\i-1}
    \filldraw[black] (\etaVec[\i],\cVec[\i]) circle (5pt) node[above right] {\scriptsize $(\aNEW_{\j},\hNEW_{\j})$};
}
\foreach \i in {2,...,4}{
    \pgfmathtruncatemacro {\j}{\i-1}
    \filldraw[black] (\etaVec[\i],\cVec[\i]) circle (5pt) node[above] {\scriptsize $(\aNEW_{\j},\hNEW_{\j})$};
}
\foreach \i in {5}{
    \pgfmathtruncatemacro {\j}{\i-1}
    \filldraw[black] (\etaVec[\i],\cVec[\i]) circle (5pt) node[above] {\scriptsize $(\aNEW_{\j},\hNEW_{\j})$};
}
\foreach \i in {6,...,7}{
    \pgfmathtruncatemacro {\j}{\i-1}
    \filldraw[black] (\etaVec[\i],\cVec[\i]) circle (5pt) node[below right] {\scriptsize $(\aNEW_{\j},\hNEW_{\j})$};
}
\def\addetaVec{{15,30.0,45}}
\def\addcVec{{25.3,21.2,31.6}}
\foreach \i in {0,...,2}{
    \pgfmathtruncatemacro {\j}{\i+7}
    \filldraw[black] (\addetaVec[\i],\addcVec[\i]) circle (5pt) node[above] {\scriptsize $(\aNEW_{\j},\hNEW_{\j})$};
}
\draw[-] (-\Itcpt,0) -- (60,0)    node[pos=1,label=below:\textcolor{black}{$\aNEW$}]{};
\draw[-] (-\Itcpt,0) -- (-\Itcpt,38) node[pos=1,label=above:\textcolor{black}{$\hNEW$}]{};
\end{tikzpicture}
}
{The \textit{efficient frontier}  (piecewise linear solid curve) for the free lower boundary case: the piecewise linear function $\phi_{Fr}(\cdot)$ that satisfies (\ref{ineq:EF_FrLB3}) and (\ref{ineq:EF_FrLB4}). \label{fig:EF_FrLB}}
{}
\end{figure}

We note that the shape of the type-III \textit{efficient frontier}  is slightly different from that of the type-II (the one associated with $\LNEW < \barLNEW$ in the fixed boundary case). In particular, the \textit{efficient frontier}  no longer contains the controls $\Nm$, $\Nm+1$, \dots, $0$. 

\textbf{Statement of the Optimal Policy.} Having defined the \textit{efficient frontier}  for the free lower boundary case, we are now in a position to present the optimal policy. We will first show that the optimal lower boundary must lie within a certain range (Proposition \ref{prop:LstarRangeFrLB}). After that, we will show that the optimal policy in the free lower boundary case uses controls on the \textit{efficient frontier}  in a sequential and increasing manner as specified in Theorem \ref{thm:Kctrls_FrLB}. 

\begin{restatable}{proposition}{LstarRangeFrLB}
\label{prop:LstarRangeFrLB}
For the free lower boundary problem, the optimal lower boundary $\LNEW^*$ satisfies
\begin{align}
    \barLNEW - \frac{1}{\aNEW_{\n}} - (\UNEW - \barLNEW)(\frac{\aNEW_{\n}}{\aNEW_1} - 1) &- O(e^{-\MNEW})\le \LNEW^* \le \barLNEW - \frac{1}{\aNEW_{\n}} - O(e^{-\MNEW}), \label{eq:FrLB_LstarRangeFrLB2}
\end{align}
where both inequalities are tight.
\end{restatable}

Proposition \ref{prop:LstarRangeFrLB} essentially says that the optimal lower boundary $\LNEW^*$ can be both upper and lower bounded. We note that in (\ref{eq:FrLB_LstarRangeFrLB2}), both inequalities are tight. The proof involves constructing a sequence of problem instances that approximately achieve the inequalities (see Lemma \ref{lem:FrLB_LstarRange_ineq} in the Appendix).

The detailed description of the optimal policy for the free lower boundary case is specified in Theorem \ref{thm:Kctrls_FrLB}. 

\begin{restatable}{theorem}{OptPlcFrLB}
\label{thm:Kctrls_FrLB}
For the free lower boundary problem, fix $\{c_i, \mu_i, \sigma_i\}_{i=1}^N > 0$, $\UNEW$, $\MNEW > 0$. There exists an optimal lower boundary $\LNEW^*$ satisfying  (\ref{eq:FrLB_LstarRangeFrLB2}), and an optimal interval policy $\pi^* = \left\{\left(i_j^*,I_j^*=[S_{j-1}^*,S_j^*) \right)\right\}_{j=1}^\tK$, where $1 \le \tK \le \n$ and $i_1^*, i_2^*, ..., i_{\tK}^*$ satisfying
\begin{align}
    & i_j^* = j \in \Phi_{Fr} \text{, for $1 \le j \le \tK$,}\\
    &\aNEW(i_1^*) < \aNEW(i_2^*) < ... < \aNEW(i_{\tK}^*),  \\
    &CE(i_1^*) < CE(i_2^*) < ... < CE(i_{\tK}^*). 
\end{align}
Furthermore, the value function $V(\cdot)$ is $\mathcal{C}^2$ over $(-\infty,\UNEW) \backslash \{\LNEW^*\}$, $\mathcal{C}^1$ at $\LNEW^*$ and satisfies
\begin{align}
&V(\UNEW) = \MNEW, \label{eq:FrLB_suffCondOpt1} \\
&V(x) = 0 \text{ for all $x \le \LNEW^*$}, V(x) \ge 0 \text{ for all $x \ge \LNEW^*$}, \label{eq:FrLB_suffCondOpt2}\\
&V'(\LNEW^*) = 0, \label{eq:FrLB_suffCondOpt3}\\
&-c_i + \mu_i V'(x) + \frac{\sigma_i^2}{2}V''(x) \le 0 \text{, for all $x \in (\LNEW^*,\UNEW)$ and $i \in \I$,} \label{eq:FrLB_suffCondOpt4} \\
&-c_{i_j^*} + \mu_{i_j^*} V'(x) + \frac{\sigma_{i_j^*}^2}{2}V''(x) = 0 \text{, for $x \in (S_{j-1}^*,S_j^*)$, $1 \le j \le \tK$}. \label{eq:FrLB_suffCondOpt5}
\end{align}

\end{restatable}

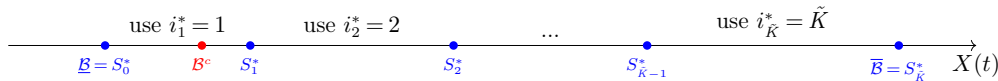
\begin{figure}[h]
    \centering
    \scalebox{0.8}
    {
\begin{tikzpicture}[xscale = 3.2]
    \def\Lbar{0.2}
    \def\Lstar{0.1}
    \def\iOne{0.175}
    \def\SOne{0.25}
    \def\iTwo{0.355}
    \def\STwo{0.46}
    \def\DotDotDot{0.56}
    \def\SKminusOne{0.66}
    \def\iK{0.79}
    \def\U{0.92}
    \def\Xt{1}
\draw[->] (0,0) -- (5,0)
node[pos=\Lbar,fill=red,circle, inner sep=0pt,minimum size=4pt,label=below:\textcolor{red}{\scriptsize{$\barLNEW$}}]{}
node[pos=\Lstar,fill=blue,circle, inner sep=0pt,minimum size=4pt,label=below:\textcolor{blue}{\scriptsize{$\LNEW=S^*_0$}}]{}
node[pos=\iOne, above] {use $i_1^* = 1$}
node[pos=\SOne,fill=blue,circle, inner sep=0pt,minimum size=4pt,label=below:\textcolor{blue}{\scriptsize{$S^*_1$}}]{}
node[pos=\iTwo, above] {use $i_2^* = 2$}
node[pos=\STwo,fill=blue,circle, inner sep=0pt,minimum size=4pt,label=below:\textcolor{blue}{\scriptsize{$S^*_2$}}]{}
node[pos=\DotDotDot, above] {...}
node[pos=\SKminusOne,fill=blue,circle, inner sep=0pt,minimum size=4pt,label=below:\textcolor{blue}{\scriptsize{$S^*_{\tK - 1}$}}]{}
node[pos=\iK, above] {use $i_{\tK}^* = \tK$}
node[pos=\U,fill=blue,circle, inner sep=0pt,minimum size=4pt,label=below:\textcolor{blue}{\scriptsize{$\UNEW=S^*_{\tK}$}}]{}
node[pos=\Xt, below] {$X(t)$};
\end{tikzpicture}
}

    \caption{Detailed characterization of the optimal policy $\pi_{Fr}^*$. As $X(t)$ increases, the optimal policy uses controls $1,2,...,\tK$ (increasing \textit{effective-drift} in increasing order in $\Phi_{Fr}$), where $\tK \le \n$.}
    \label{fig:OptPlcFrLB}
\end{figure}

Theorem \ref{thm:Kctrls_FrLB} provides a more comprehensive and detailed description of Proposition \ref{prop:Kctrls_FrLB_pvw}. The optimal action is to use control $j$ when the state $X(t)$ is in the interval $[S_{j-1}^*, S_j^*)$ (see Figure \ref{fig:OptPlcFrLB}). As $X(t)$ increases, the optimal policy uses controls with larger \textit{effective-drift} and \textit{cost-effectiveness}, in increasing order on the \textit{efficient frontier}  $\phi_{Fr}$. The controls used all have \textit{effective-drift} not exceeding that of control $\n$. Note that control $1$ (the \textit{riskiest} control on the \textit{efficient frontier} ) is always used near the optimal lower boundary.

Again to find the actual controls used in the optimal policy, or the set of switching thresholds $S_j^*$, we refer the reader to the numerical procedure in \S \ref{sec:PfThmFrLB} in the Appendix.

\subsection{Analysis of the Myopic Policy (Free Lower Boundary Case)}
\label{subsec:myoPlcBad_FrLB}

For the free lower boundary case, the \textit{myopic policy} can perform arbitrarily worse than the optimal policy. 
\begin{restatable}{proposition}{FrLBMyoPlcBad}
\label{prop:FrLBMyoPlcBad}
For the base model, let $\UNEW$, $\MNEW > 0$ be fixed. For any $\epsilon > 0$, there exists a problem instance such that 
\begin{align}
    \sup_{x \in (-\infty,\UNEW)} \frac{V_{opt}(x) - V_{myo}(x)}{\MNEW} > 1 - \epsilon,
\end{align}
where $V_{opt}(\cdot)$ and $V_{myo}(\cdot)$ denote the value function of the optimal and myopic policy. 
\end{restatable}

This conclusion is similar to the one of the fixed lower boundary case. A numerical example that illustrates when the difference can be large is presented in Appendix \ref{ec_sec:myoPlcBad_FrLB}.

\section{Discussion}
\label{sec:Implications}

In this section, we summarize the main insights from this work and its contributions relative to [\cite{wang2022new}].

The main insights are as follows. First, all controls can be characterized by two measures: \textit{riskiness} (drift-to-volatility ratio) and \textit{cost-effectiveness} (drift-to-cost ratio). Second, when there are many controls available, the entrepreneur should reduce the set of controls under consideration to be the subset specified by the corresponding \textit{efficient frontier}. This may significantly reduce the complexity of the problem. Third, the optimal strategy's structure changes qualitatively depending on the model parameters: (a) if the lower boundary is fixed and relatively near $\UNEW$ (i.e. $\LNEW > \barLNEW$), then the entrepreneur should use increasingly \textit{riskier} and more \textit{cost-effective} controls as the firm's performance metric improves; (b) if the lower boundary is fixed and relatively far away from $\UNEW$ (i.e. $\LNEW < \barLNEW$), then the entrepreneur should use increasingly \textit{safer} and more \textit{cost-effective} controls as the firm's performance metric improves; (c) if the lower boundary is free, the optimal strategy is similar to (b) and the entrepreneur should at least lower the boundary to $\barLNEW$. Lastly, whenever possible, the entrepreneur should strive to use the optimal policy and avoid using the myopic policy, as the latter may have significantly worse performance. 

The current work also represents a non-trivial generalization of [\cite{wang2022new}]. Notably, it employs an entirely distinct and novel proof techniques that involve prescribing the \textit{efficient frontier} set of optimal controls. In addition, our work demonstrates the limitation of [\cite{wang2022new}]’s results and show that their insights may be misleading in the multi-control setting. In particular, the DCCs of the controls can no longer serve as sufficient statistics to predict the optimal policy. 

Furthermore, the current work offers a new metric \textit{cost-effectiveness} to guide entrepreneurs. This metric possesses a highly intuitive interpretation that helps to determine an ordering of the controls and illustrates a clear trade-off among each control’s characteristics (namely drift, variance and cost). Lastly, our study provides a clear intuition behind why the optimal policy's structure changes qualitatively depending on the model parameters (see ``Intuition of the Optimal Policy Structure'' in \S\ref{subsec:OptPlcStrtrFxLB}). In contrast, [\cite{wang2022new}] lacks the provision of such explanatory insight. This intuition provides entrepreneurs a clear comprehension of their strategy and helps provide valuable insights into the optimal policy’s structure,

\section{Conclusion}
\label{sec:conclusion}

In this paper, we examine a model of an entrepreneurial start-up firm, where a diffusion process captures the state of the firm. The main insight is that the entrepreneur should evaluate controls in both dimensions of \textit{riskiness} and \textit{cost-effectiveness} and the optimal strategy's structure changes qualitatively depending on the lower boundary's nature and position. In solving for the optimal policy, we explicitly construct the \textit{efficient frontier}  curve that specifies the set of controls that may be used in the optimal policy. As far as we know, this is the first study that analyzes a stochastic control model which admits \textit{efficient frontier}  curves of different types.

In reality, the entrepreneurial process represents a very complex phenomenon. There is generally a lack of framework to guide the entrepreneur within a startup as it is challenging to model every aspect of the process. As such, the current study and model are not without limitations. In particular, our model assumes the existence of a single metric that can capture the state of the firm. This assumption enables the HJB equation to involve sets of ordinary differential equations, which results in clean and tractable analytical results. A more general model will be to incorporate multiple metrics to represent the state of the start-up firm, and assume each control influences all metrics. However, this will result in the HJB equation involving sets of partial differential equations and making the model intractable to analyze. 

In conclusion, this work investigates strategies for start-ups in a milestone-oriented setting. We incorporate multiple controls into the model, obtain optimal control policies and derive clean managerial insights. We believe the results from this paper provide a foundational block in the study of entrepreneurial decision-making, and help inform practice and research related to an applied theory framework to guide entrepreneurs within a start-up.

{
\SingleSpacedXI
\bibliographystyle{informs2014} 
\bibliography{reference_correlated} 
}

\newpage
\setcounter{page}{1}
\renewcommand{\thepage}{EC-\arabic{page}}

\clearpage
{\Large
\centerline{\textbf{The Electronic Companion}}
}

\begin{APPENDICES}

\section{Proof of Results for the Fixed Lower Boundary Case}
\label{ec_sec:PfRsltsFxLBCs}
Appendix \ref{ec_sec:PfRsltsFxLBCs} presents the proof of results for the fixed lower boundary case. \S\ref{sec:PfThmFxLB_case1} presents the proof of Theorem \ref{thm:Kctrls_FxLB} Case (1), \S\ref{sec:PfThmFxLB_case2} presents the proof of Theorem \ref{thm:Kctrls_FxLB} Case (2), and \S\ref{sec:PfPropFxLB} presents the proof of Propositions \ref{prop:Kctrls_FxLB} and \ref{prop:FxLBMyoPlcBad}.

\subsection{Proof of Theorem \ref{thm:Kctrls_FxLB} Case (1)}
\label{sec:PfThmFxLB_case1}


In this section, we will prove {Theorem} \ref{thm:Kctrls_FxLB} Case (1) by constructing an optimal interval policy. The proof is organized into three main steps. First, we define some relevant quantities that will be useful in the proof. Second, we construct a value function parametrized by a quantity $q$ that lies within certain ranges, and show that the set of value functions constructed in this way possess certain desired properties and are ordered. Third combining all the previous steps, we show that the constructed value function corresponds to the desired optimal policy. 

\subsubsection{Step I: Notations and Definitions}~
Let $\{\aNEW_i,\hNEW_i\}_{i=\Nm}^{\Np}$ be given. Define a sequence of thresholds $\{\xi_i\}_{i=\n}^{\Np - 1}$ as
\begin{align}
    \xi_i = \frac{\hNEW_{i+1} - \hNEW_i}{\aNEW_{i+1} - \aNEW_i}\text{ , $i = \n,\n+1...,\Np-1$.} \label{eq:xiDefFxBrCs1}
\end{align}

Define $q_{min} = \hNEW_{\n}/\aNEW_{\n}$, a function $\bar{\Gamma}_{init}(\cdot)$ that maps any value of $q > q_{min}$ to a control index,  functions $\bar{E}_i(\cdot)'s$ and quantities $\bar{D}_i's$ given by
\begin{align}
    \bar{\Gamma}_{init}(q) &= 
    \begin{cases}
    \Np \text{ if $q > \xi_{\Np - 1}$}\\
    \Np - 1 \text{ if $\xi_{\Np - 2} < q \le \xi_{\Np - 1}$}\\
    ...\\
    \n+1 \text{ if $\xi_{\n} < q \le \xi_{\n+1}$}\\
    \n \text{ if $q_{min} < q \le \xi_{\n}$}\\
    \end{cases}\\
    \bar{E}_i(q) &= \begin{cases} \frac{1}{\aNEW_i} \ln{\frac{ q-\hNEW_i/\aNEW_i}{\xi_{i-1}-\hNEW_i/\aNEW_i}} \text{, $i = \n+1, ..., \Np$, $\xi_{i-1} < q \le \xi_i$}\\
    \infty\text{, $i = \n$, $q_{min} < q \le \xi_{\n}$}\\
    \end{cases} \label{pfEq:EiqFxBrCs1}\\
    \bar{D}_i &= \frac{1}{\aNEW_i} \ln{\frac{ \xi_i - \hNEW_i/\aNEW_i}{ \xi_{i-1}-\hNEW_i/\aNEW_i}} \text{, $i = \n+1, ..., \Np-1$} \label{pfEq:DiFxBrCs1}
\end{align}
and note that: 1. $\bar{E}_i(q) > 0$ is well-defined because $\xi_{i-1} > \hNEW_i/\aNEW_i$ (c.f. Lemma \ref{suppLem:xiRltnFxBrCs1}) and $q > \xi_{i-1} > \hNEW_i/\aNEW_i$, 2. $\bar{D}_i > 0$ is well-defined because $\xi_i > \hNEW_i/\aNEW_i$ and $\xi_{i-1} > \hNEW_i/\aNEW_i$ (c.f. Lemma \ref{suppLem:xiRltnFxBrCs1}). We can also show the following property of $\bar{E}_i(q)$.

\begin{restatable}{lemma}{lemEiqMonoInqFxBrCsOne}
\label{lem:EiqMonoInqFxBrCs1}
Suppose $\xi_{i-1} < q \le \xi_i$ (or $q > \xi_{\Np-1}$), $i = \bar{\Gamma}_{init}(q)$, and Let $\bar{E}_i(q)$ be defined as in (\ref{pfEq:EiqFxBrCs1}). Then we have 
\begin{align}
    \frac{d \bar{E}_i(q)}{dq} > 0.
\end{align}
\end{restatable}

In addition, define $W_q^{(i)} (x)$ to be the unique solution of $W(x)$
\begin{align}
\label{eq:a_x0y0q_xFxBrCs1}
    -c_i + \mu_i W'(x) + \frac{\sigma_i^2}{2} W''(x) = 0, W(0) = 0, W'(0) = q,
\end{align}
and we can show the following relationship between $W_q^{(i)} (x)$ and $\bar{D}_i$, $\bar{E}_i(q)$'s.

\begin{restatable}{lemma}{lemWDEFxBrCsOne}
\label{lem:WDEFxBrCs1}
Let $W_q^{(i)} (x)$ be defined as in (\ref{eq:a_x0y0q_xFxBrCs1}), $\bar{E}_i(q)$ and $\bar{D}_i$ be defined as in (\ref{pfEq:EiqFxBrCs1}) and (\ref{pfEq:DiFxBrCs1}). Then we have
\begin{align}
    &W_q^{(i)}{}'(0) = q, \text{ } W_q^{(i)}{}''(0) = \hNEW_i - \aNEW_i q \label{pfEq:WDE1FxBrCs1}\\
    &W_q^{(i)}{}'(\bar{E}_i(q)) = \xi_{i-1}, \text{ } W_q^{(i)}{}''(\bar{E}_i(q)) = \hNEW_i - \aNEW_i \xi_{i-1} \label{pfEq:WDE2FxBrCs1}\\
    &W_{\xi_i}^{(i)}{}'(\bar{D}_i) = \xi_{i-1}, \text{ } W_{\xi_i}^{(i)}{}''(\bar{D}_i) = \hNEW_i - \aNEW_i \xi_{i-1}. \label{pfEq:WDE3FxBrCs1}
\end{align}
In addition, if $q > \xi_{i-1}$ with $\n+1 \le i \le \Np$ (or $q > q_{min}$), then 
\begin{enumerate}[\hspace{0.5cm} a)]
    \item $W_q^{(i)} (x)$ (or $W_q^{(\n)} (x)$) is concave over $(0, \infty)$, i.e. $W_q^{(i)}{}''(x) < 0 $ (or $W_q^{(\n)}{}''(x) < 0$) for all $x \in (0, \infty)$. \label{enum:W_ccvFxBrCs1}
    \item $W_q^{(i)}{}'(x) > \xi_{i-1}$ (or $W_q^{(\n)}{}'(x) > q_{min}$) for all $x \in (0, \bar{E}_i(q))$. \label{enum:W_deri_lbFxBrCs1}
\end{enumerate}
Furthermore, for any $x \in (0, \infty)$, $\n \le i \le \Np$, we have
\begin{align}
    \lim_{q \rightarrow \left(\frac{\hNEW_i}{\aNEW_i}\right)^+} W_q^{(i)}(x) = \frac{\hNEW_i}{\aNEW_i} x, \text{ } \lim_{q \rightarrow \infty} W_q^{(i)}(x) = \infty. \label{enum:WvalueRangeFxBrCs1}
\end{align}
\end{restatable}

\enlargethispage{\baselineskip}
\enlargethispage{\baselineskip}
\enlargethispage{\baselineskip}
\enlargethispage{\baselineskip}

\subsubsection{Step II: Construction of $V_E(x;q)$ for a given $q$}
\label{subsec:FxBrCs1PfStep2}
Given a value of $q \in (q_{min}, \infty)$, let $i = \bar{\Gamma}_{init}(q) \in \{\n,\n+1,...,\Np\}$ and we can construct the following trial value function:

{\scriptsize
\begin{align}
\label{pfDef:cstrtVFFxBrCs1}
V_E(x;q) &=
    \begin{cases}
    W_q^{(i)}(x) \text{ if $0\le x<\bar{E}_i(q)$}\\
    W_{\xi_{i-1}}^{(i-1)}(x - \bar{E}_i(q)) + W_q^{(i)}(\bar{E}_i(q)) \text{ if $\bar{E}_i(q) \le x < \bar{E}_i(q) + \bar{D}_{i-1}$}\\
    W_{\xi_{i-2}}^{(i-2)}(x - \bar{E}_i(q) - \bar{D}_{i-1}) +  W_{\xi_i}^{(i-1)}(\bar{D}_{i-1}) + W_q^{(i)}(\bar{E}_i(q)) \text{ if $\bar{E}_i(q) + \bar{D}_{i-1} \le x < \bar{E}_i(q) + \bar{D}_{i-1} + \bar{D}_{i-2}$}\\
    ...\\
    W_{\xi_{\n+1}  }^{(\n+1)}(x - \bar{E}_i(q) - \sum_{j=\n+2}^{i-1} \bar{D}_j) + \sum_{j=\n+2}^{i-1} W_{\xi_j}^{(j)}(\bar{D}_j) + W_q^{(i)}(\bar{E}_i(q))  \text{ if $\bar{E}_i(q) + \sum_{j=\n+2}^{i-1} \bar{D}_j \le x < \bar{E}_i(q) + \sum_{j=\n+1}^{i-1} \bar{D}_j$}\\
    W_{\xi_{\n}  }^{(\n)}(x - \bar{E}_i(q) - \sum_{j=\n+1}^{i-1} \bar{D}_j) + \sum_{j=\n+1}^{i-1} W_{\xi_j}^{(j)}(\bar{D}_j) + W_q^{(i)}(\bar{E}_i(q))  \text{ if $x \ge \bar{E}_i(q) + \sum_{j=\n+1}^{i-1} \bar{D}_j$}\\
    \end{cases}
\end{align}
}
and note that $V_E(x;q)$ is defined over $[0,\infty)$. Next we will prove some properties of $V_E(x;q)$ and show that the set $\{V_E(\cdot;q))\}_{q \in (q_{min}, \infty)}$ is ordered.

\begin{restatable}{lemma}{VForderedOneEff}
\label{lem:VForderedFxBrCs1}
Suppose $q \in (q_{min}, \infty)$ and $V_E(x;q)$ is constructed as in (\ref{pfDef:cstrtVFFxBrCs1}), then 
\begin{enumerate}[\hspace{0.5cm} a)]
    \item Over $(0,\infty)$, $V_E(\cdot;q)$ is $C^2$.  \label{enum:VForderedFxBrCs1_1}
    \item Over $(0,\infty)$, $V_E(\cdot;q)$ is concave, i.e. $V_E''(\cdot;q) < 0$. \label{enum:VForderedFxBrCs1_2}
    \item For $x \in (0,\infty)$, $V_E'(x;q) \in (q_{min}, q)$, $\lim_{\hat{q} \rightarrow q_{min}^+} V_E(x;\hat{q}) = q_{min} x$, $\lim_{\hat{q} \rightarrow \infty} V_E(x;\hat{q}) = \infty$. \label{enum:VForderedFxBrCs1_3}
    \item The functions $\{V_E(\cdot;q))\}_{q \in (q_{min}, \infty)}$ are ordered, with a smaller $q$ producing uniformly smaller function $V_E(\cdot;q)$. That is, for any $x \in (0,\infty)$, if $q < q'$ then $V_E(x;q) < V_E(x;q')$. \label{enum:VForderedFxBrCs1_4}
\end{enumerate}
\end{restatable}

\begin{figure}[h]
    \centering
    \scalebox{0.8}
    {
\begin{tikzpicture}[xscale = 3.2]
    \def\SZero{0.1}
    \def\iOne{0.15}
    \def\SOne{0.2}
    \def\iTwo{0.25}
    \def\STwo{0.3}
    \def\DotDotDot{0.4}
    \def\SiOneMinusnMinusOne{0.5}
    \def\nPlusOne{0.65}
    \def\SiOneMinusn{0.8}
    \def\nLast{0.9}
    \def\Xt{1}
\draw[->] (0,0) -- (5,0)
node[pos=\SZero,fill=blue,circle, inner sep=0pt,minimum size=4pt,label=below:\textcolor{blue}{\scriptsize{$S_0$}}]{}
node[pos=\iOne, above] {$i_1$}
node[pos=\SOne,fill=blue,circle, inner sep=0pt,minimum size=4pt,label=below:\textcolor{blue}{\scriptsize{$S_1$}}]{}
node[pos=\iTwo, above] {$i_1 - 1$}
node[pos=\STwo,fill=blue,circle, inner sep=0pt,minimum size=4pt,label=below:\textcolor{blue}{\scriptsize{$S_2$}}]{}
node[pos=\DotDotDot, above] {...}
node[pos=\SiOneMinusnMinusOne,fill=blue,circle, inner sep=0pt,minimum size=4pt,label=below:\textcolor{blue}{\scriptsize{$S_{i_1 - \n - 1}$}}]{}
node[pos=\nPlusOne, above] {\begin{tabular}{@{}c@{}}$i_1 - (i_1 - \n - 1)$ \\ $=$ \\ $\n+1$ \end{tabular} }
node[pos=\SiOneMinusn,fill=blue,circle, inner sep=0pt,minimum size=4pt,label=below:\textcolor{blue}{\scriptsize{$S_{i_1 - \n}$}}]{}
node[pos=\nLast, above] {$\n$}
node[pos=\Xt, below] {$\infty$};
\end{tikzpicture}
}
    \caption{Illustration of switch points defined by (\ref{eq:SiFxBrCs1}).}
    \label{fig:SposFxBrCs1}
\end{figure}
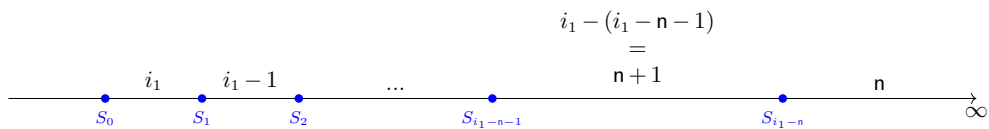

Suppose $\LNEW > \barLNEW$, then this is equivalent to $\MNEW > q_{min} (\UNEW-\LNEW)$. By Lemma \ref{lem:VForderedFxBrCs1} (\ref{enum:VForderedFxBrCs1_3}) and (\ref{enum:VForderedFxBrCs1_4}), there exists a unique $q^*$ s.t. $V_E(\cdot;q^*)$ satisfies
\begin{align}
    V_E(0;q^*) = 0, \text{ } V_E(\UNEW-\LNEW;q^*) = \MNEW. \label{pfEq:VFbdryCondFxBrCs1}
\end{align} 
Let 
\begin{align}
    i_1 &= \bar{\Gamma}_{init}(q^*) \label{eq:PfOptPlcCtrl1FxBrCs1}\\
    S_0 &= 0 \label{eq:PfOptPlcS0FxBrCs1}
\end{align} 
Since $q^* \in (q_{min}, \infty)$, we have $\n \le i_1 \le \Np$. Define $\{S_j\}_{j=1}^{i_1 - \n+1}$ to be:

\begin{align}
\label{eq:SiFxBrCs1}
    \begin{cases}
    S_1 = \bar{E}_{i_1}(q^*)\\
    S_2 = \bar{E}_{i_1}(q^*) + \bar{D}_{i_1-1}\\
    S_3 = \bar{E}_{i_1}(q^*) + \bar{D}_{i_1-1} + \bar{D}_{i_1-2}\\
    ...\\
    S_{k-1} = \bar{E}_{i_1}(q^*) + \sum_{j=1}^{k-2} \bar{D}_{i_1-j} \\
    S_k = \bar{E}_{i_1}(q^*) + \sum_{j=1}^{k-1} \bar{D}_{i_1-j} \\
    ...\\
    S_{i_1 - \n} = \bar{E}_{i_1}(q^*) + \sum_{j=1}^{i_1-\n-1} \bar{D}_{i_1-j}\\
    S_{i_1 - \n+1} = \infty,
    \end{cases}
\end{align}
and define $\tK = \inf\{j \lvert S_j \ge \UNEW-\LNEW\}$, i.e. $S_{\tK-1} < \UNEW-\LNEW$ and $S_{\tK} \ge \UNEW-\LNEW$. Further define $S^*_0 = \LNEW + S_0$, $S^*_1 = \LNEW + S_1$, ..., $S^*_{\tK-1}= \LNEW + S_{\tK-1}$, $S^*_{\tK}= \UNEW \le \LNEW + S_{\tK}$. Construct the interval policy $\pi^*_E = \{i^*_j,I^*_j\}_{j=1}^{\tK} $ as follows
\begin{align}
\label{eq:ijFxBrCs1}
    i^*_j &= i^*_1 - (j-1) \text{ for $1 \le j \le \tK$,}\\I^*_j &= [S^*_{j-1},S^*_j)\text{ for $1 \le j \le \tK$,}
\end{align}
and define $\hat{V}_E(\cdot;q^*)$ as
\begin{align}
    \hat{V}_E(\cdot - \LNEW;q^*) = V_E(\cdot;q^*). \label{eq:VhatShiftLFxBrCs1}
\end{align}
We will prove that $\pi^*_E$ is an optimal policy (with $\hat{V}_E(\cdot;q^*)$ being the associated value function).

\subsubsection{Step III: Proof of Theorem \ref{thm:Kctrls_FxLB} Case (1)}

\emph{Proof of Theorem \ref{thm:Kctrls_FxLB} Case (1).}~ We will now prove Theorem  \ref{thm:Kctrls_FxLB} Case (1). In particular, we will show that the interval policy $\pi^*_E$ whose value function $\hat{V}_E(\cdot;q^*)$ is $\mathcal{C}^2$ over $[\LNEW,\UNEW)$ and satisfies
\begin{align}
&V(\LNEW) = 0, \text{ } V(\UNEW) = \MNEW, \label{eq:PfsuffCondOptPlcFxBrCs1_1}\\
& -c_i + \mu_i V'(x) + \frac{\sigma_i^2}{2}V''(x) \le 0 \text{, for all $x \in [\LNEW,\UNEW)$ and $i \in \I$,} \label{eq:PfsuffCondOptPlcFxBrCs1_init2} \\
& -c_{i^*_j} + \mu_{i^*_j} V'(x) + \frac{\sigma_{i^*_j}^2}{2}V''(x) = 0 \text{, for $x \in [S_{j-1}^*,S_j^*)$, $1 \le j \le \tK$}, \label{eq:PfsuffCondOptPlcFxBrCs1_3}
\end{align}
and hence the interval policy is optimal.

By (\ref{pfEq:VFbdryCondFxBrCs1}) and (\ref{eq:VhatShiftLFxBrCs1}), $\hat{V}_E(\cdot;q^*)$ satisfies (\ref{eq:PfsuffCondOptPlcFxBrCs1_1}) and (\ref{eq:PfsuffCondOptPlcFxBrCs1_3}). We will now establish
\begin{align}
& -c_i + \mu_i \hat{V}_E'(x;q^*) + \frac{\sigma_i^2}{2}\hat{V}_E''(x;q^*) \le 0 \text{, for all $x \in [\LNEW,\UNEW)$ and $i \in \I$.} \label{eq:PfsuffCondOptPlcFxBrCs1_2}
\end{align}
For convenience, we will use the following terminology: we say that control $j$ dominates control $i$ at $x$ if 
\begin{align}
-c_i + \mu_i \hat{V}_E'(x;q^*) + \frac{\sigma_i^2}{2}\hat{V}_E''(x;q^*)  \le -c_j + \mu_j \hat{V}_E'(x;q^*) + \frac{\sigma_j^2}{2}\hat{V}_E''(x;q^*). \\
\nonumber
\end{align}

\enlargethispage{\baselineskip}
\enlargethispage{\baselineskip}
\enlargethispage{\baselineskip}
\enlargethispage{\baselineskip}

Given any interval $I^*_j = [S^*_{j-1},S^*_j)$, let $k = i^*_j$ be the control used, i.e. 
\begin{align}
    -&c_k + \mu_k \hat{V}_E'(x;q^*) + \frac{\sigma_k^2}{2}\hat{V}_E''(x;q^*) = 0,\\
    &\xi_{k-1} < \hat{V}_E'(x;q^*) \le \xi_k \text{ over } [S^*_{j-1},S^*_j), \label{pfEq:dmntInterm0FxBrCs1}
\end{align}
and from now on fix $x \in [S^*_{j-1},S^*_j)$. 

For a control $i \in \Phi = \{\Nm, ..., 1,2, ..., \Np\}$, we have 
\begin{align}
    -c_i + \mu_i \hat{V}_E'(x;q^*) + \frac{\sigma_i^2}{2}\hat{V}_E''(x;q^*) \le -c_k + \mu_k \hat{V}_E'(x;q^*) + \frac{\sigma_k^2}{2}\hat{V}_E''(x;q^*) 
\end{align}
if and only if (c.f. Lemma \ref{supp_lem:dmntEqvFxBr})
\begin{align}
    -\hNEW_i + \aNEW_i \hat{V}_E'(x;q^*) + \hat{V}_E''(x;q^*) \le -\hNEW_k + \aNEW_k \hat{V}_E'(x;q^*) + \hat{V}_E''(x;q^*), 
\end{align}
which is equivalent to 
\begin{align}
    \hNEW_k -\hNEW_i \le (\aNEW_k - \aNEW_i) \hat{V}_E'(x;q^*). \label{pfEq:dmntInterm1FxBrCs1}
\end{align}
Suppose $\aNEW_i < \aNEW_k$, then (\ref{pfEq:dmntInterm1FxBrCs1}) is equivalent to $\frac{\hNEW_k - \hNEW_i}{\aNEW_k - \aNEW_i} \le \hat{V}_E'(x;q^*)$, which is true because $\frac{\hNEW_k - \hNEW_i}{\aNEW_k - \aNEW_i} \le \frac{\hNEW_k - \hNEW_{k-1}}{\aNEW_k - \aNEW_{k-1}} = \xi_{k-1} < \hat{V}_E'(x;q^*)$ (the $1^{st}$ inequality is due to control $i \in \Phi$ and the $2^{nd}$ inequality is due to (\ref{pfEq:dmntInterm0FxBrCs1})). Suppose $\aNEW_i > \aNEW_k$, then (\ref{pfEq:dmntInterm1FxBrCs1}) is equivalent to $\frac{\hNEW_i - \hNEW_k}{\aNEW_i - \aNEW_k} \ge \hat{V}_E'(x;q^*)$, which is true because $\frac{\hNEW_i - \hNEW_k}{\aNEW_i - \aNEW_k} \ge  \frac{\hNEW_{k+1} - \hNEW_k}{\aNEW_{k+1} - \aNEW_k} = \xi_k \ge \hat{V}_E'(x;q^*)$ (the $1^{st}$ inequality is due to control $i \in \Phi$ and the $2^{nd}$ inequality is due to (\ref{pfEq:dmntInterm0FxBrCs1})).

For a control $i \not\in \Phi$, then it can be easily verified that for any $x \in (\LNEW,\UNEW)$ there exists at least one $j \in \Phi$ s.t. control $i$ is dominated by $j$ at $x$. Thus (\ref{eq:PfsuffCondOptPlcFxBrCs1_2}) is established.

In summary, the constructed $\hat{V}_E(\cdot;q^*)$ is the desired value function, which is $C^2$ over $(\LNEW,\UNEW)$ and satisfies (\ref{eq:PfsuffCondOptPlcFxBrCs1_1}), (\ref{eq:PfsuffCondOptPlcFxBrCs1_init2}) and (\ref{eq:PfsuffCondOptPlcFxBrCs1_3}). \Halmos

\subsubsection{Lemmas Used in the Proof of Theorem \ref{thm:Kctrls_FxLB} Case (1)}~


\emph{Proof of Lemma \ref{lem:EiqMonoInqFxBrCs1}.}~  This can be verified by straightforward computation. Proof is omitted. \Halmos


\emph{Proof of Lemma \ref{lem:WDEFxBrCs1}.}~  We first prove (\ref{pfEq:WDE1FxBrCs1}), (\ref{pfEq:WDE2FxBrCs1}) and (\ref{pfEq:WDE3FxBrCs1}). (\ref{pfEq:WDE1FxBrCs1}) holds directly by (\ref{eq:a_x0y0q_xFxBrCs1}). To prove (\ref{pfEq:WDE2FxBrCs1}), note that
\begin{align}
    W_q^{(i)}(x) = \frac{(\aNEW_i q-\hNEW_i) (1-e^{-\aNEW_i x})}{\aNEW_i^{2}}+\frac{\hNEW_i x}{\aNEW_i}, \label{eq:explicit_V00qxFxBrCs1} \\
    W_q^{(i)}{}'(x) = \frac{(\aNEW_i q-\hNEW_i) }{\aNEW_i} e^{-\aNEW_i x}+\frac{\hNEW_i}{\aNEW_i}, \label{eq:explicit_Vp00qxFxBrCs1}
\end{align}
and hence the first equation is true because
\begin{align*}
    W_q^{(i)}{}'(\bar{E}_i(q)) &= \frac{(\aNEW_i q-\hNEW_i) }{\aNEW_i} e^{-\aNEW_i \cdot \frac{1}{\aNEW_i} \ln{\frac{ q-\hNEW_i/\aNEW_i}{\xi_{i-1}-\hNEW_i/\aNEW_i}}}+\frac{\hNEW_i}{\aNEW_i}\\
    &= \xi_{i-1},
\end{align*}
and the second equation is true because of (\ref{eq:a_x0y0q_xFxBrCs1}). (\ref{pfEq:WDE3FxBrCs1}) holds because we can just replace (\ref{pfEq:WDE2FxBrCs1}) with $q = \xi_i$.

Next we prove (\ref{enum:W_ccvFxBrCs1}), i.e. the concavity of $W_q^{(i)} (x)$. By (\ref{eq:explicit_Vp00qxFxBrCs1}) we have
\begin{align}
    W_q^{(i)}{}''(x) = -(\aNEW_i q-\hNEW_i) e^{-\aNEW_i x}. \label{eq:explicit_Vpp00qxFxBrCs1}
\end{align}
Suppose $q > \xi_{i-1}$ (with $\n+1 \le i \le \Np$) then $\aNEW_i q-\hNEW_i > 0$ by Lemma \ref{suppLem:xiRltnFxBrCs1}. Hence $W_q^{(i)}{}''(x) < 0 $ for all $x \in (0, \infty)$. The case where $q > q_{min}$ is similarly proven. 

Next we prove (\ref{enum:W_deri_lbFxBrCs1}). Suppose $q > \xi_{i-1}$ (with $\n+1 \le i \le \Np$), by (\ref{eq:explicit_Vp00qxFxBrCs1}) we have 
\begin{align}
    W_q^{(i)}{}'(x) > \xi_{i-1} &\Leftrightarrow \frac{\aNEW_i q-\hNEW_i}{\aNEW_i} e^{-\aNEW_i x}+\frac{\hNEW_i}{\aNEW_i} > \xi_{i-1}\\
    &\Leftrightarrow \frac{\aNEW_i q-\hNEW_i }{\aNEW_i} e^{-\aNEW_i x} > \xi_{i-1} - \frac{\hNEW_i}{\aNEW_i}\\
    &\Leftrightarrow e^{-\aNEW_i x} > \frac{\aNEW_i}{\aNEW_i q-\hNEW_i} \left( \xi_{i-1} - \frac{\hNEW_i}{\aNEW_i} \right) \text{, since $\aNEW_i q-\hNEW_i > 0$ by Lemma \ref{suppLem:xiRltnFxBrCs1}}\\
    &\Leftrightarrow -\aNEW_i x > \ln \left[ \frac{\aNEW_i}{\aNEW_i q-\hNEW_i} \left( \xi_{i-1} - \frac{\hNEW_i}{\aNEW_i} \right) \right]\\
    &\Leftrightarrow x < \frac{1}{\aNEW_i} \ln \left[
    \frac{\aNEW_i q-\hNEW_i}{\aNEW_i \left( \xi_{i-1} - \frac{\hNEW_i}{\aNEW_i} \right) }\right]\\
    &\Leftrightarrow  x < \frac{1}{\aNEW_i} \ln \left[
    \frac{q-\hNEW_i/\aNEW_i}{ \left( \xi_{i-1} - \frac{\hNEW_i}{\aNEW_i} \right) }\right] = \bar{E}_i(q),
\end{align}
and the case where $q > q_{min}$ is similarly proven. 

Lastly we prove (\ref{enum:WvalueRangeFxBrCs1}). By (\ref{eq:explicit_V00qxFxBrCs1}) we have
\begin{align}
    \lim_{q \rightarrow \left(\frac{\hNEW_i}{\aNEW_i}\right)^+} W_q^{(i)}(x) &= \lim_{q \rightarrow \left(\frac{\hNEW_i}{\aNEW_i}\right)^+} \frac{(\aNEW_i q-\hNEW_i) (1-e^{-\aNEW_i x})}{\aNEW_i^{2}}+\frac{\hNEW_i x}{\aNEW_i}\\
    &= 0 \cdot \frac{ (1-e^{-\aNEW_i x})}{\aNEW_i^{2}}+\frac{\hNEW_i x}{\aNEW_i}\\
    &= \frac{\hNEW_i}{\aNEW_i} x,
\end{align}
and
\begin{align}
\text{ } \lim_{q \rightarrow \infty} W_q^{(i)}(x) &= \lim_{q \rightarrow \infty} \frac{(\aNEW_i q-\hNEW_i) (1-e^{-\aNEW_i x})}{\aNEW_i^{2}}+\frac{\hNEW_i x}{\aNEW_i}\\
&= \infty \cdot \frac{ (1-e^{-\aNEW_i x})}{\aNEW_i^{2}}+\frac{\hNEW_i x}{\aNEW_i}\\
&= \infty. 
\end{align}
\Halmos

\enlargethispage{\baselineskip}
\enlargethispage{\baselineskip}
\enlargethispage{\baselineskip}
\enlargethispage{\baselineskip}

\emph{Proof of Lemma \ref{lem:VForderedFxBrCs1}.}~ 
First we show (\ref{enum:VForderedFxBrCs1_1}). The construction process (\ref{pfDef:cstrtVFFxBrCs1}) ensures that $V_E(\cdot;q)$ is $C^1$ over $(0,\infty)$. To show $V_E(\cdot;q)$ is $C^2$, it suffices to only look at the switch points $\bar{E}_i(q)$, $\bar{E}_i(q) + \bar{D}_{i-1}$, .... 

Suppose $i = \n$, then $\bar{E}_i(q) = \infty$ and there is nothing to prove. Otherwise suppose $i \ge \n+1$, i.e. $\bar{E}_i(q) < \infty$. Given a switch point $x_0 = \bar{E}_i(q) + \sum_{k=j}^{i-1} \bar{D}_k$ (where $\n+1 \le j \le i$), we have (by the construction process (\ref{pfDef:cstrtVFFxBrCs1}) and Lemma \ref{lem:WDEFxBrCs1})
\begin{align}
    \lim_{x \rightarrow x_0^-} V_E''(x;q) &= W_{\xi_j}^{(j)}{}''(\bar{D}_j) \\
    &= \frac{1}{\sigma_j^2} \left(2 c_j - 2 \mu_j W_{\xi_j}^{(j)}{}'(\bar{D}_j)\right)\\
    &= \hNEW_j - \aNEW_j \xi_{j-1},
\end{align}
and 
\begin{align}
    \lim_{x \rightarrow x_0^+} V_E''(x;q)
    &= W_{\xi_{j-1}}^{(j-1)}{}''(0)\\
    &= \frac{1}{\sigma_{j-1}^2} \left(2 c_{j-1} - 2 \mu_{j-1} W_{\xi_{j-1}}^{({j-1})}{}'(0)\right)\\
    &= \hNEW_{j-1} - \aNEW_{j-1} \xi_{j-1}.
\end{align}
Since $\hNEW_{j-1} - \aNEW_{j-1} \xi_{j-1} = \hNEW_j - \aNEW_j \xi_{j-1}$ by the definition of $\xi_{j-1}$, we have $\lim_{x \rightarrow x_0^-} V_E''(x;q) = \lim_{x \rightarrow x_0^+} V_E''(x;q)$. Hence $ V_E(\cdot;q)$ is $C^2$.

Next we prove (\ref{enum:VForderedFxBrCs1_2}). Define 
\begin{align}
    \begin{cases}
    S_1 = \bar{E}_{i}(q),\\
    S_2 = \bar{E}_{i}(q) + \bar{D}_{i-1},\\
    S_3 = \bar{E}_{i}(q) + \bar{D}_{i-1} + \bar{D}_{i-2},\\
    ...
    \end{cases}
\end{align}
By (\ref{enum:VForderedFxBrCs1_1}) $V_E(\cdot;q)$ is concave on each of the following sub-intervals $(0, S_1)$, $(S_1, S_2)$, .... Since $V_E(\cdot;q)$ is also $C^2$ on $S_1, S_2, ...$ (by (\ref{enum:VForderedFxBrCs1_1})),  $V_E(\cdot;q)$ is concave.

We now proceed to show (\ref{enum:VForderedFxBrCs1_3}). By Lemma \ref{lem:WDEFxBrCs1}(\ref{enum:W_deri_lbFxBrCs1}), $V_E'(\cdot;q)$ is decreasing but is always greater than $q_{min}$. To prove the two limiting results, note that 1. when $q_{min} < q \le \xi_{\n}$ then $V_E(x;q) = W_q^{(\n)}(x)$ since $\bar{\Gamma}_{init}(q) = \n$ and $\bar{E}_{\n}(q) = \infty$. By (\ref{enum:WvalueRangeFxBrCs1}), $\lim_{\hat{q} \rightarrow q_{min}^+} V_E(x;\hat{q}) = q_{min} x$; 2. when $q$ is sufficiently large (i.e. $q > \xi_{\Np-1}$),  $\bar{\Gamma}_{init}(q) = \Np$. By Lemma \ref{lem:EiqMonoInqFxBrCs1}, when $q$ is large, $\bar{E}_{\Np}(q) > x$. Hence $\lim_{\hat{q} \rightarrow \infty} V_E(x;\hat{q}) = \lim_{\hat{q} \rightarrow \infty} W_q^{(\Np)}(x) = \infty$ by (\ref{enum:WvalueRangeFxBrCs1}).

Lastly, we proceed to show (\ref{enum:VForderedFxBrCs1_4}). Let $q$, $q'$ be fixed with $q_{min} < q < q'$. Let $x_0 > 0$ be fixed and we will show that $V_E(x_0,q') > V_E(x_0,q)$. By definition $V_E'(0,q') = q' > q$. Since $V_E''(x,q') < 0$ always, there exists $x^+ > 0$ s.t. $V_E'(x^+,q') = q$. By the construction process (\ref{pfDef:cstrtVFFxBrCs1}) of $V_E(\cdot,q')$ and $V_E(\cdot,q)$, we have
\begin{align*}
    V_E'(\cdot + x^+,q') = V_E'(\cdot,q).
\end{align*}
Hence given $x_0 > 0$, we have 
\begin{align*}
    V_E(x_0,q) &= \int_0^{x_0} V_E'(x,q) dx + V_E(0,q)\\
    &= \int_0^{x_0} V_E'(x+x^+,q') dx + 0\\
    &\le \int_0^{x_0} V_E'(x,q') dx + V_E(0,q')\\
    &= V_E(x_0,q').
\end{align*}
where the inequality is due to (\ref{enum:VForderedFxBrCs1_1}). \Halmos

\begin{restatable}{lemma}{suppLem:xiRltnFxBrCs1}
\label{suppLem:xiRltnFxBrCs1}
Let $\xi_i$'s be defined as in (\ref{eq:xiDefFxBrCs1}). For controls $\n+1, ...,\Np$, we have
\begin{align}
    \xi_{i} > \xi_{i-1} > \hNEW_i/\aNEW_i.
\end{align}
\end{restatable}

\emph{Proof of Lemma \ref{suppLem:xiRltnFxBrCs1}.}~ 
The first inequality holds by the definition of \textit{efficient frontier} . For the second inequality, $\xi_{i-1} > \hNEW_i/\aNEW_i$ is equivalent to $\frac{\hNEW_i - \hNEW_{i-1}}{\aNEW_i - \aNEW_{i-1}} > \hNEW_i/\aNEW_i$, which is equivalent to $(\hNEW_i - \hNEW_{i-1}) \aNEW_i > \hNEW_i (\aNEW_i - \aNEW_{i-1})$, i.e. $\hNEW_{i-1} \aNEW_i < \hNEW_i  \aNEW_{i-1}$, i.e. $\hNEW_{i-1}/\aNEW_{i-1}  < \hNEW_i/\aNEW_i  $, which holds for controls $\n+1,...,\Np$ (and note that $\frac{\hNEW_{\n}}{\aNEW_{\n}} = \min_{i \in \I} \frac{\hNEW_i}{\aNEW_i}$).

\enlargethispage{\baselineskip}
\enlargethispage{\baselineskip}
\enlargethispage{\baselineskip}
\enlargethispage{\baselineskip}

\subsection{Proof of Theorem \ref{thm:Kctrls_FxLB} Case (2)}
\label{sec:PfThmFxLB_case2}

The proof mimics that of Theorem \ref{thm:Kctrls_FxLB} Case (1) and is presented below.


\subsubsection{Step I: Notations and Definitions}~
\label{subsec:FxBrCs2PfStep1}
Let $\{\aNEW_i,\hNEW_i\}_{i=\Nm}^{\Np}$ be given. Define a sequence of thresholds $\{\xi_i\}_{\Nm}^{\n - 1}$ as
\begin{align}
    \xi_i = \frac{\hNEW_{i+1} - \hNEW_i}{\aNEW_{i+1} - \aNEW_i}\text{ , $i = \Nm,\Nm+1...,\n-1$.} \label{eq:xiDefFxBrCs2}
\end{align}

Define $q_{min} = \hNEW_{\n}/\aNEW_{\n}$, a function $\underline{\Gamma}_{init}(\cdot)$ that maps any value of $q < q_{min}$ to a control index,  functions $\underline{E}_i(\cdot)'s$ and quantities $\underline{D}_i's$ given by
\begin{align}
    \underline{\Gamma}_{init}(q) &= 
    \begin{cases}
    \Nm \text{ if $q < \xi_{\Nm}$}\\
    \Nm + 1 \text{ if $\xi_{\Nm} \le q < \xi_{\Nm+1}$}\\
    ...\\
    \n-1 \text{ if $\xi_{\n-2} \le q < \xi_{\n-1}$}\\
    \n \text{ if $\xi_{\n-1} \le q < q_{min}$}\\
    \end{cases}\\
    \underline{E}_i(q) &= \begin{cases} \frac{1}{\aNEW_i} \ln{\frac{ q-\hNEW_i/\aNEW_i}{\xi_i-\hNEW_i/\aNEW_i}} \text{, $i = \Nm, ..., \n-1$, $\xi_{i-1} \le q < \xi_i$}\\
    \infty\text{, $i = \n$, $\xi_{\n-1} \le q < q_{min}$}\\
    \end{cases} \label{pfEq:EiqFxBrCs2}\\
    \underline{D}_i &= \frac{1}{\aNEW_i} \ln{\frac{ \xi_{i-1} - \hNEW_i/\aNEW_i}{ \xi_i-\hNEW_i/\aNEW_i}} \text{, $i = \Nm+1, ..., \n-1$}. \label{pfEq:DiFxBrCs2}
\end{align}


In addition, define $W_q^{(i)} (x)$ to be the unique solution of $W(x)$
\begin{align}
\label{eq:a_x0y0q_xFxBrCs2}
    -c_i + \mu_i W'(x) + \frac{\sigma_i^2}{2} W''(x) = 0, W(0) = 0, W'(0) = q.
\end{align}


\subsubsection{Step II: Construction of $V_I(x;q)$ for a given $q$}
\label{subsec:FxBrCs2PfStep2}
Given a value of $q \in (-\infty, q_{min})$, let $i = \underline{\Gamma}_{init}(q) \in \{\Nm,\Nm+1,...,\n\}$ and we can construct the following trial value function:

{\scriptsize
\begin{align}
\label{pfDef:cstrtVFFxBrCs2}
V_I(x;q) &=
    \begin{cases}
    W_q^{(i)}(x) \text{ if $0\le x<\underline{E}_i(q)$}\\
    W_{\xi_{i}}^{(i+1)}(x - \underline{E}_i(q)) + W_q^{(i)}(\underline{E}_i(q)) \text{ if $\underline{E}_i(q) \le x < \underline{E}_i(q) + \underline{D}_{i+1}$}\\
    W_{\xi_{i+1}}^{(i+2)}(x - \underline{E}_i(q) - \underline{D}_{i+1}) +  W_{\xi_i}^{(i+1)}(\underline{D}_{i+1}) + W_q^{(i)}(\underline{E}_i(q)) \text{ if $\underline{E}_i(q) + \underline{D}_{i+1} \le x < \underline{E}_i(q) + \underline{D}_{i+1} + \underline{D}_{i+2}$}\\
    ...\\
    W_{\xi_{\n-2}  }^{(\n-1)}(x - \underline{E}_i(q) - \sum_{j=i+1}^{\n-2} \underline{D}_j) + \sum_{j=i+1}^{\n-2} W_{\xi_{j-1}}^{(j)}(\underline{D}_j) + W_q^{(i)}(\underline{E}_i(q))  \text{ if $\underline{E}_i(q) + \sum_{j=i+1}^{\n-2} \underline{D}_j \le x < \underline{E}_i(q) + \sum_{j=i+1}^{\n-1} \underline{D}_j$}\\
    W_{\xi_{\n-1}  }^{(\n)}(x - \underline{E}_i(q) - \sum_{j=i+1}^{\n-1} \underline{D}_j) + \sum_{j=i+1}^{\n-1} W_{\xi_{j-1}}^{(j)}(\underline{D}_j) + W_q^{(i)}(\underline{E}_i(q))  \text{ if $x \ge \underline{E}_i(q) + \sum_{j=i+1}^{\n-1} \underline{D}_j$}\\
    \end{cases}
\end{align}
}
and note that $V_I(x;q)$ is defined over $[0,\infty)$. We can prove some properties of $V_I(x;q)$ and show that the set $\{V_I(\cdot;q))\}_{q \in (-\infty, q_{min})}$ is ordered.

\enlargethispage{\baselineskip}
\enlargethispage{\baselineskip}
\enlargethispage{\baselineskip}
\enlargethispage{\baselineskip}

\begin{restatable}{lemma}{VForderedAllIneff}
\label{lem:VForderedFxBrCs2}
Suppose $q \in (-\infty, q_{min})$ and $V_I(x;q)$ is constructed as in (\ref{pfDef:cstrtVFFxBrCs2}), then 
\begin{enumerate}[\hspace{0.5cm} a)]
    \item Over $(0,\infty)$, $V_I(\cdot;q)$ is $C^2$.  \label{enum:VForderedFxBrCs2_1}
    \item Over $(0,\infty)$, $V_I(\cdot;q)$ is convex, i.e. $V_I''(\cdot;q) > 0$. \label{enum:VForderedFxBrCs2_2}
    \item For $x \in (0,\infty)$, $V_I'(x;q) \in (q_{min}, q)$, $\lim_{\hat{q} \rightarrow q_{min}^-} V_I(x;\hat{q}) = q_{min} x$, $\lim_{\hat{q} \rightarrow -\infty} V_I(x;\hat{q}) = -\infty$. \label{enum:VForderedFxBrCs2_3}
    \item The functions $\{V_I(\cdot;q))\}_{q \in (-\infty, q_{min})}$ are ordered, with a smaller $q$ producing uniformly smaller function $V_I(\cdot;q)$. That is, for any $x \in (0,\infty)$, if $q < q'$ then $V_I(x;q) < V_I(x;q')$. \label{enum:VForderedFxBrCs2_4}
\end{enumerate}
\end{restatable}

\bproof The Proof follows similar steps as Lemma \ref{lem:VForderedFxBrCs1} and is omitted. \Halmos

Suppose $\LNEW < \barLNEW$, then this is equivalent to $\MNEW < q_{min} (\UNEW-\LNEW)$. By Lemma \ref{lem:VForderedFxBrCs2} (\ref{enum:VForderedFxBrCs2_3}) and (\ref{enum:VForderedFxBrCs2_4}), there exists a unique $q^*$ s.t. $V_I(\cdot;q^*)$ satisfies
\begin{align}
    V_I(0;q^*) = 0, \text{ } V_I(\UNEW-\LNEW;q^*) = \MNEW. \label{pfEq:VFbdryCondFxBrCs2}
\end{align} 
Let 
\begin{align}
    i_1 &= \underline{\Gamma}_{init}(q^*) \label{eq:PfOptPlcCtrl1FxBrCs2}\\
    S_0 &= 0 \label{eq:PfOptPlcS0FxBrCs2}
\end{align} 
Since $q^* \in (-\infty, q_{min})$, we have $\Nm \le i_1 \le \n$. Define $\{S_j\}_{j=1}^{\n - i_1 + 1}$ to be:


\begin{align}
\label{eq:SiFxBrCs2}
    \begin{cases}
    S_1 = \underline{E}_{i_1}(q^*)\\
    S_2 = \underline{E}_{i_1}(q^*) + \underline{D}_{i_1+1}\\
    S_3 = \underline{E}_{i_1}(q^*) + \underline{D}_{i_1+1} + \underline{D}_{i_1+2}\\
    ...\\
    S_{k-1} = \underline{E}_{i_1}(q^*) + \sum_{j=1}^{k-2} \underline{D}_{i_1+j}\\
    S_k = \underline{E}_{i_1}(q^*) + \sum_{j=1}^{k-1} \underline{D}_{i_1+j}\\
    ...\\
    S_{\n-i_1} = \underline{E}_{i_1}(q^*) + \sum_{j=1}^{\n-i_1-1} \underline{D}_{i_1+j}\\
    S_{\n-i_1+1} = \infty,
    \end{cases}
\end{align}
and define $\tK = \inf\{j \lvert S_j \ge \UNEW-\LNEW\}$, i.e. $S_{\tK-1} < \UNEW-\LNEW$ and $S_{\tK} \ge \UNEW-\LNEW$. Further define $S^*_0 = \LNEW + S_0$, $S^*_1 = \LNEW + S_1$, ..., $S^*_{\tK-1}= \LNEW + S_{\tK-1}$, $S^*_{\tK}= \UNEW \le \LNEW + S_{\tK}$. Construct the interval policy $\pi^*_I = \{i^*_j,I^*_j\}_{j=1}^{\tK} $ as follows
\begin{align}
\label{eq:ijFxBrCs2}
    i^*_j &= i^*_1 + (j-1) \text{ for $1 \le j \le \tK$,}\\I^*_j &= [S^*_{j-1},S^*_j)\text{ for $1 \le j \le \tK$,}
\end{align}
and define $\hat{V}_I(\cdot;q^*)$ as
\begin{align}
    \hat{V}_I(\cdot - \LNEW;q^*) = V_I(\cdot;q^*). \label{eq:VhatShiftLFxBrCs2}
\end{align}
We will prove that $\pi^*_I$ is an optimal policy (with $\hat{V}_I(\cdot;q^*)$ being the associated value function). 

\subsubsection{Step III: Proof of Theorem \ref{thm:Kctrls_FxLB} Case (2)}
\label{subsec:FxBrCs2PfStep3}

\enlargethispage{\baselineskip}
\enlargethispage{\baselineskip}
\enlargethispage{\baselineskip}
\enlargethispage{\baselineskip}

\emph{Proof of Theorem \ref{thm:Kctrls_FxLB} Case (2).}~ We will now prove Theorem  \ref{thm:Kctrls_FxLB} Case (2). In particular, we will show that the interval policy $\pi^*_I$ whose value function $\hat{V}_I(\cdot;q^*)$ is $\mathcal{C}^2$ over $[\LNEW,\UNEW)$ and satisfies
\begin{align}
&V(\LNEW) = 0, \text{ } V(\UNEW) = \MNEW, \label{eq:PfsuffCondOptPlcFxBrCs2_1}\\
& -c_i + \mu_i V'(x) + \frac{\sigma_i^2}{2}V''(x) \le 0 \text{, for all $x \in [\LNEW,\UNEW)$ and $i \in \I$,} \label{eq:PfsuffCondOptPlcFxBrCs2_init2} \\
& -c_{i^*_j} + \mu_{i^*_j} V'(x) + \frac{\sigma_{i^*_j}^2}{2}V''(x) = 0 \text{, for $x \in [S_{j-1}^*,S_j^*)$, $1 \le j \le \tK$}, \label{eq:PfsuffCondOptPlcFxBrCs2_3}
\end{align}
and hence the interval policy is optimal.

By (\ref{pfEq:VFbdryCondFxBrCs2}) and (\ref{eq:VhatShiftLFxBrCs2}), $\hat{V}_I(\cdot;q^*)$ satisfies (\ref{eq:PfsuffCondOptPlcFxBrCs2_1}) and (\ref{eq:PfsuffCondOptPlcFxBrCs2_3}). We will now establish
\begin{align}
& -c_i + \mu_i \hat{V}_I'(x;q^*) + \frac{\sigma_i^2}{2}\hat{V}_I''(x;q^*) \le 0 \text{, for all $x \in [\LNEW,\UNEW)$ and $i \in \I$.} \label{eq:PfsuffCondOptPlcFxBrCs2_2}
\end{align}
For convenience, we will use the following terminology: we say that control $j$ dominates control $i$ at $x$ if 
\begin{align}
-c_i + \mu_i \hat{V}_I'(x;q^*) + \frac{\sigma_i^2}{2}\hat{V}_I''(x;q^*)  \le -c_j + \mu_j \hat{V}_I'(x;q^*) + \frac{\sigma_j^2}{2}\hat{V}_I''(x;q^*). \\
\nonumber
\end{align}

Given any interval $I^*_j = [S^*_{j-1},S^*_j)$, let $k = i^*_j$ be the control used, i.e. 
\begin{align}
    -&c_k + \mu_k \hat{V}_I'(x;q^*) + \frac{\sigma_k^2}{2}\hat{V}_I''(x;q^*) = 0,\\
    &\xi_{k-1} \le \hat{V}_I'(x;q^*) < \xi_k \text{ over } [S^*_{j-1},S^*_j), \label{pfEq:dmntInterm0FxBrCs2}
\end{align}
and from now on fix $x \in [S^*_{j-1},S^*_j)$. 

For a control $i \in \Phi = \{\Nm, ..., 1,2, ..., \Np\}$, we have 
\begin{align}
    -c_i + \mu_i \hat{V}_I'(x;q^*) + \frac{\sigma_i^2}{2}\hat{V}_I''(x;q^*) \le -c_k + \mu_k \hat{V}_I'(x;q^*) + \frac{\sigma_k^2}{2}\hat{V}_I''(x;q^*) 
\end{align}
if and only if (c.f. Lemma \ref{supp_lem:dmntEqvFxBr})
\begin{align}
    -\hNEW_i + \aNEW_i \hat{V}_I'(x;q^*) + \hat{V}_I''(x;q^*) \le -\hNEW_k + \aNEW_k \hat{V}_I'(x;q^*) + \hat{V}_I''(x;q^*), 
\end{align}
which is equivalent to 
\begin{align}
    \hNEW_k -\hNEW_i \le (\aNEW_k - \aNEW_i) \hat{V}_I'(x;q^*). \label{pfEq:dmntInterm1FxBrCs2}
\end{align}
Suppose $\aNEW_i < \aNEW_k$, then (\ref{pfEq:dmntInterm1FxBrCs2}) is equivalent to $\frac{\hNEW_k - \hNEW_i}{\aNEW_k - \aNEW_i} \le \hat{V}_I'(x;q^*)$, which is true because $\frac{\hNEW_k - \hNEW_i}{\aNEW_k - \aNEW_i} \le \frac{\hNEW_k - \hNEW_{k-1}}{\aNEW_k - \aNEW_{k-1}} = \xi_{k-1} \le \hat{V}_I'(x;q^*)$ (the $1^{st}$ inequality is due to control $i \in \Phi$ and the $2^{nd}$ inequality is due to (\ref{pfEq:dmntInterm0FxBrCs2})). Suppose $\aNEW_i > \aNEW_k$, then (\ref{pfEq:dmntInterm1FxBrCs2}) is equivalent to $\frac{\hNEW_i - \hNEW_k}{\aNEW_i - \aNEW_k} \ge \hat{V}_I'(x;q^*)$, which is true because $\frac{\hNEW_i - \hNEW_k}{\aNEW_i - \aNEW_k} \ge  \frac{\hNEW_{k+1} - \hNEW_k}{\aNEW_{k+1} - \aNEW_k} = \xi_k > \hat{V}_I'(x;q^*)$ (the $1^{st}$ inequality is due to control $i \in \Phi$ and the $2^{nd}$ inequality is due to (\ref{pfEq:dmntInterm0FxBrCs2})).

For a control $i \not\in \Phi$, then it can be easily verified that for any $x \in (\LNEW,\UNEW)$ there exists at least one $j \in \Phi$ s.t. control $i$ is dominated by $j$ at $x$. Thus (\ref{eq:PfsuffCondOptPlcFxBrCs2_2}) is established.

In summary, the constructed $\hat{V}_I(\cdot;q^*)$ is the desired value function, which is $C^2$ over $(\LNEW,\UNEW)$ and satisfies (\ref{eq:PfsuffCondOptPlcFxBrCs2_1}), (\ref{eq:PfsuffCondOptPlcFxBrCs2_init2}) and (\ref{eq:PfsuffCondOptPlcFxBrCs2_3}). \Halmos

\subsection{Proof of Propositions \ref{prop:Kctrls_FxLB} and 
\ref{prop:FxLBMyoPlcBad}}
\label{sec:PfPropFxLB}~
\emph{Proof of Proposition \ref{prop:Kctrls_FxLB}.}~ This is just a direct consequence of Theorem \ref{thm:Kctrls_FxLB}. \Halmos

\emph{Proof of Proposition \ref{prop:FxLBMyoPlcBad}.}~ 
Fix $\UNEW=1$, $\MNEW=1$, $\LNEW=0$. We will construct a sequence of problem instances parametrized by $\epsilon > 0$ such that when $\epsilon$ goes to $0$, we have
\begin{align}
    \lim_{\epsilon \rightarrow 0^+}\sup_{x \in (\LNEW,\UNEW)} V_{opt}(x;\epsilon) - V_{myo}(x;\epsilon) = 1, 
\end{align}
where $V_{opt}(\cdot;\epsilon)$ and $V_{myo}(\cdot;\epsilon)$ denote the value function of the optimal and myopic policy, respectively, under the problem instance parametrized by $\epsilon$.

In problem instance parametrized by $\epsilon > 0$, we define two controls with $(\aNEW_1,\hNEW_1) = (1,\epsilon)$, $(\aNEW_2,\hNEW_2) = (1/\epsilon,2)$. Then $\aNEW_1 / \hNEW_1 > \aNEW_2 / \hNEW_2$, i.e. control $1$ has a higher \textit{cost-effectiveness}. Denote $V_1(x;\epsilon)$, $V_2(x;\epsilon)$, $V_{opt}(x;\epsilon)$, $V_{myo}(x;\epsilon)$ as the value function of using $1$ only, using $2$ only, the optimal policy and the myopic policy (for the fixed lower boundary problem) respectively, then we have 
\begin{align}
    V_1(x;\epsilon) &= V_{myo}(x;\epsilon) = -\frac{e^{-x} (\epsilon-1)}{e^{-1}-1}+\epsilon x+\frac{\epsilon-1}{e^{-1}-1},\\
    V_2(x;\epsilon) &=  -\frac{e^{-\frac{x}{\epsilon}} (2 \epsilon-1)}{e^{-\frac{1}{\epsilon}}-1}+2 \epsilon x+\frac{2 \epsilon-1}{e^{-\frac{1}{\epsilon}}-1},
\end{align}

We have
\begin{align}
    V_{opt}(x;\epsilon) - V_{myo}(x;\epsilon) &\ge V_2(x;\epsilon) - V_{opt}(x;\epsilon)\\
    &= V_2(x;\epsilon) - V_1(x;\epsilon)\\
    &= -\frac{e^{-\frac{x}{\epsilon}} (2 \epsilon-1)}{e^{-\frac{1}{\epsilon}}-1}+\epsilon x+\frac{2 \epsilon-1}{e^{-\frac{1}{\epsilon}}-1}+\frac{e^{-x} (\epsilon-1)}{e^{-1}-1}-\frac{\epsilon-1}{e^{-1}-1}.
\end{align}

By Lemma \ref{ec_lem:fxLBMyoPlcBadEpsLmt} we have
\begin{align}
    \lim_{\epsilon \rightarrow 0^+} V_{opt}(x;\epsilon) - V_{myo}(x;\epsilon) &\ge \frac{e^{x} e^{-1}-1}{e^{x} e^{-1}-e^{x}}, 
\end{align}
and consequently,
\begin{align}
    \lim_{\epsilon \rightarrow 0^+}\sup_{x \in (\LNEW,\UNEW)} V_{opt}(x;\epsilon) - V_{myo}(x;\epsilon) &\ge \lim_{\epsilon \rightarrow 0^+}\lim_{x \rightarrow 0^+} V_{opt}(x;\epsilon) - V_{myo}(x;\epsilon) \\
    &= \lim_{x \rightarrow 0^+}\lim_{\epsilon \rightarrow 0^+} V_{opt}(x;\epsilon) - V_{myo}(x;\epsilon) \\
    &\ge \lim_{x \rightarrow 0^+} \frac{e^{x} e^{-1}-1}{e^{x} e^{-1}-e^{x}} \\
    &= 1,
\end{align}
where the first equality is due to Bounded Convergence Theorem. 
\Halmos

\enlargethispage{\baselineskip}
\enlargethispage{\baselineskip}
\enlargethispage{\baselineskip}
\enlargethispage{\baselineskip}

\section{Proof of Results in the Free Lower Boundary Extension}
\label{ec_sec:PfRsltsFrLBCs}

Appendix \ref{ec_sec:PfRsltsFrLBCs} presents the proof of results for the free lower boundary case. \S\ref{sec:PfThmFrLB} presents the proof of Theorem \ref{thm:Kctrls_FrLB}, \S\ref{sec:PfPropFrLB} presents the proof of Propositions 
\ref{prop:Kctrls_FrLB_pvw}, 
\ref{prop:LstarRangeFrLB} and \ref{prop:FrLBMyoPlcBad}.

\subsection{Proof of Theorem \ref{thm:Kctrls_FrLB}}
\label{sec:PfThmFrLB}


In this section, we will prove {Theorem} \ref{thm:Kctrls_FrLB} by constructing an optimal interval policy. We follow similar proof steps. We use the same quantities defined as in \S\ref{subsec:FxBrCs2PfStep1} and construct a value function. After that, we show that the constructed value function corresponds to the desired optimal policy. 

\subsubsection{Step I: Notations and Definitions}~
Let $\{\aNEW_i,\hNEW_i\}_{i=\Nm}^{\Np}$ be given. We define $\xi_i$, $\underline{\Gamma}_{init}$, $\underline{D}_i$, $W_q^{(i)} (x)$ as in \S\ref{subsec:FxBrCs2PfStep1}.

\subsubsection{Step II: Construction of $V_{Fr}(x)$}
We construct the following trial value function:
\label{subsec:FrBrPfStep2}
\begin{align}
\label{pfDef:cstrtVFFrBr}
V_{Fr}(x) &=
    \begin{cases}
    V_I(x;q=0) \text{, if $x \ge 0$,}\\
    0 \text{, if $x < 0$,}
    \end{cases}
\end{align}
where $V_I(x;q)$ is defined as in (\ref{pfDef:cstrtVFFxBrCs2}). Note that $V_{Fr}(x)$ is defined over $(-\infty,\infty)$ and possess the following properties.

\begin{restatable}{lemma}{VFpropertyFrBr}
\label{lem:VFpropertyFrBr}
Let $V_{Fr}(x)$ be constructed as in (\ref{pfDef:cstrtVFFrBr}), then 
\begin{enumerate}[\hspace{0.5cm} a)]
    \item Over $(-\infty,0) \cup (0,\infty)$, $V_{Fr}(\cdot)$ is $C^2$. At $0$,  $V_{Fr}(\cdot)$ is $C^1$. \label{enum:VForderedFrBr_1}
    \item Over $(0,\infty)$, $V_{Fr}(\cdot)$ is convex, i.e. $V_{Fr}''(\cdot) > 0$. \label{enum:VForderedFrBr_2}
    \item $V_{Fr}'(0) = 0$. For $x \in (0,\infty)$, $V_{Fr}'(x) > 0$, $\lim_{x \rightarrow 0} V_{Fr}(x) = 0$, $\lim_{x \rightarrow \infty} V_{Fr}(x) = \infty$. \label{enum:VForderedFrBr_3}
\end{enumerate}
\end{restatable}

\emph{Proof of Lemma \ref{lem:VFpropertyFrBr}.}~ Part (\ref{enum:VForderedFrBr_1}) holds by (\ref{pfDef:cstrtVFFrBr}) and Lemma \ref{lem:VForderedFxBrCs2}(\ref{enum:VForderedFxBrCs2_1}). Part  (\ref{enum:VForderedFrBr_2}) and (\ref{enum:VForderedFrBr_3}) hold by Lemma \ref{lem:VForderedFxBrCs2}(\ref{enum:VForderedFxBrCs2_2}) and (\ref{enum:VForderedFxBrCs2_3}). \Halmos


Since $\MNEW > 0$, by Lemma \ref{lem:VFpropertyFrBr} (\ref{enum:VForderedFrBr_3}) there exists a unique $x^*$ s.t. 
\begin{align}
    V_{Fr}(x^*) = \MNEW. \label{pfEq:VFbdryCondFrBr}
\end{align} 
Let 
\begin{align}
    i_1 &= \underline{\Gamma}_{init}(0) \label{eq:PfOptPlcCtrl1FrBr}\\
    S_0 &= 0 \label{eq:PfOptPlcS0FrBr}
\end{align} 
and define $\LNEW^* = \UNEW - x^*$ and $\{S_j\}_{j=1}^{\n}$ to be:

\enlargethispage{\baselineskip}
\enlargethispage{\baselineskip}
\enlargethispage{\baselineskip}
\enlargethispage{\baselineskip}

\begin{align}
\label{eq:SiFrBr}
    \begin{cases}
    S_1 = \underline{E}_{1}(0)\\
    S_2 = \underline{E}_{1}(0) + \underline{D}_{2}\\
    S_3 = \underline{E}_{1}(0) + \underline{D}_{2} + \underline{D}_{3}\\
    ...\\
    S_{k-1} = \underline{E}_{1}(0) + \sum_{j=1}^{k-1} \underline{D}_{j}\\
    S_k = \underline{E}_{1}(0) + \sum_{j=1}^{k} \underline{D}_{j}\\
    ...\\
    S_{\n-1} = \underline{E}_{1}(0) + \sum_{j=1}^{\n-1} \underline{D}_{j}\\
    S_{\n} = \infty.
    \end{cases}
\end{align}
Define $\tK = \inf\{j \lvert S_j \ge \UNEW-\LNEW^*\}$, i.e. $S_{\tK-1} < \UNEW-\LNEW^*$ and $S_{\tK} \ge \UNEW-\LNEW^*$. Further define $S^*_0 = \LNEW^* + S_0$, $S^*_1 = \LNEW^* + S_1$, ..., $S^*_{\tK-1}= \LNEW^* + S_{\tK-1}$, $S^*_{\tK}= \UNEW \le \LNEW^* + S_{\tK}$. Construct the interval policy $\pi^*_{Fr} = \{i^*_j,I^*_j\}_{j=1}^{\tK} $ as follows
\begin{align}
\label{eq:ijFrBr}
    i^*_j &= j \text{ for $1 \le j \le \tK$,}\\I^*_j &= [S^*_{j-1},S^*_j)\text{ for $1 \le j \le \tK$,}
\end{align}
and define $\hat{V}_{Fr}(\cdot)$ as
\begin{align}
    \hat{V}_{Fr}(\cdot - \LNEW^*) = V_{Fr}(\cdot). \label{eq:VhatShiftLFrBr}
\end{align}
We will prove that $\pi^*_{Fr}$ is an optimal policy (with $\hat{V}_{Fr}(\cdot)$ being the associated value function). 

\subsubsection{Step III: Proof of Theorem \ref{thm:Kctrls_FrLB}}

\enlargethispage{\baselineskip}
\enlargethispage{\baselineskip}
\enlargethispage{\baselineskip}
\enlargethispage{\baselineskip}

\emph{Proof of Theorem \ref{thm:Kctrls_FrLB}.}~ We will now prove Theorem  \ref{thm:Kctrls_FrLB}. In particular, we will show that the interval policy $\pi^*_{Fr}$ whose value function $\hat{V}_{Fr}(\cdot)$ is $\mathcal{C}^2$ over $[\LNEW,\UNEW)$ and satisfies (\ref{eq:FrLB_suffCondOpt1}) - (\ref{eq:FrLB_suffCondOpt5}) and hence the interval policy is optimal.

By (\ref{pfEq:VFbdryCondFrBr}) and (\ref{eq:VhatShiftLFrBr}), $\hat{V}_{Fr}(\cdot)$ satisfies (\ref{eq:FrLB_suffCondOpt1}),  (\ref{eq:FrLB_suffCondOpt2}) and (\ref{eq:FrLB_suffCondOpt5}). By Lemma \ref{lem:VFpropertyFrBr}(\ref{enum:VForderedFrBr_1}), $\hat{V}_{Fr}(\cdot)$ satisfies (\ref{eq:FrLB_suffCondOpt3}). The establishment of (\ref{eq:FrLB_suffCondOpt4}) involves proving
\begin{align}
& -c_i + \mu_i \hat{V}_{Fr}'(x) + \frac{\sigma_i^2}{2}\hat{V}_{Fr}''(x) \le 0 \text{, for all $x \in [\LNEW,\UNEW)$ and $i \in \I$,} \label{eq:FrLB_suffCondOpt2_restate}
\end{align}
which follows the identical procedure of \S\ref{subsec:FxBrCs2PfStep3} and is omitted. \Halmos

\subsection{Proof of Propositions 
\ref{prop:Kctrls_FrLB_pvw}, 
\ref{prop:LstarRangeFrLB} and \ref{prop:FrLBMyoPlcBad}}
\label{sec:PfPropFrLB}~
\emph{Proof of Proposition \ref{prop:Kctrls_FrLB_pvw}.}~ This is just a direct consequence of Theorem \ref{thm:Kctrls_FrLB}. \Halmos

\emph{Proof of Proposition \ref{prop:LstarRangeFrLB}.} First we prove the right inequality. Let $\LNEW_{\n}$ be the optimal lower boundary if we use $\n$ alone. Then we have (c.f. Lemma \ref{lem:Lstar1Ctrl})
\begin{align}
    \UNEW - \LNEW_{\n} = \MNEW \left[\frac{\mu_i}{c_i} (1+G(\MNEW \aNEW_i \frac{\mu_i}{c_i})) \right] \lvert_{i = \n},
\end{align}
where $G(y) = \frac{1}{y}\left[1+LW(-e^{-1-y})\right] > 0$. Moreover since $\UNEW - \barLNEW = \MNEW\frac{\mu_{\n}}{c_{\n}}$ we have
\begin{align}
     \barLNEW - \LNEW_{\n} &= \MNEW \frac{\mu_{\n}}{c_{\n}} G(\MNEW \aNEW_{\n} \frac{\mu_{\n}}{c_{\n}}) \\
     &= \frac{1 + LW(-e^{-1-\MNEW \aNEW_{\n} \frac{\mu_{\n}}{c_{\n}}})}{\aNEW_{\n}}\\
     &= \frac{1}{\aNEW_{\n}} + O(e^{-\MNEW}),
\end{align}
where the last equality is obtained by Taylor expansion at $A = +\infty$, $LW(-e^{-A}) = O(e^{-A})$. Since $\LNEW_{\n} \ge \LNEW^*$, we have
\begin{align}
     \barLNEW - \LNEW^* \ge \barLNEW - \LNEW_{\n} = \frac{1}{a_{\n}} + O(e^{-\MNEW}),
\end{align}
which proves the right inequality.

Next we prove the left inequality. We add in a control with $(\aNEW_{\n}, \hNEW_1)$. Then the optimal policy with all previous controls plus this new one, is to use this new control only. Denote $\LNEW_+$ to be the optimal lower boundary with the new control, then
\begin{align}
    \UNEW - \LNEW_+ = \MNEW \left[\frac{\aNEW_{\n}}{\hNEW_1} (1+G(\MNEW \aNEW_{\n} \frac{\aNEW_{\n}}{\hNEW_1})) \right],
\end{align}
where $G(y) = \frac{1}{y}\left[1+LW(-e^{-1-y})\right] > 0$. Moreover since $\UNEW - \barLNEW = \MNEW\frac{\mu_{\n}}{c_{\n}} = \MNEW\frac{\aNEW_{\n}}{\hNEW_{\n}}$ we have
\begin{align}
    \barLNEW - \LNEW_+ &= \MNEW \left[\frac{\aNEW_{\n}}{\hNEW_1} (1+G(\MNEW \aNEW_{\n} \frac{\aNEW_{\n}}{\hNEW_1})) \right] - \MNEW\frac{\aNEW_{\n}}{\hNEW_{\n}} \\
    &= \MNEW\frac{\aNEW_{\n}}{\hNEW_1} - \MNEW\frac{\aNEW_{\n}}{\hNEW_{\n}} + \MNEW \frac{\aNEW_{\n}}{\hNEW_1} G(\MNEW \aNEW_{\n} \frac{\aNEW_{\n}}{\hNEW_1}) \\
    &=\MNEW\frac{\aNEW_{\n}}{\hNEW_{\n}} (\frac{\hNEW_{\n}}{\hNEW_1} - 1) + \frac{1 + LW(-e^{-1-\MNEW \aNEW_{\n} \frac{\aNEW_{\n}}{\hNEW_1}})}{\aNEW_{\n}}\\
    &=\MNEW\frac{\aNEW_{\n}}{\hNEW_{\n}} (\frac{\hNEW_{\n}}{\hNEW_1} - 1) + \frac{1}{\aNEW_{\n}} + O(e^{-\MNEW})\\
    &\le \MNEW\frac{\aNEW_{\n}}{\hNEW_{\n}} (\frac{\aNEW_{\n}}{\aNEW_1} - 1) + \frac{1}{\aNEW_{\n}} + O(e^{-\MNEW}),
\end{align}
where the last inequality comes from $\frac{\hNEW_{\n}}{\hNEW_1} \le \frac{\aNEW_{\n}}{\aNEW_1}$, and this is true because control $\n$ has the highest \textit{CE}, i.e. $\frac{\aNEW_1}{\hNEW_1} \le \frac{\aNEW_{\n}}{\hNEW_{\n}}$. Since $\MNEW\frac{\aNEW_{\n}}{\hNEW_{\n}} = (\UNEW - \barLNEW)$ and $\LNEW_+ \le \LNEW^*$, we have
\begin{align}
     \barLNEW - \LNEW^* \le \barLNEW - \LNEW_+ \le  (\UNEW - \barLNEW)(\frac{\aNEW_{\n}}{\aNEW_1} - 1) + \frac{1}{\aNEW_{\n}} + O(e^{-\MNEW}),
\end{align}
which proves the left inequality. 

Note that in (\ref{eq:FrLB_LstarRangeFrLB2}) both inequalities are tight. This is shown in Lemma \ref{lem:FrLB_LstarRange_ineq}. \Halmos

\emph{Proof of Proposition \ref{prop:FrLBMyoPlcBad}.}~ Fix $\UNEW=1$, $\MNEW=1$. We will construct a sequence of problem instances parametrized by $\epsilon > 0$ such that when $\epsilon$ goes to $0$, we have
\begin{align}
    \lim_{\epsilon \rightarrow 0^+}\sup_{x \in (-\infty,\UNEW)} V_{opt}(x;\epsilon) - V_{myo}(x;\epsilon) = 1, \label{eq:prop:FrLBMyoPlcBad}
\end{align}
where $V_{opt}(\cdot;\epsilon)$ and $V_{myo}(\cdot;\epsilon)$ denote the value function of the optimal and myopic policy, respectively, under the problem instance parametrized by $\epsilon$.

In problem instance parametrized by $\epsilon > 0$, we define two controls with $(\aNEW_1,\hNEW_1) = (\epsilon^2, 2\epsilon^3)$, $(\aNEW_2,\hNEW_2) = (1, \epsilon)$. Then $\aNEW_1 / \hNEW_1 < \aNEW_2 / \hNEW_2$, i.e. control $2$ has a higher \textit{cost-effectiveness}. Denote $V_1(x;\epsilon)$, $V_2(x;\epsilon)$, $V_{opt}(x;\epsilon)$, $V_{myo}(x;\epsilon)$ as the value function of using $1$ only, using $2$ only, the optimal policy and the myopic policy (for the free lower boundary problem) respectively (without loss of generality we shift the $x$-coordinate of all value functions such that $V_1(0) = 0$). 
\begin{align}
    V_1(x;\epsilon) &= \frac{2 e^{-\epsilon^{2} x}}{\epsilon}+2 \epsilon x-\frac{2}{\epsilon},\\
    V_2(x;\epsilon) &= V_{myo}(x;\epsilon) = \tilde{V}_2(x - (G_1-G_2);\epsilon), \label{lem_eq:FrLBMyoBadV2Def}
\end{align}
where 
\begin{align}
  \tilde{V}_2(x;\epsilon) = e^{-x} \epsilon+\epsilon x-\epsilon,
\end{align}
and
\begin{align}
G_1(\epsilon) =\frac{2 LW(-e^{-(\frac{1}{2}) \epsilon-1})+\epsilon+2}{2 \epsilon^{2}}, \text{ } G_2(\epsilon) =LW(-e^{-\frac{\epsilon+1}{\epsilon}})+\frac{\epsilon+1}{\epsilon}.
\end{align}
We can show that by direct computation, \eqref{eq:prop:FrLBMyoPlcBad} holds. \Halmos

\section{Numerical Example for the Free Lower Boundary Case}
\label{ec_sec:myoPlcBad_FrLB}


Similarly as in the fixed lower boundary case, for the free lower boundary problem the myopic policy can perform arbitrarily worse than the optimal policy. 


Figure \ref{fig:myoPlcBad_FrLB_Plc} provides a numerical example that illustrate when this difference can be large. In this example, $\UNEW=90$, $\MNEW=10$ and there are three controls $(\aNEW_1,\hNEW_1) = (0.003, 0.0033)$, $(\aNEW_2,\hNEW_2) = (0.006, 0.0036)$, $(\aNEW_3,\hNEW_3) = (1, 0.54)$. The myopic policy $\pi_{myo}$ is to always use control $3$ (highest \textit{cost-effectiveness}), with an associated lower boundary $\LNEW_3 = 71.7$. The optimal policy is ``12'' with an associate lower boundary $\LNEW^* = 9.2$. Figure \ref{fig:MyoPlcBad_FrLB_VF} plots the value function of both policies. The biggest gap  between the two value functions is $6.1$, which occurs when the state $X(t)$ is at $72.2$. Contrary to the numerical example of the fixed lower boundary case, the difference can be large even when $X(t)$ is close to $\UNEW$ (but not too close). 

\begin{figure}[h]
    \centering
    \begin{subfigure}[b]{0.45\textwidth}
    \begin{tikzpicture}[xscale = 1.7]
        \def\LPos{0.8}
        \def\UPos{0.9}
        \def\SPos{0.85}  
    
        \draw[->] (-0.5,1) -- (2.5,1) node[pos=1, right, black] {$X(t)$};
        \draw[-, blue] (-0.5+\LPos*3, 1) -- (-0.5+\UPos*3, 1)  node[pos=0.5, above] {use $3$};
        \draw[-] (-0.5+\LPos*3, 1.1) -- (-0.5+\LPos*3, 0.9)  node[pos=1, below, black] {$\LNEW_3$};
        \draw[-] (-0.5+\UPos*3, 1.1) -- (-0.5+\UPos*3, 0.9)  node[pos=1, below, black] {$\UNEW$};
    \end{tikzpicture}
    \caption{Myopic policy is ``always use 3''}
    \end{subfigure}
    \begin{subfigure}[b]{0.45\textwidth}
    \begin{tikzpicture}[xscale = 1.7]
        \def\LPos{0.1}
        \def\UPos{0.9}
        \def\SPos{0.45}  
    
        \draw[->] (-0.5,1) -- (2.5,1) node[pos=1, right, black] {$X(t)$};
        \draw[-, cyan] (-0.5+\LPos*3, 1) -- (-0.5+\SPos*3, 1)  node[pos=0.5, above] {use $1$};
        \draw[-, red] (-0.5+\SPos*3, 1) -- (-0.5+\UPos*3, 1)  node[pos=0.5, above] {use $2$};
        \draw[-] (-0.5+\LPos*3, 1.1) -- (-0.5+\LPos*3, 0.9)  node[pos=1, below, black] {$\LNEW^*$};
        \draw[-] (-0.5+\UPos*3, 1.1) -- (-0.5+\UPos*3, 0.9)  node[pos=1, below, black] {$\UNEW$};
        \draw[-] (-0.5+\SPos*3, 1.1) -- (-0.5+\SPos*3, 0.9)  node[pos=1, below, black] {$S_1$};
    \end{tikzpicture}
    \caption{Optimal policy is ``12''.}
    \end{subfigure}
\caption{A numerical example (free lower boundary case) that shows myopic policy can perform much worse than the optimal policy. Parameters are:
$\UNEW=90$, $\MNEW=10$, $(\aNEW_1,\hNEW_1) = (0.003, 0.0033)$, $(\aNEW_2,\hNEW_2) = (0.006, 0.0036)$, $(\aNEW_3,\hNEW_3) = (1, 0.54)$. The myopic policy $\pi_{myo}$ is ``$3$'' with $\LNEW_3 = 71.7$. The optimal policy is ``$12$'' with $\LNEW^* = 9.2$ and $S_1 = 47.3$. }
\label{fig:myoPlcBad_FrLB_Plc}
\end{figure}
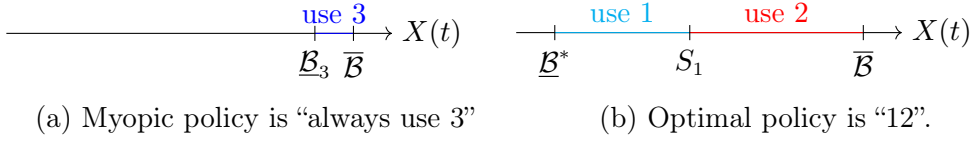



\begin{figure}[h]
    \centering
    \includegraphics[scale=0.5]{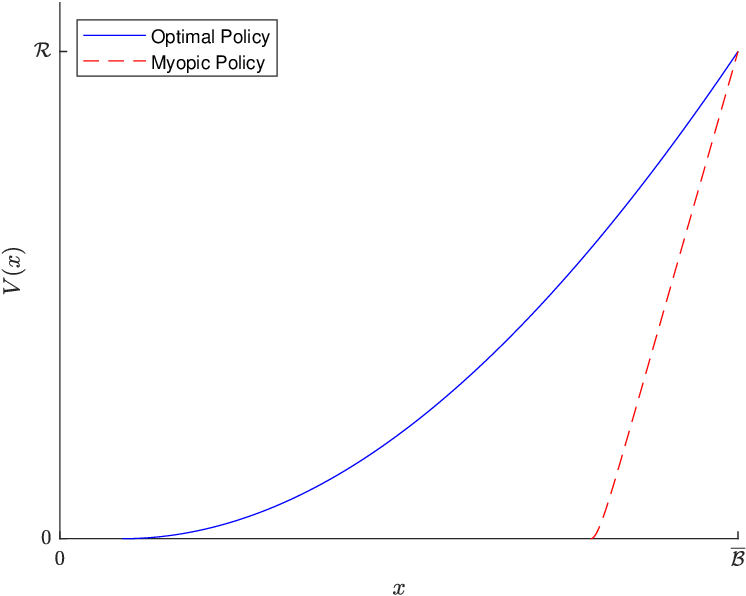}
    \caption{Value functions of the optimal (blue, solid) and myopic (red, dashed) policies. The parameters, optimal policy and myopic policy are the same as in Figure \ref{fig:myoPlcBad_FrLB_Plc}.}
    \label{fig:MyoPlcBad_FrLB_VF}
\end{figure}

\enlargethispage{\baselineskip}
\enlargethispage{\baselineskip}
\enlargethispage{\baselineskip}
\enlargethispage{\baselineskip}

\section{Supplementary Lemmas}

\begin{restatable}{lemma}{FrLBLstarRangeIneq}
\label{lem:FrLB_LstarRange_ineq}
There exists a sequence of problem instances $\left\{ \{c_{i,k}, \mu_{i,k}, \sigma_{i,k}\}_{i=1,2} \right\}_{k = 1,2,3,...}$, $\UNEW$, $M_k$ such that
\begin{align}
    \lim_{k \rightarrow \infty} \frac{\barLNEW_k - \frac{1}{\aNEW_{\n,k}} - (\UNEW - \barLNEW_k)(\frac{\aNEW_{\n,k}}{\aNEW_{1,k}} - 1)}{\LNEW^*_k} = 1,
\end{align}
where $\barLNEW_k = \UNEW - \MNEW_k \max_i CE_{i,k}$ and subscript $k$ denotes the $k$-th problem instance.
\end{restatable}

\emph{Proof of Lemma \ref{lem:FrLB_LstarRange_ineq}.}~ 
Define the following sequences
\begin{align}
    \MNEW_k &= \frac{1}{\epsilon_k}, \\
    \aNEW_{2,k} &= 1, \hNEW_{2,k} = 1, \aNEW_{1,k} = 1-\epsilon_k, \\
    \hNEW_{1,k} &= \frac{(-2+\epsilon_k)(1-\epsilon_k)}{-\epsilon_k LW\left((-2+\epsilon_k)e^{-\MNEW_k}\right) + (-2+\epsilon_k) }, 
\end{align}
where $\epsilon_k = \frac{1}{10^k}$. For any integer $k\ge 1$, note $\hNEW_{1,k} > \aNEW_{1,k}$ always, and hence control $2$ always has the highest \textit{CE}, i.e. $\n = 2$. Note also that $\hNEW_{1,k} = 1-\epsilon_k + O(\epsilon_k^2)$ and $\MNEW_k \epsilon_k = 1$. We can show that Lemma \ref{lem:FrLB_LstarRange_ineq} holds by direct computation. \Halmos




\begin{restatable}{lemma}{supp_lem:dmntEqvFxBr}
\label{supp_lem:dmntEqvFxBr}
Suppose $V(\cdot)$ is $C^2$ over $(\LNEW,\UNEW)$, and $-c_k + \mu_k V'(x) + \frac{\sigma_k^2}{2}V''(x) = 0$, then
\begin{align}
    -c_i + \mu_i V'(x) + \frac{\sigma_i^2}{2}V''(x) \le -c_k + \mu_k V'(x) + \frac{\sigma_k^2}{2}V''(x) \label{supp_lem_eq:dmntEqv1FxBrCs1}
\end{align}
if and only if
\begin{align}
    -\hNEW_i + \aNEW_i V'(x) + V''(x) \le -\hNEW_k + \aNEW_k V'(x) + V''(x). \label{supp_lem_eq:dmntEqv2FxBrCs1}
\end{align}
\end{restatable}

\emph{Proof of Lemma \ref{supp_lem:dmntEqvFxBr}.}~ Note that $-c_k + \mu_k V'(x) + \frac{\sigma_k^2}{2}V''(x) = 0$ is equivalent to $-\hNEW_k + \aNEW_k V'(x) + V''(x) = 0$. 

Suppose (\ref{supp_lem_eq:dmntEqv1FxBrCs1}) holds, then we have 
\begin{align*}
    -\hNEW_i + \aNEW_i V'(x) + V''(x) \le 0 = -\hNEW_k + \aNEW_k V'(x) + V''(x),
\end{align*}
i.e. (\ref{supp_lem_eq:dmntEqv2FxBrCs1}) holds.

Suppose (\ref{supp_lem_eq:dmntEqv2FxBrCs1}) holds, then we have
\begin{align*}
    -c_i + \mu_i V'(x) + \frac{\sigma_i^2}{2}V''(x) \le 0 = -c_k + \mu_k V'(x) + \frac{\sigma_k^2}{2}V''(x),
\end{align*}
i.e. (\ref{supp_lem_eq:dmntEqv1FxBrCs1}) holds.
\Halmos

\enlargethispage{\baselineskip}
\enlargethispage{\baselineskip}
\enlargethispage{\baselineskip}
\enlargethispage{\baselineskip}

\begin{restatable}{lemma}{BMBasicQtt}
\label{lem:BMBasicQtt}
Let $\{ B(t) \}_{t \in \R^+}$ be a Brownian Motion confined within $[\LNEW,\UNEW]$, with drift $\mu$ and variance $\sigma^2$ within 
$[\LNEW,\UNEW]$. Let $P_{x}^\UNEW, T_x$ be the probability of hitting $\UNEW$ before $\LNEW$, time 
spent before hitting either boundary respectively, starting from position $x \in [\LNEW,\UNEW]$. Then we have:
\begin{align*}
P_x^\UNEW &= \frac{  e^{v(x-\LNEW)}-1  }{  e^{v(x-\LNEW)}-e^{-v(\UNEW-x)} }\\
E T_x &= \frac{-\LNEW e^{-v\UNEW} + \UNEW e^{-v\LNEW} }{\mu (e^{-v\LNEW} - e^{-v\UNEW})} - \frac{(\UNEW-\LNEW) e^{-vx}}{\mu (e^{-v\LNEW} - e^{-v\UNEW})} - \frac{x}{\mu}
\end{align*}
\end{restatable}

\emph{Proof of Lemma \ref{lem:BMBasicQtt}.}~ These represent basic properties of Brownian Motion. Proof is omitted. \Halmos

\begin{restatable}{lemma}{effectiveDrift}
\label{lem:HighEDHighProbHitU}
If $\aNEW_2 > \aNEW_1$, then $P_x(X_T = \UNEW | \text{always use } 2) > P_x(X_T = \UNEW | \text{always use } 1)$ for all $x \in (\LNEW,\UNEW)$.
\end{restatable}

\bproof This is proven in Lemma 19 of [\cite{wang2022new}]. \Halmos

\begin{restatable}{lemma}{Lstar1Ctrl}
\label{lem:Lstar1Ctrl}
For a control with $(c_i, \mu_i, \sigma_i)$, the optimal lower boundary associated with using this control only is
\begin{align*}
    \LNEW_{i} = \UNEW - \MNEW \cdot CE_{i} - \frac{LW(-e^{-\MNEW \cdot CE_{i} \aNEW_i - 1})}{\aNEW_{i}},
\end{align*}
$LW(\cdot)$ is the Lambert-W function. 
\end{restatable}
\emph{Proof of Lemma \ref{lem:Lstar1Ctrl}.}~ The optimal lower boundary satisfies (c.f. [\cite{wang2022new}])
\begin{align*}
    -c_i + \mu_i V'(x) + \frac{\sigma_i^2}{2}V''(x) = 0,\\
    V'(\LNEW_i) = 0, \text{ } V(\LNEW_i) = 0, \text{ } V(\UNEW) = 0,
\end{align*}
and the result follows by solving the above ordinary differential equations. \Halmos

\begin{restatable}{lemma}{fxLBMyoPlcBadEpsLmt}~
\label{ec_lem:fxLBMyoPlcBadEpsLmt}
For any $x > 0$,
\begin{align}
    \lim_{\epsilon \rightarrow 0^+} -\frac{e^{-\frac{x}{\epsilon}} (2 \epsilon-1)}{e^{-\frac{1}{\epsilon}}-1}+\epsilon x+\frac{2 \epsilon-1}{e^{-\frac{1}{\epsilon}}-1}+\frac{e^{-x} (\epsilon-1)}{e^{-1}-1}-\frac{\epsilon-1}{e^{-1}-1} = \frac{e^{x} e^{-1}-1}{e^{x} e^{-1}-e^{x}}
\end{align}
\end{restatable}

\emph{Proof of Lemma \ref{ec_lem:fxLBMyoPlcBadEpsLmt}.}~ 
We know that $\lim_{\epsilon\to 0^+}-\frac{1}{\epsilon}\to-\infty$, and for any $x>0$, $\lim_{\epsilon\to 0^+}-\frac{x}{\epsilon}\to-\infty$. So
\begin{equation}
    \lim_{\epsilon\to 0^+}e^{-\frac{x}{\epsilon}}=\lim_{c\to-\infty}e^{c}=0,\hspace{0.3cm}\lim_{\epsilon\to 0^+}e^{-\frac{1}{\epsilon}}=\lim_{c\to-\infty}e^{c}=0.
\end{equation}
Thus, 
\begin{align}
&\lim_{\epsilon \rightarrow 0^+} -\frac{e^{-\frac{x}{\epsilon}} (2 \epsilon-1)}{e^{-\frac{1}{\epsilon}}-1}+\epsilon x+\frac{2 \epsilon-1}{e^{-\frac{1}{\epsilon}}-1}+\frac{e^{-x} (\epsilon-1)}{e^{-1}-1}-\frac{\epsilon-1}{e^{-1}-1}\\
=&\lim_{\epsilon\to 0^+}\frac{(-e^{-\frac{x}{\epsilon}}+1)(2\epsilon - 1)}{e^{-\frac{1}{\epsilon}} - 1} + \lim_{\epsilon\to 0^+}\epsilon x + \lim_{\epsilon\to 0^+}\frac{(e^{-x}-1)(\epsilon-1)}{e^{-1}-1}\\
=&\frac{-1}{-1}+0+\frac{-(e^{-x}-1)}{e^{-1}-1}\\
=&\frac{e^{-1}-1-e^{-x}+1}{e^{-1}-1}\\
=&\frac{e^xe^{-1}-1}{e^xe^{-1}-e^x}.
\end{align}
\Halmos

\section{Motivating Examples: When Insights from 2-Control Setting May Fail}
\label{sec:mtvEgTwoCtrlInsightsFail}

We will use two examples to demonstrate that the conclusions from [\cite{wang2022new}]'s 2-control setting may be misleading in the multi-control setting. For the reader's convenience, we replicate the result from [\cite{wang2022new}], which is Theorem A below.

\begin{restatable}{selfthm}{IandIE}~
\label{thm:IandIE}
\textit{($2$-Control Setting) Given two controls} $(\mu_i, \sigma_i, c_i)$, $(\mu_j, \sigma_j, c_j)$ \textit{with} $\mu_i/\sigma_i^2 < \mu_j/\sigma_j^2$\textit{, define the following conditions:}
\begin{align}
\Delta c_j \ge \max \{ d_{ij}^{I,\LNEW} \Delta c_i, d_{ij}^{I,\UNEW} \Delta c_i  \} \label{cond:I}, \\
\Delta c_j \le \min \{ d_{ij}^{II,\LNEW} \Delta c_i, d_{ij}^{II,\UNEW} \Delta c_i  \}, \label{cond:E}\\
\Delta c_j \in (d_{ij}^{II,\LNEW} \Delta c_i, d_{ij}^{I,\UNEW} \Delta c_i) \text{ and } \Delta c_i > 0, \label{cond:IE}\\
\Delta c_j \in (d_{ij}^{II,\UNEW} \Delta c_i, d_{ij}^{I,\LNEW} \Delta c_i) \text{ and } \Delta c_i < 0,  \label{cond:EI}
\end{align}
\textit{where $\Delta c_k = c_k - \MNEW \mu_k/(\UNEW-\LNEW)$, $k \in \{i,j \}$,  represents the \textit{deviation-to-critical cost (DCC)} of control $k$, $d_{ij}^{I,\LNEW}$, $d_{ij}^{I,\UNEW}$, $d_{ij}^{II,\LNEW}$, $d_{ij}^{II,\UNEW}$ are constants that only depend on the drifts and variances of the two controls (explicit expressions see \S \ref{EC_Sec:dILdIILFmla} in the Appendix). We have}

\begin{enumerate}[\hspace{0.25cm} i)]
\item \textit{If (\ref{cond:I}) holds, then the policy of always using control 1 is optimal.} \label{enum:condI}

\item \textit{If (\ref{cond:E}) holds, then the policy of always using control 2 is optimal.} \label{enum:condE}

\item \textit{If (\ref{cond:IE}) holds, then the optimal policy is a 2-interval policy in which control 1 is utilized in $[\LNEW, S^*)$ and control 2 in $[S^*,\UNEW)$, for some $S^* \in (\LNEW,\UNEW)$.} \label{enum:condIE}

\item \textit{If (\ref{cond:EI}) holds, then the optimal policy is a 2-interval policy which involves using control 2 in $[\LNEW, S^*)$ and using control 1 in $[S^*,\UNEW)$, for some $S^* \in (\LNEW,\UNEW)$.} \label{enum:condEI}
\end{enumerate}
\end{restatable}
Intuitively, \textit{DCC} represents the difference between a control's real cost and a benchmark cost. In the $2$-control setting, if the \textit{DCC} of a control is relatively large, then that control is undesirable to the entrepreneur and will not be used in the optimal policy; otherwise if \textit{DCC} is relatively small, that control is desirable and will be used in the optimal policy.

\textbf{Motivating Example 1.} First consider a case in which $\LNEW=0$, $\UNEW = 10$, $\MNEW = 5$ and there are three controls: $(\mu_1,\sigma_1,c_1) = (0.26, 2, 0.01)$, $(\mu_2,\sigma_2,c_2) = (0.05, 0.5, 0.02)$, $(\mu_3,\sigma_3,c_3) = (2, 1, 1.7)$. When the entrepreneur only has control $1$ and $3$ at his disposal, we have $\Delta c_3 \ge \max\{d_{13}^{I,\LNEW}\Delta c_1,d_{13}^{I,\UNEW}\Delta c_1\}$ and by Theorem A i) the optimal policy is to always use control $1$ (in this case, we say control $3$ is \textit{displaced} by control $1$). When the entrepreneur only has control $2$ and $3$ at his disposal, again we have $\Delta c_3 \ge \max\{d_{23}^{I,\LNEW}\Delta c_2,d_{23}^{I,\UNEW}\Delta c_2\}$ and the optimal policy is to always use control $2$, i.e. control $3$ is \textit{displaced} by control $2$. By the intuition from the $2$-control setting, control $3$ is completely ``dominated'' by control $1$ or $2$, which suggests that control $3$ has a cost that is too high compared to either control $1$ or control $2$, and it seems control $3$ should not provide any value to the entrepreneur when control $1$ and/or $2$ is present. Surprisingly, when all three controls are available to the entrepreneur, the optimal policy is ``321'', i.e. consists of all three controls, as shown in Figure \ref{fig:myoPlcIttFail1}.

In other words, control $3$ is ``dominated'' by either control $1$ or $2$ alone but is not ``dominated'' when both of these controls are present. This suggests we cannot rely on the intuition from the $2$-control setting to predict a priori which controls will be used in the optimal policy in the multi-control setting.

\begin{figure}[h]
    \centering
    \begin{subfigure}[b]{0.45\textwidth}
    \begin{tikzpicture}[xscale = 1.7]
        \def\LPos{0.1}
        \def\UPos{0.9}
        
        \draw[->] (-0.5,1) -- (2.5,1) node[pos=1, right, black] {$X(t)$};
        \draw[-, cyan] (-0.5+\LPos*3, 1) -- (-0.5+\UPos*3, 1) node[pos=0.5, above] {use $1$};
        \draw[-] (-0.5+\LPos*3, 1.1) -- (-0.5+\LPos*3, 0.9)  node[pos=1, below, black] {$\LNEW$};
        \draw[-] (-0.5+\UPos*3, 1.1) -- (-0.5+\UPos*3, 0.9)  node[pos=1, below, black] {$\UNEW$};
    \end{tikzpicture}
    \caption{Optimal policy is ``always \\ use 1''  when only control 1 \\ and 3 are available.}
    \end{subfigure}
    \begin{subfigure}[b]{0.45\textwidth}
    \begin{tikzpicture}[xscale = 1.7]
        \def\LPos{0.1}
        \def\UPos{0.9}
    
        \draw[->] (-0.5,1) -- (2.5,1) node[pos=1, right, black] {$X(t)$};
        \draw[-, red] (-0.5+\LPos*3, 1) -- (-0.5+\UPos*3, 1) node[pos=0.5, above] {use $2$};
        \draw[-] (-0.5+\LPos*3, 1.1) -- (-0.5+\LPos*3, 0.9)  node[pos=1, below, black] {$\LNEW$};
        \draw[-] (-0.5+\UPos*3, 1.1) -- (-0.5+\UPos*3, 0.9)  node[pos=1, below, black] {$\UNEW$};
    \end{tikzpicture}
    \caption{Optimal policy is ``always \\ use 2''  when only control 2 \\ and 3 are available.}
    \end{subfigure}
    
    \begin{subfigure}[b]{0.4\textwidth}
    \begin{tikzpicture}[xscale = 2.1]
        \def\LPos{0.1}
        \def\UPos{0.9}
        \def\FirstSPos{0.4}  
        \def\SecondSPos{0.7}  
    
        \draw[->] (-0.5,1) -- (2.5,1) node[pos=1, right, black] {$X(t)$};
        \draw[-, blue] (-0.5+\LPos*3, 1) -- (-0.5+\FirstSPos*3, 1)  node[pos=0.5, above] {use $3$};
        \draw[-, red] (-0.5+\FirstSPos*3, 1) -- (-0.5+\SecondSPos*3, 1)  node[pos=0.5, above] {use $2$};
        \draw[-, cyan] (-0.5+\SecondSPos*3, 1) -- (-0.5+\UPos*3, 1)  node[pos=0.5, above] {use $1$};
        \draw[-] (-0.5+\LPos*3, 1.1) -- (-0.5+\LPos*3, 0.9)  node[pos=1, below, black] {$\LNEW$};
        \draw[-] (-0.5+\UPos*3, 1.1) -- (-0.5+\UPos*3, 0.9)  node[pos=1, below, black] {$\UNEW$};
        \draw[-] (-0.5+\FirstSPos*3, 1.1) -- (-0.5+\FirstSPos*3, 0.9)  node[pos=1, below, black] {$S_1$};
        \draw[-] (-0.5+\SecondSPos*3, 1.1) -- (-0.5+\SecondSPos*3, 0.9)  node[pos=1, below, black] {$S_2$};
    \end{tikzpicture}
    \caption{Optimal policy is ``321'' when all three controls are available.}
    \end{subfigure}
    
\caption{Motivating Example 1. Parameters are: $(\mu_1,\sigma_1,c_1) = (0.26, 2, 0.01)$, $(\mu_2,\sigma_2,c_2) = (0.05, 0.5, 0.02)$, $(\mu_3,\sigma_3,c_3) = (2, 1, 1.7)$, $\LNEW = 0$, $\UNEW = 10$, $\MNEW = 5$.}
\label{fig:myoPlcIttFail1}
\end{figure}

\textbf{Motivating Example 2.} Next consider the case where $\LNEW = 2$, $\UNEW = 10$, $\MNEW = 10$ and the following three controls are present: $(\mu_1,\sigma_1,c_1) = (0.26, 1.8, 0.14)$, $(\mu_2,\sigma_2,c_2) = (0.16, 0.8, 0.15)$, $(\mu_3,\sigma_3,c_3) = (0.14, 0.24, 0.16)$. When the entrepreneur only has control $1$ and $2$ at his disposal, we have $\Delta c_2 \in (d_{12}^{II,\UNEW} \Delta c_1, d_{12}^{I,\LNEW} \Delta c_1)$ and by Theorem A iv) the optimal policy is the two-interval policy ``21''. When the entrepreneur only has control $2$ and $3$ at his disposal, again we have $\Delta c_3 \in (d_{23}^{II,\UNEW} \Delta c_2, d_{23}^{I,\LNEW} \Delta c_2)$ and the optimal policy is the two-interval policy ``32''. By the intuition from the $2$-control setting, control $2$ is in some sense ``comparable'' to control $1$ and $3$, i.e. control $2$ has a cost that is relatively moderate compared to either control $1$ or control $3$ and seems as useful as them. However, when all three controls are available to the entrepreneur, the optimal policy only uses control $1$ and $3$ and does not use control $2$, as shown in Figure \ref{fig:myoPlcIttFail2}.

In other words, control $2$ is ``comparable'' to either control $1$ or $3$ alone but is  ``dominated'' by control $1$ and $3$ together. This shows again that the intuition from the $2$-control setting may fail in the multi-control one.

\begin{figure}[h]
    \centering
    \begin{subfigure}[b]{0.45\textwidth}
    \begin{tikzpicture}[xscale = 1.7]
        \def\LPos{0.1}
        \def\UPos{0.9}
        \def\SPos{0.5}  
    
        \draw[->] (-0.5,1) -- (2.5,1) node[pos=1, right, black] {$X(t)$};
        \draw[-, red] (-0.5+\LPos*3, 1) -- (-0.5+\SPos*3, 1)  node[pos=0.5, above] {use $2$};
        \draw[-, cyan] (-0.5+\SPos*3, 1) -- (-0.5+\UPos*3, 1)  node[pos=0.5, above] {use $1$};
        \draw[-] (-0.5+\LPos*3, 1.1) -- (-0.5+\LPos*3, 0.9)  node[pos=1, below, black] {$\LNEW$};
        \draw[-] (-0.5+\UPos*3, 1.1) -- (-0.5+\UPos*3, 0.9)  node[pos=1, below, black] {$\UNEW$};
        \draw[-] (-0.5+\SPos*3, 1.1) -- (-0.5+\SPos*3, 0.9)  node[pos=1, below, black] {$S_1$};
    \end{tikzpicture}
    \caption{Optimal policy is ``21''\\ when only control 1 and 2 \\are available.}
    \end{subfigure}
    \begin{subfigure}[b]{0.45\textwidth}
    \begin{tikzpicture}[xscale = 1.7]
        \def\LPos{0.1}
        \def\UPos{0.9}
        \def\SPos{0.5}  
    
        \draw[->] (-0.5,1) -- (2.5,1) node[pos=1, right, black] {$X(t)$};
        \draw[-, blue] (-0.5+\LPos*3, 1) -- (-0.5+\SPos*3, 1)  node[pos=0.5, above] {use $3$};
        \draw[-, red] (-0.5+\SPos*3, 1) -- (-0.5+\UPos*3, 1)  node[pos=0.5, above] {use $2$};
        \draw[-] (-0.5+\LPos*3, 1.1) -- (-0.5+\LPos*3, 0.9)  node[pos=1, below, black] {$\LNEW$};
        \draw[-] (-0.5+\UPos*3, 1.1) -- (-0.5+\UPos*3, 0.9)  node[pos=1, below, black] {$\UNEW$};
        \draw[-] (-0.5+\SPos*3, 1.1) -- (-0.5+\SPos*3, 0.9)  node[pos=1, below, black] {$S_1$};
    \end{tikzpicture}
    \caption{Optimal policy is ``32'' \\ when only control 2 and 3 \\are available.}
    \end{subfigure}
    
    \begin{subfigure}[b]{0.4\textwidth}
    \begin{tikzpicture}[xscale = 2.1]
        \def\LPos{0.1}
        \def\UPos{0.9}
        \def\SPos{0.5}  
    
        \draw[->] (-0.5,1) -- (2.5,1) node[pos=1, right, black] {$X(t)$};
        \draw[-, blue] (-0.5+\LPos*3, 1) -- (-0.5+\SPos*3, 1)  node[pos=0.5, above] {use $3$};
        \draw[-, cyan] (-0.5+\SPos*3, 1) -- (-0.5+\UPos*3, 1)  node[pos=0.5, above] {use $1$};
        \draw[-] (-0.5+\LPos*3, 1.1) -- (-0.5+\LPos*3, 0.9)  node[pos=1, below, black] {$\LNEW$};
        \draw[-] (-0.5+\UPos*3, 1.1) -- (-0.5+\UPos*3, 0.9)  node[pos=1, below, black] {$\UNEW$};
        \draw[-] (-0.5+\SPos*3, 1.1) -- (-0.5+\SPos*3, 0.9)  node[pos=1, below, black] {$S_1$};
    \end{tikzpicture}
    \caption{Optimal policy is ``31'' when all three controls are available.}
    \end{subfigure}
    
\caption{Motivating Example 2. Parameters are: $(\mu_1,\sigma_1,c_1) = (0.26, 1.8, 0.14)$, $(\mu_2,\sigma_2,c_2) = (0.16, 0.8, 0.15)$, $(\mu_3,\sigma_3,c_3) = (0.14, 0.24, 0.16)$, $\LNEW = 2$, $\UNEW = 10$, $\MNEW = 10$.}
\label{fig:myoPlcIttFail2}
\end{figure}

This has implications on the entrepreneur. It shows that an activity with a relatively high \textit{risk-effectiveness} may become more valuable when more \textit{cost-effective} activities are available. Specifically, such an activity can be employed near the lower boundary to avoid hitting it, which creates synergy with other \textit{cost-effective} activities. On the other hand, an activity with moderate \textit{risk-effectiveness} may lose its value when both low and high \textit{risk-effective} activities are present. In this case, such an activity may become redundant by the combined effect of two other activities—one with high and one with low \textit{risk-effectiveness}. Consequently, the entrepreneur needs to holistically evaluate all the activities when deciding which ones to be used in the optimal policy. \paragraph{}\label{para:2ctrlToMultiCtrlImplc}

\enlargethispage{1\baselineskip}

\subsection{Expressions for $d_{ij}^{I,\LNEW}$, $d_{ij}^{I,\UNEW}$, $d_{ij}^{II,\LNEW}$, $d_{ij}^{II,\UNEW}$ in the Two-Control Setting}
\label{EC_Sec:dILdIILFmla}
The four constants $d^{I,\LNEW}_{ij},d^{II,\LNEW}_{ij}, d^{I,\UNEW}_{ij}, d^{II,\UNEW}_{ij}$ are defined as follows:
\begin{align}
d_{ij}^{I,\LNEW} \triangleq    -\frac{\sigma_j^2}{\sigma_i^2}(\aNEW_j - \aNEW_i) \frac{\UNEW}{1- e^{-\aNEW_i \UNEW}} +\frac{\mu_j}{\mu_i}, \label{eq:d1L}\\
d_{ij}^{I,\UNEW} \triangleq -\frac{\sigma_j^2}{\sigma_i^2}(\aNEW_j - \aNEW_i) \frac{\UNEW e^{-\aNEW_i \UNEW}}{1- e^{-\aNEW_i \UNEW}} +\frac{\mu_j}{\mu_i} , \label{eq:d1U}\\
d_{ij}^{II,\LNEW} \triangleq  \frac{1}{  \frac{\mu_i}{\mu_j} + \frac{\sigma_i^2}{\sigma_j^2}(\aNEW_j - \aNEW_i)\frac{\UNEW}{1-e^{-\aNEW_j \UNEW}}  }, \label{eq:d2L}\\
d_{ij}^{II,\UNEW} \triangleq  \frac{1}{\frac{\mu_i}{\mu_j} + \frac{\sigma_i^2}{\sigma_j^2}(\aNEW_j - \aNEW_i)\frac{\UNEW e^{-\aNEW_j \UNEW}}{1-e^{-\aNEW_j \UNEW}} }. \label{eq:d2U}
 \end{align}

\end{APPENDICES}

\end{document}